\documentclass[a4paper,10pt]{article}

\usepackage[english]{babel}
\usepackage[utf8]{inputenc} 
\usepackage[T1]{fontenc}
\usepackage{lmodern}
\usepackage{geometry}
\usepackage{graphicx} 
\usepackage{color}
\usepackage{amsmath}
\usepackage{amssymb}
\usepackage{amsthm}
\usepackage{booktabs}
\usepackage{longtable}
\usepackage{multirow}
\usepackage{siunitx}
\usepackage{cite}
\usepackage{hyperref}
\usepackage{algorithm}
\usepackage{algorithmic}
\usepackage{boldline}
\usepackage{multirow}
\usepackage{pgfplots}
\DeclareGraphicsExtensions{.png}
\pgfplotsset{width=2.8cm,height=2.84cm,compat=newest}
\allowdisplaybreaks

\newcommand{\COLS}[1]{\multicolumn{3}{c|@{\hspace{2pt}}}{\rule[0.065cm]{0.6cm}{0.01mm} #1 \rule[0.065cm]{0.6cm}{0.01mm}}}
\newcommand{\CLASS}{\multicolumn{2}{c}{Classifier}}
\newcommand{\PHI}{\multicolumn{2}{c}{\rule[0.065cm]{0.6cm}{0.01mm} $\phi$ \rule[0.065cm]{0.6cm}{0.01mm}}}
\newcommand{\RMD}{\multirow{11}{*}{\raisebox{-0.65mm}{\rotatebox{90}{\rule[-0.25cm]{2.35cm}{0.01mm} \rotatebox{-90}{\texttt{RMD}} \rule[-0.25cm]{2.35cm}{0.01mm}}}}}
\newcommand{\SVM}{\multirow{12}{*}{\raisebox{-0.65mm}{\rotatebox{90}{\rule[-0.25cm]{2.6cm}{0.01mm} \rotatebox{-90}{\texttt{SVM}} \rule[-0.25cm]{2.6cm}{0.01mm}}}}}

\newcommand{\N}{\mathbb{N}}
\newcommand{\R}{\mathbb{R}}
\newcommand{\rmd}{\texttt{RMD}}
\newcommand{\oemd}{\texttt{EMD OLD}}
\newcommand{\nemd}{\texttt{EMD NEW}}
\newcommand{\E}{\mathcal{E}}
\newcommand{\C}{\mathcal{C}}
\newcommand{\D}{\mathcal{D}}
\newcommand{\X}{\mathcal{X}}
\renewcommand{\L}{\mathcal{L}}

\newcommand{\p}{\mathbf{p}}
\newcommand{\q}{\mathbf{q}}
\renewcommand{\u}{\mathbf{u}}
\renewcommand{\v}{\mathbf{v}}
\newcommand{\x}{\mathbf{x}}
\newcommand{\y}{\mathbf{y}}

\renewcommand{\P}{\mathbf{P}}
\newcommand{\ones}{\mathbf{1}}
\newcommand{\zeros}{\mathbf{0}}

\DeclareMathOperator{\dom}{dom}
\DeclareMathOperator{\interior}{int}
\DeclareMathOperator{\ri}{ri}
\DeclareMathOperator{\bd}{bd}
\DeclareMathOperator{\diag}{diag}
\DeclareMathOperator{\vect}{vec}
\DeclareMathOperator{\sgn}{sgn}

\DeclareMathOperator*{\argmin}{argmin}

\theoremstyle{plain}
\newtheorem{theorem}{Theorem}
\newtheorem{corollary}[theorem]{Corollary}
\newtheorem{proposition}[theorem]{Proposition}
\newtheorem{lemma}[theorem]{Lemma}

\theoremstyle{plain}
\newtheorem{property}{Property}

\theoremstyle{plain}
\newtheorem{definition}{Definition}

\theoremstyle{remark}

 \newtheorem{myremark}{Remark}

\title{Regularized Optimal Transport and the ROT Mover's Distance}

\date{\today}

\author{Arnaud Dessein\\Qucit, B{\`{e}}gles, France\\\href{mailto:arnaud.dessein@qucit.com}{arnaud.dessein@qucit.com} \and Nicolas Papadakis\\IMB, CNRS, Bordeaux, France\\\hspace{2cm}\href{mailto:nicolas.papadakis@math.u-bordeaux.fr}{nicolas.papadakis@math.u-bordeaux.fr}\hspace{2cm} \and Jean-Luc Rouas\\LaBRI, CNRS, Bordeaux, France\\\href{mailto:jean-luc.rouas@labri.fr}{jean-luc.rouas@labri.fr}}

\begin{document}

\maketitle

\maketitle

\begin{abstract}%
This paper presents a unified framework for smooth convex regularization of discrete optimal transport problems. In this context, the regularized optimal transport turns out to be equivalent to a matrix nearness problem with respect to Bregman divergences. Our framework thus naturally generalizes a previously proposed regularization based on the Boltzmann-Shannon entropy related to the Kullback-Leibler divergence, and solved with the Sinkhorn-Knopp algorithm. We call the regularized optimal transport distance the rot mover's distance in reference to the classical earth mover's distance. By exploiting alternate Bregman projections, we develop the alternate scaling algorithm and non-negative alternate scaling algorithm, to compute efficiently the regularized optimal plans depending on whether the domain of the regularizer lies within the non-negative orthant or not. We further enhance the separable case with a sparse extension to deal with high data dimensions. We also instantiate our framework and discuss the inherent specificities for well-known regularizers and statistical divergences in the machine learning and information geometry communities. Finally, we demonstrate the merits of our methods with experiments using synthetic data to illustrate the effect of different regularizers, penalties and dimensions, as well as real-world data for a pattern recognition application to audio scene classification.
\end{abstract}

\section{Introduction}
\label{sec:intro}

A recurrent problem in statistical machine learning is the choice of a relevant distance measure to compare probability distributions. Various information divergences are famous, among which Euclidean, Mahalanobis, Kullback-Leibler, Itakura-Saito, Hellinger, $\chi^2$, $\ell_p$ (quasi-)norm, total variation, logistic loss function, or more general Csisz{\'{a}}r and Bregman divergences and parametric families of such divergences such as $\alpha$\nobreakdash- and $\beta$-divergences.

An alternative family of distances between probability distributions can be introduced in the framework of optimal transport (OT). Rather than performing a pointwise comparison of the distributions, the idea is to quantify the minimal effort for moving the probability mass of one distribution to the other, where the transport plan to move the mass is optimized according to a given ground cost. This makes OT distances suitable and robust in certain applications, notably in the field of computer vision where the discrete OT distance, also known as earth mover's distance (EMD), has been popularized to compare histograms of features for pattern recognition tasks~\cite{Rubner2000}.

Despite its appealing theoretical properties, intuitive formulation, and excellent performance in various problems of information retrieval, the computation of the EMD involves solving a linear program whose cost quickly becomes prohibitive with the data dimension. In practice, the best algorithms currently proposed, such as the network simplex~\cite{Ahuja1993}, scale at least with a super-cubic complexity. Embeddings of the distributions can be used to approximate the EMD with linear complexity~\cite{Indyk2003,Grauman2004,Shirdhonkar2008}, and the network simplex can be modified to run in quadratic time~\cite{Gudmundsson2007,Ling2007,Pele2009}. Nevertheless, the distortions inherent to such embeddings~\cite{Naor2007}, and the exponential increase of costs incurred by such modifications, make these approaches inapplicable for dimensions higher than four. Instead, multi-scale strategies~\cite{Obermann2015} and shortcut paths~\cite{Schmitzer2016a} can speed up the estimation of the exact optimal plan. These approaches are yet limited to particular convex costs such as $\ell_2$, while other costs such as $\ell_1$ and truncated or compressed versions are often preferred in practice for an increased robustness to data outliers~\cite{Pele2008,Pele2009,Rabin2009}. For general applications, a gain in performance can also be obtained with a cost directly learned from labeled data~\cite{Cuturi2014}. The aforementioned accelerated methods that are dedicated to $\ell_2$ or convex costs are thus not adapted in this context.

On another line of research, the regularization of the transport plan, for example via graph modeling~\cite{Ferradans2014}, has been considered to deal with noisy data, though this latter approach does not address the computational issue of efficiency for high dimensions. In this continuity, an entropic regularization was shown to admit an efficient algorithm with quadratic complexity that speeds up the computation of solutions by several orders of magnitude, and to improve performance on applications such as handwritten digit recognition~\cite{Cuturi2013}. In addition, a tailored computation can be obtained via convolution for specific ground costs~\cite{Solomon2015}. Since the introduction of the entropic regularization, OT has benefited from extensive developments in the machine learning community, with applications such as label propagation~\cite{Solomon2014}, domain adaptation~\cite{Courty2015}, matrix factorization~\cite{Zen2014}, dictionary learning~\cite{Rolet2016,Schmitz2018}, barycenter computation~\cite{Cuturi2016}, geodesic principal component analysis~\cite{Bigot2013,Seguy2015,Cazelles2017}, data fitting~\cite{Frogner2015}, statistical inference~\cite{Bernton2017}, training of Boltzmann machines~\cite{Montavon2016} and generative adversarial networks~\cite{Arjovsky2017,Bousquet2017,Genevay2017}.

With the entropic regularization, the gain in computational time is only important for high dimensions or large levels of regularization. For low regularization, advanced optimization strategies can still be used to obtain a significant speed-up~\cite{Thibault2017,Schmitz2018}. It is also a well-known effect that the entropic regularization overspreads the transported mass, which may be undesirable for certain applications as in the case of interpolation purposes. An interesting perspective of these works, however, is that many more regularizers are worth investigating to solve OT problems both efficiently and robustly~\cite{Galichon2015,Muzellec2016,Blondel2017}. This is the idea we address in the present work, focusing on smooth convex regularization.

\subsection{Notations}

For the sake of simplicity, we consider distributions with same dimension $d$, and thus work with the Euclidean space $\R^{d \times d}$ of square matrices. It is straightforward, however, to extend all results for a different number of bins $m, n$ by using rectangular matrices in $\R^{m \times n}$ instead. We denote the null matrix of $\R^{d \times d}$ by $\zeros$, and the matrix full of ones by $\ones$. The Frobenius inner product between two matrices $\boldsymbol\pi, \boldsymbol\xi \in \R^{d \times d}$ is defined by:
\begin{equation}
\langle \boldsymbol\pi, \boldsymbol\xi \rangle = \sum_{i = 1}^d \sum_{j = 1}^d \pi_{ij} \xi_{ij} \enspace.
\end{equation}
When the intended meaning is clear from the context, we also write $\zeros$ for the null vector of $\R^d$, and $\ones$ for the vector full of ones. The notation $\cdot^\top$ represents the transposition operator for matrices or vectors. The probability simplex of $\R^d$ is defined as follows:
\begin{equation}
\Sigma_d = \{\p \in \R_+^d \colon \p^\top \ones = 1\} \enspace.
\end{equation}

The operator $\diag(\v)$ transforms a vector $\v \in \R^d$ into a diagonal matrix $\boldsymbol\pi \in \R^{d \times d}$ such that $\pi_{ii} = v_i$, for all $1 \leq i \leq d$. The operator $\vect(\boldsymbol\pi)$ transforms a matrix $\boldsymbol\pi \in \R^{d \times d}$ into a vector $\x \in \R^{d^2}$ such that $x_{i + (j - 1) d} = \pi_{ij}$, for all $1 \leq i, j \leq d$. The operator $\sgn(x)$ for $x \in \R$ returns $-1, 0, +1$, if $x$ is negative, null, positive, respectively. Functions of a real variable, such as the absolute value, sign, exponential or power functions, are considered element-wise when applied to matrices. The max operator and inequalities between matrices should also be interpreted element-wise. Matrix divisions are similarly considered element-wise, whereas element-wise matrix multiplications, also known as Hadamard or Schur products, are denoted by $\odot$ to remove any ambiguity with standard matrix multiplications. Lastly, addition or subtraction of a scalar and a matrix should be understood element-wise by replicating the scalar.

\subsection{Background and Related Work}
\label{subsec:background}

Given two probability vectors $\p, \q \in \Sigma_d$, and a cost matrix $\boldsymbol\gamma \in \R_+^{d \times d}$ whose coefficients $\gamma_{ij}$ represent the cost of moving the mass from bin $p_i$ to $q_j$, the total cost of a given transport plan, or coupling, $\boldsymbol\pi \in \Pi(\p, \q)$ can be quantified as $\langle \boldsymbol\pi, \boldsymbol\gamma \rangle$. An optimal cost is then obtained by solving a linear program:
\begin{equation}
d_{\boldsymbol\gamma}(\p, \q) = \min_{\boldsymbol\pi \in \Pi(\p, \q)} \langle \boldsymbol\pi, \boldsymbol\gamma \rangle \enspace,
\end{equation}
with the transport polytope of $\p$ and $\q$, also known as the polytope of couplings between $\p$ and $\q$, defined as the following polyhedron:
\begin{equation}
\Pi(\p, \q) = \{\boldsymbol\pi \in \R_+^{d \times d} \colon \boldsymbol\pi \ones = \p, \boldsymbol\pi^\top \ones = \q\} \enspace.
\end{equation}
The EMD associated to the cost matrix $\boldsymbol\gamma$ is given by $d_{\boldsymbol\gamma}$ and is a true distance metric on the probability simplex $\Sigma_d$ whenever $\boldsymbol\gamma$ is itself a distance matrix. In general, the optimal plans, or earth mover's plans, have at most $2d - 1$ nonzero entries, and consist either of a single vertex or of a whole facet of the transport polytope. One of the earth mover's plans can be obtained with the network simplex~\cite{Ahuja1993} among other approaches. For a general cost matrix $\boldsymbol\gamma$, the complexity of solving an OT problem scales at least in $O(d^3 \log d)$ for the best algorithms currently proposed, including the network simplex, and turns out to be super-cubic in practice as well.

\cite{Cuturi2013} proposed a new family of OT distances, called Sinkhorn distances, from the perspective of maximum entropy. The idea is to smooth the original problem with a strictly convex regularization via the Boltzmann-Shannon entropy. The primal problem involves the entropic regularization as an additional constraint:
\begin{equation}
d_{\boldsymbol\gamma, \alpha}'(\p, \q) = \min_{\boldsymbol\pi \in \Pi_\alpha(\p, \q)} \langle \boldsymbol\pi, \boldsymbol\gamma \rangle \enspace,
\end{equation}
with the regularized transport polytope defined as follows:
\begin{equation}
\Pi_\alpha(\p, \q) = \{\boldsymbol\pi \in \Pi(\p, \q) \colon E(\boldsymbol\pi) \leq E(\p\q^\top) + \alpha\} \enspace,
\end{equation}
where $\alpha \geq 0$ is a regularization term and $E$ is minus the Boltzmann-Shannon entropy as defined in~\eqref{eq:mbse}. It is also straightforward to prove that we have:
\begin{align}
\Pi_\alpha(\p, \q) & = \{\boldsymbol\pi \in \Pi(\p, \q) \colon K(\boldsymbol\pi \Vert \ones) \leq K(\p\q^\top \Vert \ones) + \alpha\} \enspace,\\
\Pi_\alpha(\p, \q) & = \{\boldsymbol\pi \in \Pi(\p, \q) \colon K(\boldsymbol\pi \Vert \p\q^\top) \leq \alpha\} \enspace.
\end{align}
where $K$ is the Kullback-Leibler divergence as defined in~\eqref{eq:kl}. This enforces the solution to have sufficient entropy, or equivalently small enough mutual information, by constraining it to the Kullback-Leibler ball of radius $K(\p\q^\top \Vert \ones) + \alpha$, respectively $\alpha$, and center the matrix $\ones \in \R_{++}^{d \times d}$, respectively the transport plan $\p\q^\top \in \R_{++}^{d \times d}$, which have maximum entropy. The dual problem exploits a Lagrange multiplier to relax the entropic regularization as a penalty:
\begin{equation}
d_{\boldsymbol\gamma, \lambda}(\p, \q) = \langle \boldsymbol\pi_\lambda^\star, \boldsymbol\gamma \rangle \enspace,
\end{equation}
with the regularized optimal plan $\boldsymbol\pi_\lambda^\star$ defined as follows:
\begin{equation}
\boldsymbol\pi_\lambda^\star = \argmin_{\boldsymbol\pi \in \Pi(\p, \q)} \, \langle \boldsymbol\pi, \boldsymbol\gamma \rangle + \lambda E(\boldsymbol\pi) \enspace,
\end{equation}
where $\lambda > 0$ is a regularization term. The problem can then be solved empirically in quadratic complexity with linear convergence using the Sinkhorn-Knopp algorithm~\cite{Sinkhorn1967} based on iterative matrix scaling, where rows and columns are rescaled in turn so that they respectively sum up to $\p$ and $\q$ until convergence. Finally, it is easy to prove that we have:
\begin{align}
\boldsymbol\pi_\lambda^\star = & \argmin_{\boldsymbol\pi \in \Pi(\p, \q)} \, \langle \boldsymbol\pi, \boldsymbol\gamma \rangle + \lambda K(\boldsymbol\pi \Vert \ones) \enspace,\\
\boldsymbol\pi_\lambda^\star = & \argmin_{\boldsymbol\pi \in \Pi(\p, \q)} \, \langle \boldsymbol\pi, \boldsymbol\gamma \rangle + \lambda K(\boldsymbol\pi \Vert \p\q^\top) \enspace.
\end{align}
This again shows that the regularization enforces the solution to have sufficient entropy, or equivalently small enough mutual information, by shrinking it toward the matrix $\ones$ and the joint distribution $\p\q^\top$ which have maximum entropy.

\cite{Benamou2015} revisited the entropic regularization in a geometrical framework with iterative information projections. They showed that computing a Sinkhorn distance in dual form actually amounts to the minimization of a Kullback-Leibler divergence:
\begin{equation}
\boldsymbol\pi_\lambda^\star = \argmin_{\boldsymbol\pi \in \Pi(\p, \q)} K(\boldsymbol\pi \Vert \exp(-\boldsymbol\gamma / \lambda)) \enspace.
\end{equation}
Precisely, this amounts to computing the Kullback-Leibler projection of $\exp(-\boldsymbol\gamma / \lambda) \in \R_{++}^{d \times d}$ onto the transport polytope $\Pi(\p, \q)$. In this context, the Sinkhorn-Knopp algorithm turns out to be a special instance of Bregman projection onto the intersection of convex sets via alternate projections. Specifically, we see $\Pi(\p, \q)$ as the intersection of the non-negative orthant with two affine subspaces containing all matrices with rows and columns summing to $\p$ and $\q$ respectively, and we alternate projection on these two subspaces according to the Kullback-Leibler divergence until convergence.

\cite{Kurras2015} further studied this equivalence in the wider context of iterative proportional fitting. He notably showed that the Sinkhorn-Knopp and alternate Bregman projections can be extended to account for infinite entries in the cost matrix $\boldsymbol\gamma$, and thus null entries in the regularized optimal plan. Hence, it is possible to develop a sparse version of the entropic regularization to OT problems. This becomes interesting to store the $d \times d$ matrix variables and perform the required computations when the data dimension gets large.

\cite{Dhillon2007} had already enlightened such an equivalence in the field of matrix analysis. They actually considered the estimation of contingency tables with fixed marginals as a matrix nearness problem based on the Kullback-Leibler divergence. In more detail, they use a rough estimate $\boldsymbol\xi \in \R_{++}^{d \times d}$ to produce a contingency table $\boldsymbol\pi^\star$ that has fixed marginals $\p, \q$ by Kullback-Leibler projection of $ \boldsymbol\xi$ onto $\Pi(\p, \q)$:
\begin{equation}
\boldsymbol\pi^\star = \argmin_{\boldsymbol\pi \in \Pi(\p, \q)} K(\boldsymbol\pi \Vert \boldsymbol\xi) \enspace.
\end{equation}
They showed that alternate Bregman projections specialize to the Sinkhorn-Knopp algorithm in this context. However, no relationship to OT problems was highlighted.

\subsection{Contributions and Organization}

Our main contribution is to formulate a unified framework for discrete regularized optimal transport (ROT) by considering a large class of smooth convex regularizers. We call the underlying distance the rot mover's distance (RMD) and show that a given ROT problem actually amounts to the minimization of an associated Bregman divergence. This allows the derivation of two schemes that we call the alternate scaling algorithm (ASA) and the non-negative alternate scaling algorithm (NASA), to compute efficiently the regularized optimal plans depending on whether the domain of the regularizer lies within the non-negative orthant or not. These schemes are based on the general form of alternate projections for Bregman divergences. They also exploit the Newton-Raphson method to approximate the projections for separable divergences. The separable case is further enhanced with a sparse extension to deal with high data dimensions. We also instantiate our two generic schemes with widely-used regularizers and statistical divergences.

The proposed framework naturally extends the Sinkhorn-Knopp algorithm for the regularization based on the Boltzmann-Shannon entropy~\cite{Cuturi2013}, or equivalently the minimization of a Kullback-Leibler divergence~\cite{Benamou2015}, and their sparse version~\cite{Kurras2015}, which turn out to be special instances of ROT problems. It also relates to matrix nearness problems via minimization of Bregman divergences, and it is straightforward to construct more general estimators for contingency tables with fixed marginals than the classical estimator based on the Kullback-Leibler divergence~\cite{Dhillon2007}. Lastly, it brings some new insights between transportation theory~\cite{Villani2009} and information geometry~\cite{Amari2000}, where Bregman divergences are known to possess a dually flat structure with a generalized Pythagorean theorem in relation to information projections.

The remainder of this paper is organized as follows. In Section~\ref{sec:preliminaries}, we introduce some necessary preliminaries. In Section~\ref{sec:math}, we present our theoretical results for a unified framework of ROT problems. We then derive the algorithmic methods for solving ROT problems in Section~\ref{sec:alg}. We also discuss the inherent specificities of ROT problems for classical regularizers and associated divergences in Section~\ref{sec:examples}. In Section~\ref{sec:experiments}, we provide experiments to illustrate our methods on synthetic data and real-world audio data in a classification problem. Finally, in Section~\ref{sec:conclusion}, we draw some conclusions and perspectives for future work.

\section{Theoretical Preliminaries}
\label{sec:preliminaries}

In this section, we introduce the required preliminaries to our framework. We begin with elements of convex analysis (Section~\ref{subsec:convex}) and of Bregman geometry (Section~\ref{subsec:bregman}). We proceed with theoretical results for convergence of alternate Bregman projections (Section~\ref{subsec:dykstra}) and of the Newton-Raphson method (Section~\ref{subsec:newton}).

\subsection{Convex Analysis}
\label{subsec:convex}

Let $\E$ be a Euclidean space with inner product $\langle \cdot, \cdot \rangle$ and induced norm $\Vert\cdot\Vert$. The boundary, interior and relative interior of a subset $\X \subseteq \E$ are respectively denoted by $\bd(\X)$, $\interior(\X)$, and $\ri(\X)$, where we recall that for a convex set $\C$, we have:
\begin{equation}
\label{eq:ri}
\ri(\C) = \{\x \in \C \colon \forall \y \in \C, \exists \lambda > 1, \lambda \x + (1 - \lambda) \y \in \C\} \enspace.
\end{equation}

In convex analysis, scalar functions are defined over the whole space $\E$ and take values in the extended real number line $\R \cup \{-\infty, +\infty\}$. The effective domain, or simply domain, of a function $f$ is then defined as the set:
\begin{equation}
\dom f = \{\x \in \E \colon f(\x) < +\infty\} \enspace.
\end{equation}

A convex function $f$ is proper if $f(\x) < +\infty$ for at least one $\x \in \E$ and $f(\x) > -\infty$ for all $\x \in \E$, and it is closed if its lower level sets $\{\x \in \E \colon f(\x) \leq \alpha\}$ are closed for all $\alpha \in \R$. If $\dom f$ is closed, then $f$ is closed, and a proper convex function is closed if and only if it is lower semi-continuous. Moreover, a closed function $f$ is continuous relative to any simplex, polytope or polyhedral subset in $\dom f$. It is also well-known that a convex function $f$ is always continuous in the relative interior $\ri(\dom f)$ of its domain.

A function $f$ is essentially smooth if it is differentiable on $\interior(\dom f) \neq \emptyset$ and verifies $\lim_{k \to +\infty} \Vert \nabla f(\x_k) \Vert = +\infty \enspace$ for any sequence ${(\x_k)}_{k \in \N}$ from $\interior(\dom f)$ that converges to a point $\x \in \bd(\dom f)$. A function $f$ is of Legendre type if it is a closed proper convex function that is also essentially smooth and strictly convex on $\interior(\dom f)$.

The Fenchel conjugate $f^\star$ of a function $f$ is defined for all $\y \in \E$ as follows:
\begin{equation}\label{Fenchel}
f^\star(\y) = \sup_{\x \in \interior(\dom f)} \, \langle \x, \y \rangle - f(\x) \enspace.
\end{equation}
The Fenchel conjugate $f^\star$ is always a closed convex function. Moreover, if $f$ is a closed convex function, then $(f^\star)^\star = f$, and $f$ is of Legendre type if and only if $f^\star$ is of Legendre type. In this latter case, the gradient mapping $\nabla f$ is a homeomorphism between $\interior(\dom f)$ and $\interior(\dom f^\star)$, with inverse mapping ${(\nabla f)}^{-1} = \nabla f^\star$, which guarantees the existence of dual coordinate systems $\x(\y) = \nabla f^\star(\y)$ and $\y(\x) = \nabla f(\x)$ on $\interior(\dom f)$ and $\interior(\dom f^\star)$.

Finally, we say that a function $f$ is cofinite if it verifies:
\begin{equation}\label{cofinite}
\lim_{\lambda \to +\infty} f(\lambda \x) / \lambda = +\infty \enspace,
\end{equation}
for all nonzero $\x \in \E$. Intuitively, it means that $f$ grows super-linearly in every direction. In particular, a closed proper convex function is cofinite if and only if $\dom f^\star = \E$.

\subsection{Bregman Geometry}
\label{subsec:bregman}

Let $\phi$ be a convex function on $\E$ that is differentiable on $\interior(\dom \phi) \neq \emptyset$. The Bregman divergence generated by $\phi$ is defined as follows:
\begin{equation}
\label{eq:bdiv}
B_\phi(\x \Vert \y) = \phi(\x) - \phi(\y) - \langle \x - \y, \nabla\phi(\y) \rangle \enspace,
\end{equation}
for all $\x \in \dom \phi$ and $\y \in \interior(\dom \phi)$. We have $B_\phi(\x \Vert \y) \geq 0$ for any $\x \in \dom \phi$ and $\y \in \interior(\dom \phi)$. If in addition $\phi$ is strictly convex on $\interior(\dom \phi)$, then $B_\phi(\x \Vert \y) = 0$ if and only if $\x = \y$. Bregman divergences are also always convex in the first argument, and are invariant under adding an arbitrary affine term to their generator.

Bregman divergences are not symmetric and do not verify the triangle inequality in general, and thus are not necessarily distances in the strict sense. However, they still enjoy some nice geometrical properties that somehow generalize the Euclidean geometry. In particular, they verify a four-point identity similar to a parallelogram law:
\begin{equation}
B_\phi(\x \Vert \y) + B_\phi(\x' \Vert \y') = B_\phi(\x' \Vert \y) + B_\phi(\x \Vert \y') - \langle \x - \x', \nabla\phi(\y) - \nabla\phi(\y') \rangle \enspace,
\end{equation}
for all $\x, \x' \in \dom \phi$ and $\y, \y' \in \interior(\dom \phi)$. A special instance of this relation gives rise to a three-point property similar to a triangle law of cosines:
\begin{equation}
B_\phi(\x \Vert \y) = B_\phi(\x \Vert \y') + B_\phi(\y' \Vert \y) - \langle \x - \y', \nabla\phi(\y) - \nabla\phi(\y') \rangle \enspace,
\end{equation}
for all $\x \in \dom \phi$ and $\y, \y' \in \interior(\dom \phi)$.

Suppose now that $\phi$ is of Legendre type, and let $\C \subseteq \E$ be a closed convex set such that $\C \cap \interior(\dom \phi) \neq \emptyset$. Then, for any point $\y \in \interior(\dom \phi)$, the following problem:
\begin{equation}
P_\C(\y) = \argmin_{\x \in \C} B_\phi(\x \Vert \y) \enspace,
\end{equation}
has a unique solution, then called the Bregman projection of $\y$ onto $\C$. This solution actually belongs to $\C \cap \interior(\dom \phi)$, and is also characterized as the unique point $\y' \in \C \cap \interior(\dom \phi)$ that verifies the variational relation:
\begin{equation}
\label{eq:var}
\langle \x - \y', \nabla\phi(\y) - \nabla\phi(\y') \rangle \leq 0 \enspace,
\end{equation}
for all $\x \in \C \cap \dom \phi$. This characterization is equivalent to a well-known generalized Pythagorean theorem for Bregman divergences, which states that the Bregman projection of $\y$ onto $\C$ is the unique point $\y' \in \C \cap \interior(\dom \phi)$ that verifies the following inequality:
\begin{equation}
B_\phi(\x \Vert \y) \geq B_\phi(\x \Vert \y') + B_\phi(\y' \Vert \y) \enspace,
\end{equation}
for all $\x \in \C \cap \dom \phi$. When $\C$ is further an affine subspace, or more generally when the Bregman projection further belongs to $\ri(\C)$, the scalar product actually vanishes:
\begin{equation}
\label{eq:orth}
\langle \x - \y', \nabla\phi(\y) - \nabla\phi(\y') \rangle = 0 \enspace,
\end{equation}
leading to an equality in the generalized Pythagorean theorem:
\begin{equation}
\label{eq:pyth}
B_\phi(\x \Vert \y) = B_\phi(\x \Vert \y') + B_\phi(\y' \Vert \y) \enspace.
\end{equation}

A famous example of Bregman divergence is the Kullback-Leibler divergence, defined for matrices $\boldsymbol\pi \in \R_+^{d \times d}$ and $\boldsymbol\xi \in \R_{++}^{d \times d}$ as follows:
\begin{equation}
\label{eq:kl}
K(\boldsymbol\pi \Vert \boldsymbol\xi) = \sum_{i = 1}^d \sum_{j = 1}^d \left(\pi_{ij} \log\left(\frac{\pi_{ij}}{\xi_{ij}}\right) - \pi_{ij} + \xi_{ij}\right) \enspace.
\end{equation}
This divergence is generated by a function of Legendre type for $\boldsymbol\pi \in \R_+^{d \times d}$ given by minus the Boltzmann-Shannon entropy:
\begin{equation}
\label{eq:mbse}
E(\boldsymbol\pi) = K(\boldsymbol\pi \Vert \boldsymbol 1) = \sum_{i = 1}^d \sum_{j = 1}^d \left(\pi_{ij} \log(\pi_{ij}) - \pi_{ij} + 1\right) \enspace,
\end{equation}
with the convention $0 \log(0) = 0$. Another well-known example is the Itakura-Saito divergence, defined for matrices $\boldsymbol\pi, \boldsymbol\xi \in \R_{++}^{d \times d}$ as follows:
\begin{equation}
\label{eq:is}
I(\boldsymbol\pi \Vert \boldsymbol\xi) = \sum_{i = 1}^d \sum_{j = 1}^d \left(\frac{\pi_{ij}}{\xi_{ij}} - \log\left(\frac{\pi_{ij}}{\xi_{ij}}\right) - 1\right) \enspace.
\end{equation}
This divergence is generated by a function of Legendre type for $\boldsymbol\pi \in \R_{++}^{d \times d}$ given by minus the Burg entropy:
\begin{equation}
\label{eq:mbe}
F(\boldsymbol\pi) = \sum_{i = 1}^d \sum_{j = 1}^d \left(\pi_{ij} - \log \pi_{ij} - 1\right) \enspace.
\end{equation}

On the one hand, these examples belong to a particular type of so-called separable Bregman divergences between matrices on $\R^{d \times d}$, that can be seen as the aggregation of element-wise Bregman divergences between scalars on $\R$:
\begin{equation}
B_\phi(\boldsymbol\pi \Vert \boldsymbol\xi) = \sum_{i = 1}^d \sum_{j = 1}^d B_{\phi_{ij}}(\pi_{ij} \Vert \xi_{ij}) \enspace,
\end{equation}
\begin{equation}
\phi(\boldsymbol\pi) = \sum_{i = 1}^d \sum_{j = 1}^d \phi_{ij}(\pi_{ij}) \enspace.
\end{equation}
Often, all element-wise generators $\phi_{ij}$ are chosen equal, and are thus simply written as $\phi$ with a slight abuse of notation. Other examples of such divergences are discussed in Section~\ref{sec:examples}, and include the logistic loss function generated by minus the Fermi-Dirac entropy, or the squared Euclidean distance generated by the Euclidean norm.

On the other hand, a classical example of non-separable Bregman divergence is half the squared Mahalanobis distance, defined for matrices $\boldsymbol\pi, \boldsymbol\xi \in \R^{d \times d}$ as follows:
\begin{equation}
\label{eq:mahal}
M(\boldsymbol\pi \Vert \boldsymbol\xi) = \frac{1}{2} {\vect(\boldsymbol\pi - \boldsymbol\xi)}^\top \P \vect(\boldsymbol\pi - \boldsymbol\xi) \enspace,
\end{equation}
for a positive-definite matrix $\P \in \R^{d^2 \times d^2}$. This divergence is generated by a function of Legendre type for $\boldsymbol\pi \in \R^{d \times d}$ given by a quadratic form:
\begin{equation}
\label{eq:aqc}
Q(\boldsymbol\pi) = \frac{1}{2} {\vect(\boldsymbol\pi)}^\top \P \vect(\boldsymbol\pi) \enspace.
\end{equation}
This example is also discussed in Section~\ref{sec:examples}.

\subsection{Alternate Bregman Projections}
\label{subsec:dykstra}

Let $\phi$ be a function of Legendre type with Fenchel conjugate $\phi^\star = \psi$. In general, computing Bregman projections onto an arbitrary closed convex set $\C \subseteq \E$ such that $\C \cap \interior(\dom \phi) \neq \emptyset$ is nontrivial. Sometimes, it is possible to decompose $\C$ into the intersection of finitely many closed convex sets:
\begin{equation}
\C = \bigcap_{l = 1}^s \C_l \enspace,
\end{equation}
where the individual Bregman projections onto the respective sets $\C_1, \dotsc, \C_s$ are easier to compute. It is then possible to obtain the Bregman projection onto $\C$ by alternate projections onto $\C_1, \dotsc, \C_s$ according to Dykstra's algorithm.

In more detail, let $\sigma \colon \N \to \{1, \dotsc, s\}$ be a control mapping that determines the sequence of subsets onto which we project. For a given point $\x_0 \in \C \cap \interior(\dom \phi)$, the Bregman projection $P_\C(\x_0)$ of $\x_0$ onto $\C$ can be approximated with Dykstra's algorithm by iterating the following updates:
\begin{equation}
\label{dykstra}
\x_{k + 1} \leftarrow P_{\C_{\sigma(k)}}(\nabla\psi(\nabla\phi(\boldsymbol\x_{k}) + \y^{\sigma(k)})) \enspace,
\end{equation}
where the correction terms $\y^1, \dotsc, \y^s$ for the respective subsets are initialized with the null element of $\E$, and are updated after projection as follows:
\begin{equation}
\label{correction}
\y^{\sigma(k)} \leftarrow \y^{\sigma(k)} + \nabla\phi(\x_k) - \nabla\phi(\x_{k + 1}) \enspace.
\end{equation}
Under some technical conditions, the sequence of updates ${(\x_k)}_{k \in \N}$ then converges in norm to $P_\C(\x_0)$ with a linear rate. Several sets of such conditions have been studied, notably by \cite{Tseng1993}, \cite{Bauschke2000}, \cite{Dhillon2007}.

We here use the conditions proposed by \cite{Dhillon2007}, which reveal to be the less restrictive ones in our framework. Specifically, the convergence of Dykstra's algorithm is guaranteed as soon as the function $\phi$ is cofinite, the constraint qualification $\ri(\C_1) \cap \dotsb \cap \ri(\C_s) \cap \interior(\dom \phi) \neq \emptyset$ holds, and the control mapping $\sigma$ is essentially cyclic, that is, there exists a number $t \in \N$ such that $\sigma$ takes each output value at least once during any $t$ consecutive input values. If a given $\C_l$ is a polyhedral set, then the relative interior can be dropped from the constraint qualification. Hence, when all subsets $\C_l$ are polyhedral, the constraint qualification simply reduces to $\C \cap \interior(\dom \phi) \neq \emptyset$, which is already enforced for the definition of Bregman projections. 

Finally, if all subsets $\C_l$ are further affine, then we can relax other assumptions. Notably, we do not require $\phi$ to be cofinite \eqref{cofinite}, or equivalently $\dom\psi = \E$, but only $\dom\psi$ to be open. The control mapping need not be essentially cyclic anymore, as long as it takes each output value an infinite number of times. More importantly, we can completely drop the correction terms from the updates, leading to a simpler technique known as projections onto convex sets (POCS):
\begin{equation}\label{pocs}
\x_{k + 1} \leftarrow P_{\C_{\sigma(k)}}(\boldsymbol\x_{k}) \enspace.
\end{equation}

\subsection{Newton-Raphson Method}
\label{subsec:newton}

Let $f$ be a continuously differentiable scalar function on an open interval $I \subseteq \R$. Assume $f$ is increasing on a non-empty closed interval $[x^-, x^+] \subset I$, and write $y^- = f(x^-)$ and $y^+ = f(x^+)$. Then, for any $y \in [y^-, y^+]$, the equation $f(x) = y$ has at least one solution $x^\star \in [x^-, x^+]$. Such a solution can be approximated by iterative updates according to the Newton-Raphson method:
\begin{equation}
x \leftarrow \max\left\{x^-, \min\left\{x^+, x - \frac{f(x) - y}{f'(x)}\right\}\right\} \enspace,
\end{equation}
where the fraction takes infinite values when $f'(x) = 0$ and $f(x) \neq y$, and a null value by convention when $f'(x) = 0$ and $f(x) = y$. It is well-known that the Newton-Raphson method converges to a solution $x^\star$ as soon as $x$ is initialized sufficiently close to $x^\star$. Convergence is then quadratic provided that $f'(x^\star) \neq 0$. However, this local convergence has little importance in practice because it is hard to quantify the required proximity to the solution.

\cite{Thorlund-Petersen2004} elucidated results on global convergence of the Newton-Raphson method. He proved a necessary and sufficient condition of convergence for an arbitrary value $y \in [y^-, y^+]$ and from any starting point $x \in [x^-, x^+]$. This condition is that for any $a, b \in [x^-, x^+]$, $f(b) > f(a)$ implies:
\begin{equation}
\label{eq:nasc}
f'(a) + f'(b) > \frac{f(b) - f(a)}{b - a} \enspace.
\end{equation}
In particular, a sufficient condition is that the underlying function $f$ is an increasing convex or increasing concave function on $[x^-, x^+]$, or can be decomposed as the sum of such functions. In addition, if $f$ satisfies the necessary and sufficient condition and is strictly increasing with $f'(x) > 0$ for all $x \in [x^-, x^+]$, then initializing with a boundary point $x^- \neq x^\star$ or $x^+ \neq x^\star$ ensures that the entire sequence of updates is interior to $(x^-, x^+)$, so that we can actually drop the min and max truncation operators in the updates:
\begin{equation}
x \leftarrow x - \frac{f(x) - y}{f'(x)} \enspace.
\end{equation}

\section{Mathematical Formulation}
\label{sec:math}

In this section, we develop a unified framework to define ROT problems. We start by drawing some technical assumptions for our generalized framework to hold (Section~\ref{subsec:assumptions}). We then formulate primal ROT problems and study their properties (Section~\ref{subsec:primal}). We also formulate dual ROT problems and discuss their properties in relation to primal ones (Section~\ref{subsec:dual}). Finally, we provide some geometrical insights to summarize our developments in the light of information geometry (Section~\ref{subsec:geomoverview}).

\subsection{Technical Assumptions}
\label{subsec:assumptions}

Some mild technical assumptions are required on the convex regularizer $\phi$ and its Fenchel conjugate $\psi = \phi^*$ for the proposed framework to hold. Some assumptions relate to required conditions for the definition of Bregman projections and convergence of the algorithms, while others are more specific to ROT problems. In our framework, we also need to distinguish between two situations where the underlying closed convex set can be described as the intersection of either affine subspaces or polyhedral subsets. The two sets of assumptions (A) and (B) are summarized in Table~\ref{tab:assump}.

\begin{table}[t!]
\centering
\begin{tabular}{l@{\hspace{0.5cm}}l}
\toprule
(A) Affine constraints 						& (B) Polyhedral constraints\\
\midrule
(A1) $\phi$ is of Legendre type. 				& (B1) $\phi$ is of Legendre type.\\
(A2) ${(0, 1)}^{d \times d} \subseteq \dom \phi$. 	& (B2) ${(0, 1)}^{d \times d} \subseteq \dom \phi$.\\
(A3) $\dom \phi \subseteq \R_+^{d \times d}$.		& (B3) $\dom \phi \nsubseteq \R_+^{d \times d}$.\\
(A4) $\dom \psi$ is open.						& (B4) $\dom \psi = \R^{d \times d}$.\\
(A5) $\R_-^{d \times d} \subset \dom \psi$.		& \\
\bottomrule
\end{tabular}
\caption{Set of assumptions for the considered regularizers $\phi$.}
\label{tab:assump}
\end{table}

For the first assumptions (A1) and (B1), we recall that a closed proper convex function is of Legendre type if and only if it is essentially smooth and strictly convex on the interior of its domain (Section~\ref{subsec:convex}). This is required for the definition of Bregman projections (Section~\ref{subsec:bregman}). In addition, it guarantees the existence of dual coordinate systems on $\interior(\dom \phi)$ and $\interior(\dom \psi)$ via the homeomorphism $\nabla\phi = {\nabla\psi}^{-1}$:
\begin{align}
\label{eq:primalparam}
\boldsymbol\pi(\boldsymbol\theta) & = \nabla\psi(\boldsymbol\theta) \enspace,\\
\label{eq:dualparam}
\boldsymbol\theta(\boldsymbol\pi) & = \nabla\phi(\boldsymbol\pi) \enspace.
\end{align}
With a slight abuse of notation, we omit the reparameterization to simply denote corresponding primal and dual parameters by $\boldsymbol\pi$ and $\boldsymbol\theta$.

The second assumptions (A2) and (B2) imply that $\ri(\Pi(\p, \q)) \subset \dom \phi$ and ensure the constraint qualification $\Pi(\p, \q) \cap \interior(\dom \phi) \neq \emptyset$ for Bregman projection onto the transport polytope, independently of the input distributions $\p, \q$ as long as they do not have null or unit entries. We assume hereafter that this implicitly holds, and discuss in the practical considerations (Section~\ref{subsec:practice}) how our methods actually generalize to deal explicitly with null or unit entries in the input distributions.

The third assumptions (A3) and (B3) separate between two cases depending on whether $\dom \phi$ lies within the non-negative orthant or not for the alternate Bregman projections (Section~\ref{subsec:dykstra}). In the former case, non-negativity is already ensured by the domain of the regularizer, so that the underlying closed convex set is made of two affine subspaces for the row and column sum constraints, and the POCS method can be considered. The fourth assumption (A4) thus requires that $\dom \psi$ be open for convergence of this algorithm. In the latter case, there is one additional polyhedral subset for the non-negative constraints and Dykstra's algorithm should be used. The fourth assumption (B4) hence further requires that $\dom \psi = \R^{d \times d}$, or equivalently that $\phi$ be cofinite \eqref{cofinite}, for convergence. In both cases, we remark that we necessarily have $\dom \psi = \dom \nabla\psi$.

The fifth assumption (A5) in the affine constraints ensures that $-\boldsymbol\gamma / \lambda$ belongs to $\dom \nabla\psi$ for definition of ROT problems, independently of the non-negative cost matrix $\boldsymbol\gamma$ and positive regularization term $\lambda$. Notice that this is already guaranteed by the fourth assumption in the polyhedral constraints. We also show in the sparse extension (Section~\ref{subsec:sparse}) how to deal with infinite entries in the cost matrix $\boldsymbol\gamma$ for separable regularizers, so as to enforce null entries in the regularized optimal plan.

On the one hand, some common regularizers under assumptions (A) are the Boltzmann-Shannon entropy associated to the Kullback-Leibler divergence, the Burg entropy associated to the Itakura-Saito divergence, and the Fermi-Dirac entropy associated to the logistic loss function. To solve the underlying ROT problems, we employ our method called ASA based on the POCS technique, where alternate Bregman projections onto the two affine subspaces for the row and column sum constraints are considered (Section~\ref{subsec:asa}). On the other hand, examples under assumptions (B) include the Euclidean norm associated to the Euclidean distance, and the quadratic form associated to the Mahalanobis distance. For these ROT problems, we use our second method called NASA based on Dykstra's algorithm, where correction terms and a further Bregman projection onto the polyhedral non-negative orthant are needed (Section~\ref{subsec:nasa}).

\subsection{Primal Problem}
\label{subsec:primal}

We start our primal formulation with the following lemmas and definition for the RMD.

\begin{lemma}
\label{lemma:primal_glob_min}
The regularizer $\phi$ attains its global minimum uniquely at $\boldsymbol\xi' = \nabla\psi(\zeros)$.
\end{lemma}

\begin{proof}
Using the assumptions (A4) and (A5), respectively (B4), we have that $\zeros \in \dom \psi = \interior(\dom \psi)$. Thus, there exists a unique $\boldsymbol\xi' \in \interior(\dom \phi)$ such that $\nabla\phi(\boldsymbol\xi') = \zeros$, or equivalently $\boldsymbol\xi' = \nabla\psi(\zeros)$, via the homeomorphism $\nabla\phi = {\nabla\psi}^{-1}$ ensured by assumption (A1), respectively (B1). Hence, $\phi$ attains its global minimum uniquely at $\boldsymbol\xi'$ by strict convexity on $\interior(\dom \phi)$.
\end{proof}

\begin{lemma}
\label{lemma:primal_rest_min}
The restriction of the regularizer $\phi$ to the transport polytope $\Pi(\p, \q)$ attains its global minimum uniquely at the Bregman projection $\boldsymbol\pi'$ of $\boldsymbol\xi'$ onto $\Pi(\p, \q)$.
\end{lemma}

\begin{proof}
Using the assumption (A2), respectively (B2), we have that $\Pi(\p, \q) \cap \interior(\dom \phi) \neq \emptyset$. Since $\boldsymbol\xi' \in \interior(\dom \phi)$ and $\Pi(\p, \q)$ is a closed convex set, the Bregman projection $\boldsymbol\pi'$ of $\boldsymbol\xi'$ onto $\Pi(\p, \q)$ according to the function $\phi$ of Legendre type is well-defined. Moreover, it is characterized by the variational relation~\eqref{eq:var} as follows:
\begin{equation}
\langle \boldsymbol\pi - \boldsymbol\pi', \nabla\phi(\boldsymbol\pi') \rangle \geq 0 \enspace,
\end{equation}
for all $\boldsymbol\pi \in \Pi(\p, \q) \cap \dom \phi$. We also have $B_\phi(\boldsymbol\pi \Vert \boldsymbol\pi') > 0$ when $\boldsymbol\pi \neq \boldsymbol\pi'$ by strict convexity of $\phi$ on $\interior(\dom \phi)$. As a result, we have:
\begin{equation}
\phi(\boldsymbol\pi) - \phi(\boldsymbol\pi') > \langle \boldsymbol\pi - \boldsymbol\pi', \nabla\phi(\boldsymbol\pi') \rangle \enspace.
\end{equation}
Combining the two inequalities, we obtain $\phi(\boldsymbol\pi) > \phi(\boldsymbol\pi')$ and the restriction of $\phi$ to $\Pi(\p, \q)$ attains its global minimum uniquely at $\boldsymbol\pi'$.
\end{proof}

\begin{lemma}
\label{lemma:rtp}
The restriction of the cost $\langle \cdot, \boldsymbol\gamma \rangle$ to the regularized transport polytope:
\begin{equation}
\label{eq:rtp}
\Pi_{\alpha, \phi}(\p, \q) = \{\boldsymbol\pi \in \Pi(\p, \q) \colon \phi(\boldsymbol\pi) \leq \phi(\boldsymbol\pi') + \alpha\} \enspace,
\end{equation}
where $\alpha \geq 0$, attains its global minimum.
\end{lemma}

\begin{proof}
The regularized transport polytope is the intersection of the compact set $\Pi(\p, \q)$ with a lower level set of $\phi$ which is also closed since $\phi$ is closed. Hence, $\Pi_{\alpha, \phi}(\p, \q)$ is compact and the restriction of $\langle \cdot, \boldsymbol\gamma \rangle$ to $\Pi_{\alpha, \phi}(\p, \q)$ attains its global minimum by continuity on a compact set.
\end{proof}

\begin{definition}
\label{def:prmd}
The primal rot mover's distance is the quantity defined as:
\begin{equation}
\label{eq:prmd}
d_{\boldsymbol\gamma, \alpha, \phi}'(\p, \q) = \min_{\boldsymbol\pi \in \Pi_{\alpha, \phi}(\p, \q)} \langle \boldsymbol\pi, \boldsymbol\gamma \rangle \enspace.
\end{equation}
A minimizer ${\boldsymbol\pi'}\vphantom{\boldsymbol\pi}_\alpha^\star$ is then called a primal rot mover's plan.
\end{definition}

\begin{myremark}
For the sake of notation, we omit the dependence on $\p, \q, \boldsymbol\gamma, \phi$ in the index of primal rot mover's plans ${\boldsymbol\pi'}\vphantom{\boldsymbol\pi}_\alpha^\star$.
\end{myremark}

The regularization enforces the associated minimizers to have small enough Bregman information $\phi({\boldsymbol\pi'}\vphantom{\boldsymbol\pi}_\alpha^\star) \leq \phi(\boldsymbol\pi') + \alpha$ compared to the minimal one $\phi(\boldsymbol\pi')$ for transport plans. We also have a geometrical interpretation where the solutions are constrained to a Bregman ball whose center $\boldsymbol\xi'$ is the matrix with minimal Bregman information.

\begin{proposition}
\label{prop:bigbball}
The regularized transport polytope is the intersection of the transport polytope with the Bregman ball of radius $B_\phi(\boldsymbol\pi' \Vert \boldsymbol\xi') + \alpha$ and center $\boldsymbol\xi'$:
\begin{equation}
\label{eq:bigbball}
\Pi_{\alpha, \phi}(\p, \q) = \{\boldsymbol\pi \in \Pi(\p, \q) \colon B_\phi(\boldsymbol\pi \Vert \boldsymbol\xi') \leq B_\phi(\boldsymbol\pi' \Vert \boldsymbol\xi') + \alpha\} \enspace.
\end{equation}
\end{proposition}

\begin{proof}
Expanding the Bregman divergences from their definition~\eqref{eq:bdiv}, we obtain:
\begin{align}
B_\phi(\boldsymbol\pi \Vert \boldsymbol\xi')	& = \phi(\boldsymbol\pi) - \phi(\boldsymbol\xi') - \langle \boldsymbol\pi - \boldsymbol\xi', \nabla\phi(\boldsymbol\xi') \rangle\enspace,\\ 
B_\phi(\boldsymbol\pi' \Vert \boldsymbol\xi') 	& = \phi(\boldsymbol\pi') - \phi(\boldsymbol\xi') - \langle \boldsymbol\pi' - \boldsymbol\xi', \nabla\phi(\boldsymbol\xi') \rangle\enspace.
\end{align}
Since $\nabla\phi(\boldsymbol\xi') = \zeros$, the last terms with scalar products vanish, leading to: 
\begin{equation}
\phi(\boldsymbol\pi) - \phi(\boldsymbol\pi') = B_\phi(\boldsymbol\pi \Vert \boldsymbol\xi') - B_\phi(\boldsymbol\pi' \Vert \boldsymbol\xi') \enspace.
\end{equation}
Therefore, in the definition~\eqref{eq:rtp} of $\Pi_{\alpha, \phi}(\p, \q)$, we have $\phi(\boldsymbol\pi) \leq \phi(\boldsymbol\pi') + \alpha$ if and only if $\boldsymbol\pi$ is in the Bregman ball of radius $B_\phi(\boldsymbol\pi' \Vert \boldsymbol\xi') + \alpha$ and center $\boldsymbol\xi'$.
\end{proof}

Under some additional conditions, this geometrical interpretation still holds with a Bregman ball whose center $\boldsymbol\pi'$ has minimal Bregman information for transport plans.

\begin{proposition}
\label{prop:bball}
If $\boldsymbol\pi' \in \ri(\Pi(\p, \q))$, then the regularized transport polytope is the intersection of the transport polytope with the Bregman ball of radius $\alpha$ and center $\boldsymbol\pi'$:
\begin{equation}
\Pi_{\alpha, \phi}(\p, \q) = \{\boldsymbol\pi \in \Pi(\p, \q) \colon B_\phi(\boldsymbol\pi \Vert \boldsymbol\pi') \leq \alpha\} \enspace.
\end{equation}
\end{proposition}

\begin{proof}
Since $\boldsymbol\pi' \in \ri(\Pi(\p, \q))$, there is equality in the generalized Pythagorean theorem~\eqref{eq:pyth}:
\begin{equation}
B_\phi(\boldsymbol\pi \Vert \boldsymbol\xi') = B_\phi(\boldsymbol\pi \Vert \boldsymbol\pi') + B_\phi(\boldsymbol\pi' \Vert \boldsymbol\xi') \enspace.
\end{equation}
The regularized transport polytope as seen from~\eqref{eq:bigbball} is then the intersection of the transport polytope $\Pi(\p, \q)$ with the Bregman ball of radius $\alpha$ and center $\boldsymbol\pi'$.
\end{proof}

\begin{myremark}
The proposition also holds trivially when the global minimum is attained on the transport polytope, that is, when $\boldsymbol\xi' = \boldsymbol\pi'$.
\end{myremark}

\begin{corollary}
\label{corol:bball}
Under assumptions (A), the regularized transport polytope is the intersection of the transport polytope with the Bregman ball of radius $\alpha$ and center $\boldsymbol\pi'$:
\begin{equation}
\Pi_{\alpha, \phi}(\p, \q) = \{\boldsymbol\pi \in \Pi(\p, \q) \colon B_\phi(\boldsymbol\pi \Vert \boldsymbol\pi') \leq \alpha\} \enspace.
\end{equation}
\end{corollary}

\begin{proof}
This is a result of $\boldsymbol\pi' \in \Pi(\p, \q) \cap \interior(\dom \phi) = \ri(\Pi(\p, \q))$ when $\dom \phi \subseteq \R_+^{d \times d}$. Indeed, we then have $\ri(\Pi(\p, \q)) \subset \Pi(\p, \q)$ and $\ri(\Pi(\p, \q)) \subset \interior(\dom \phi)$, so that $\ri(\Pi(\p, \q)) \subseteq \Pi(\p, \q) \cap \interior(\dom \phi)$. Conversely, let $\boldsymbol\pi \in \Pi(\p, \q) \cap \interior(\dom \phi)$ so that $\boldsymbol\pi \in \R_{++}^{d \times d}$. Then, for a given $\overline{\boldsymbol\pi} \in \Pi(\p, \q)$, let us pose $\boldsymbol\pi_\lambda = \lambda \boldsymbol\pi + (1 - \lambda) \overline{\boldsymbol\pi}$ for $\lambda > 1$. We easily have $\boldsymbol\pi_\lambda \ones = \p$ and $\boldsymbol\pi_\lambda^\top \ones = \q$. Moreover, since all entries of $\boldsymbol\pi$ are positive and that of $\overline{\boldsymbol\pi}$ are non-negative, we can always choose a given $\lambda$ sufficiently close to $1$ such that $\boldsymbol\pi_\lambda \in \R_+^{d \times d}$. We then have $\boldsymbol\pi_\lambda \in \Pi(\p, \q)$ so that $\boldsymbol\pi \in \ri(\Pi(\p, \q))$ as characterized by~\eqref{eq:ri}, and thus $\Pi(\p, \q) \cap \interior(\dom \phi) \subseteq \ri(\Pi(\p, \q))$.
\end{proof}

\begin{myremark}
Under assumptions (B), the Bregman projection $\boldsymbol\pi'$ does not necessarily lie within $\ri(\Pi(\p, \q))$. Hence, the geometrical interpretation in terms of a Bregman ball might break down, although the solutions are still constrained to have a small enough Bregman information above that of $\boldsymbol\pi'$.
\end{myremark}

Although Sinkhorn distances verify the triangular inequality when $\boldsymbol\gamma$ is a distance matrix, thanks to specific chain rules and information inequalities for the Bolzmann-Shannon entropy and Kullback-Leibler divergence, it is not necessarily the case for the RMD with other regularizations, even for separable regularizers. Hence, the RMD does not provide a true distance metric on $\Sigma_d$ in general even if $\boldsymbol\gamma$ is a distance matrix. Nonetheless, the RMD is symmetric as soon as $\phi$ is invariant by transposition, which holds for separable regularizers $\phi_{ij} = \phi$, and $\boldsymbol\gamma$ is symmetric. We now study some properties of the RMD that hold for general regularizers.

\begin{property}
The primal rot mover's distance $d_{\boldsymbol\gamma, \alpha, \phi}'(\p, \q)$ is a decreasing convex and continuous function of $\alpha$.
\end{property}

\begin{proof}
The fact that it is decreasing is a direct consequence of the regularized transport polytope $\Pi_{\alpha, \phi}(\p, \q)$ growing with $\alpha$. The convexity can be proved as follows. Let $\alpha_0, \alpha_1 \geq 0$, and $0 < \lambda < 1$. We pose $\alpha_\lambda = (1 - \lambda) \alpha_0 + \lambda \alpha_1 \geq 0$. We also choose arbitrary rot mover's plans ${\boldsymbol\pi'}\vphantom{\boldsymbol\pi}_{\alpha_0}^\star, {\boldsymbol\pi'}\vphantom{\boldsymbol\pi}_{\alpha_1}^\star, {\boldsymbol\pi'}\vphantom{\boldsymbol\pi}_{\alpha_\lambda}^\star$. We finally pose $\boldsymbol\pi_\lambda = (1 - \lambda) {\boldsymbol\pi'}\vphantom{\boldsymbol\pi}_{\alpha_0}^\star + \lambda {\boldsymbol\pi'}\vphantom{\boldsymbol\pi}_{\alpha_1}^\star$. By convexity of $\phi$, we have:
\begin{align}
\phi(\boldsymbol\pi_\lambda)	& \leq (1 - \lambda) \phi({\boldsymbol\pi'}\vphantom{\boldsymbol\pi}_{\alpha_0}^\star) + \lambda \phi({\boldsymbol\pi'}\vphantom{\boldsymbol\pi}_{\alpha_1}^\star)\\
						& \leq (1 - \lambda) (\alpha_0 + \phi(\boldsymbol\pi')) + \lambda (\alpha_1 + \phi(\boldsymbol\pi'))\\
						& = \alpha_\lambda + \phi(\boldsymbol\pi') \enspace.
\end{align}
Hence, $\boldsymbol\pi_\lambda \in \Pi_{\alpha_\lambda, \phi}(\p, \q)$, and by construction we have $\langle {\boldsymbol\pi'}\vphantom{\boldsymbol\pi}_{\alpha_\lambda}^\star, \boldsymbol\gamma \rangle \leq \langle \boldsymbol\pi_\lambda, \boldsymbol\gamma \rangle$, or equivalently:
\begin{equation}
\langle {\boldsymbol\pi'}\vphantom{\boldsymbol\pi}_{\alpha_\lambda}^\star, \boldsymbol\gamma \rangle \leq (1 - \lambda) \langle {\boldsymbol\pi'}\vphantom{\boldsymbol\pi}_{\alpha_0}^\star, \boldsymbol\gamma \rangle + \lambda \langle {\boldsymbol\pi'}\vphantom{\boldsymbol\pi}_{\alpha_1}^\star, \boldsymbol\gamma \rangle \enspace.
\end{equation}
The continuity for $\alpha > 0$ is a direct consequence of convexity for $\alpha > 0$, since a convex function is always continuous on the relative interior of its domain. Lastly, the continuity at $\alpha = 0$ can be seen as follows. Let ${(\alpha_k)}_{k \in \N}$ be a sequence of positive numbers that converges to $0$. We choose arbitrary rot mover's plans ${({\boldsymbol\pi'}\vphantom{\boldsymbol\pi}_{\alpha_k}^\star)}_{k \in \N}$. By compactness of $\Pi(\p, \q)$, we can extract a subsequence of rot mover's plans that converges in norm to a point ${\boldsymbol\pi'}\vphantom{\boldsymbol\pi}^\star \in \Pi(\p, \q)$. For the sake of simplicity, we do not relabel this subsequence. By construction, we have $\phi(\boldsymbol\pi') \leq \phi({\boldsymbol\pi'}\vphantom{\boldsymbol\pi}_{\alpha_k}^\star) \leq \phi(\boldsymbol\pi') + \alpha_k$, and $\phi({\boldsymbol\pi'}\vphantom{\boldsymbol\pi}_{\alpha_k}^\star)$ converges to $\phi(\boldsymbol\pi')$. By lower semi-continuity of $\phi$, we thus have $\phi({\boldsymbol\pi'}\vphantom{\boldsymbol\pi}^\star) \leq \phi(\boldsymbol\pi')$. Since the global minimum of $\phi$ on $\Pi(\p, \q)$ is attained uniquely at $\boldsymbol\pi'$, we must have ${\boldsymbol\pi'}\vphantom{\boldsymbol\pi}^\star = \boldsymbol\pi'$, and the original sequence also converges in norm to $\boldsymbol\pi'$. By continuity of the total cost $\langle \cdot, \boldsymbol\gamma \rangle$ on $\R^{d \times d}$, $\langle {\boldsymbol\pi'}\vphantom{\boldsymbol\pi}_{\alpha_k}^\star, \boldsymbol\gamma \rangle$ converges to $\langle \boldsymbol\pi', \boldsymbol\gamma \rangle$. Hence, the limit of the RMD when $\alpha$ tends to $0$ from above is $\langle \boldsymbol\pi', \boldsymbol\gamma \rangle$, which equals the RMD for $\alpha = 0$ as shown in the next property.
\end{proof}

\begin{property}
When $\alpha = 0$, the primal rot mover's distance reduces to:
\begin{equation}
d_{\boldsymbol\gamma, 0, \phi}'(\p, \q) = \langle \boldsymbol\pi', \boldsymbol\gamma \rangle \enspace,
\end{equation}
and the unique primal rot mover's plan is the transport plan with minimal Bregman information:
\begin{equation}
{\boldsymbol\pi'}\vphantom{\boldsymbol\pi}_0^\star = \boldsymbol\pi' \enspace.
\end{equation}
\end{property}

\begin{proof}
Since $\boldsymbol\pi'$ is the unique global minimizer of $\phi$ on $\Pi(\p, \q)$, the regularized transport polytope reduces to the singleton $\Pi_{0, \phi}(\p, \q) = \{\boldsymbol\pi \in \Pi(\p, \q) \colon \phi(\boldsymbol\pi) \leq \phi(\boldsymbol\pi')\} = \{\boldsymbol\pi'\}$. The property follows immediately.
\end{proof}

\begin{property}
When $\alpha$ tends to $+\infty$, the primal rot mover's distance converges to the earth mover's distance:
\begin{equation}
\lim_{\alpha \to +\infty} d_{\boldsymbol\gamma, \alpha, \phi}'(\p, \q) = d_{\boldsymbol\gamma}(\p, \q) \enspace.
\end{equation}
\end{property}

\begin{proof}
Let $\boldsymbol\pi^\star \in \Pi(\p, \q)$ be an earth mover's plan so that $d_{\boldsymbol\gamma}(\p, \q) = \langle \boldsymbol\pi^\star, \boldsymbol\gamma \rangle$. By continuity of the total cost $\langle \cdot, \boldsymbol\gamma \rangle$ on $\R^{d \times d}$, we have that for all $\epsilon > 0$, there exists an open neighborhood of $\boldsymbol\pi^\star$ such that $\langle \boldsymbol\pi, \boldsymbol\gamma \rangle \leq \langle \boldsymbol\pi^\star, \boldsymbol\gamma \rangle + \epsilon$ for any transport plan $\boldsymbol\pi$ within this neighborhood. We can always choose a transport plan such that $\boldsymbol\pi \in \ri(\Pi(\p, \q))$. Since $\ri(\Pi(\p, \q)) \subset \dom \phi$, $\phi(\boldsymbol\pi)$ is finite and we can fix $\alpha_\epsilon = \phi(\boldsymbol\pi) - \phi(\boldsymbol\pi') \geq 0$. Hence, $\boldsymbol\pi \in \Pi_{\alpha, \phi}(\p, \q)$ for any $\alpha \geq \alpha_\epsilon$, and we have $d_{\boldsymbol\gamma}(\p, \q) \leq d_{\boldsymbol\gamma, \alpha, \phi}'(\p, \q) \leq \langle \boldsymbol\pi, \boldsymbol\gamma \rangle \leq d_{\boldsymbol\gamma}(\p, \q) + \epsilon$. 
\end{proof}

\begin{property}
If ${[0, 1)}^{d \times d} \subseteq \dom \phi$, then there exists a minimal $\alpha' \geq 0$ such that for all $\alpha \geq \alpha'$, the primal rot mover's distance reduces to the earth mover's distance:
\begin{equation}
d_{\boldsymbol\gamma, \alpha, \phi}'(\p, \q) = d_{\boldsymbol\gamma}(\p, \q) \enspace.
\end{equation}
\end{property}

\begin{proof}
The extra condition guarantees that $\Pi(\p, \q) \subset \dom \phi$, and thus that $\phi$ is bounded on the closed set $\Pi(\p, \q)$. The property is then a direct consequence of $\Pi_{\alpha, \phi}(\p, \q) = \Pi(\p, \q)$ for $\alpha$ large enough.
\end{proof}

\begin{property}
If ${[0, 1)}^{d \times d} \subseteq \dom \phi$ and $\phi$ is strictly convex on ${[0, 1)}^{d \times d}$, then the unique primal rot mover's plan for $\alpha = \alpha'$ is the earth mover's plan $\boldsymbol\pi_0^\star$ with minimal Bregman information:
\begin{equation}
{\boldsymbol\pi'}\vphantom{\boldsymbol\pi}_{\alpha'}^\star = \boldsymbol\pi_0^\star \enspace.
\end{equation}
\end{property}

\begin{proof}
First, we recall that the set of earth mover's plans $\boldsymbol\pi^\star$ is either a single vertex or a whole facet of $\Pi(\p, \q)$. Hence, it forms a closed convex subset in $\Pi(\p, \q)$, and there is a unique earth mover's plan $\boldsymbol\pi_0^\star$ with minimal Bregman information by strict convexity of $\phi$ on this subset. Second, it is trivial that all primal rot mover's plan ${\boldsymbol\pi'}\vphantom{\boldsymbol\pi}_{\alpha'}^\star$ must be earth mover's plans. If there is a single vertex as earth mover's plan, then the property follows immediately. Otherwise, we can see the property geometrically as follows. The whole facet of earth mover's plans is orthogonal to $\boldsymbol\gamma$. Nevertheless, by strict convexity of $\phi$ on ${[0, 1)}^{d \times d}$, the facet must be tangent to $\Pi_{\alpha', \phi}(\p, \q)$ at the unique earth mover's plan $\boldsymbol\pi_0^\star$ with minimal Bregman information $\phi(\boldsymbol\pi_0^\star) = \phi(\boldsymbol\pi') + \alpha'$, and $\boldsymbol\pi_0^\star$ is also the rot mover's plan ${\boldsymbol\pi'}\vphantom{\boldsymbol\pi}_{\alpha'}^\star$. Another way to prove the property more formally is as follows. Suppose that a primal rot mover's plan $\boldsymbol\pi^\star$ is not the earth mover's plan with minimal Bregman information. We thus have $\phi(\boldsymbol\pi_0^\star) < \phi(\boldsymbol\pi^\star) \leq \phi(\boldsymbol\pi') + \alpha'$. We can then choose a smaller $\alpha'$ such that $\phi(\boldsymbol\pi_0^\star) \leq \phi(\boldsymbol\pi') + \alpha'$ and the RMD still equals the EMD for this smaller value, and actually all values in between by monotonicity. This leads to a contradiction and $\boldsymbol\pi_0^\star$ must be the earth mover's plan with minimal Bregman information.
\end{proof}

\begin{myremark}
When $\alpha > \alpha'$, the regularized transport polytope might grow to include several earth mover's plans with different Bregman information, which are then all minimizers for the RMD. When we do not have strict convexity outside ${(0, 1)}^{d \times d}$, there might also be multiple earth mover's plans with minimal Bregman information.
\end{myremark}

If ${[0, 1)}^{d \times d} \subseteq \dom \phi$, then it is easy to check that the strict convexity of $\phi$ on ${[0, 1)}^{d \times d}$ is always verified when $\phi$ is separable under assumptions (A) or (B), or when ${[0, 1)}^{d \times d} \subset \interior(\dom \phi)$ under assumptions (B). This holds for almost all typical regularizers, notably for all regularizers considered in this paper except from minus the Burg entropy as defined in~\eqref{eq:mbe} and associated to the Itakura-Saito divergence in~\eqref{eq:is}. For this latter regularizer, the solutions for an increasing $\alpha$ all lie within $\ri(\Pi(\p, \q))$, and the RMD never reaches the EMD. In such cases where the minimal $\alpha'$ does not exist, we can use the convention $\alpha' = +\infty$ since the RMD always converges to the EMD in the limit when $\alpha$ tends to $+\infty$. We can then prove that there is a unique rot mover's plan ${\boldsymbol\pi'}\vphantom{\boldsymbol\pi}_\alpha^\star$ as long as $0 < \alpha < \alpha'$, which can be seen informally as follows. The solutions geometrically lie at the intersection of $\Pi_{\alpha, \phi}(\p, \q)$ and of a supporting hyperplane with normal $\boldsymbol\gamma$. By strict convexity of $\phi$ on $\ri(\Pi(\p, \q))$, this intersection is a singleton inside the polytope. When the intersection reaches a facet, the only facet that can coincide locally with the hyperplane is the one that contains the earth mover's plans. Hence, we also have a singleton on the boundary of the polytope before reaching an earth mover's plan. We formally prove this uniqueness result next by exploiting duality.

\subsection{Dual Problem}
\label{subsec:dual}

We now present the following two lemmas before defining our dual formulation for the RMD.

\begin{lemma}
\label{lemma:xi}
The regularized cost $\langle \cdot, \boldsymbol\gamma \rangle + \lambda \phi(\cdot)$, where $\lambda > 0$, attains its global minimum uniquely at $\boldsymbol\xi = \nabla\psi(-\boldsymbol\gamma / \lambda)$.
\end{lemma}

\begin{proof}
The regularized cost is convex with same domain as $\phi$, and is strictly convex on $\interior(\dom \phi)$. Thus, it attains its global minimum at a unique point $\boldsymbol\xi \in \interior(\dom \phi)$ if and only if $\boldsymbol\gamma + \lambda \nabla\phi(\boldsymbol\xi) = \zeros$, or equivalently $\nabla\phi(\boldsymbol\xi) = -\boldsymbol\gamma / \lambda$. By assumptions (A4) and (A5), respectively (B4), $-\boldsymbol\gamma / \lambda \in \dom \nabla\psi$, so that the global minimum is attained uniquely at $\boldsymbol\xi = \nabla\psi(-\boldsymbol\gamma / \lambda)$ in virtue of the homeomorphism in~\eqref{eq:primalparam} and~\eqref{eq:dualparam}.
\end{proof}

\begin{lemma}
\label{lemma:dual_rest_min}
The restriction of the regularized cost $\langle \cdot, \boldsymbol\gamma \rangle + \lambda \phi(\cdot)$ to the transport polytope $\Pi(\p, \q)$ attains its global minimum uniquely.
\end{lemma}

\begin{proof}
We notice that the regularized cost is equal to a Bregman divergence up to a positive factor and additive constant:
\begin{equation}
\label{eq:lemmaeq}
\langle \boldsymbol\pi, \boldsymbol\gamma \rangle + \lambda \phi(\boldsymbol\pi) - \lambda \phi(\boldsymbol\xi) = \lambda B_\phi(\boldsymbol\pi \Vert \boldsymbol\xi) \enspace.
\end{equation}
Hence, its minimization over the closed convex set $\Pi(\p, \q)$ is equivalent to the Bregman projection of $\boldsymbol\xi \in \interior(\dom \phi)$ onto $\Pi(\p, \q)$ according to the function $\phi$ of Legendre type. Since $\Pi(\p, \q) \cap \interior(\dom \phi) \neq \emptyset$, this projection exists and is unique.
\end{proof}

\begin{definition}
\label{def:drmd}
The dual rot mover's distance is the quantity defined as:
\begin{equation}
d_{\boldsymbol\gamma, \lambda, \phi}(\p, \q) = \langle \boldsymbol\pi_\lambda^\star, \boldsymbol\gamma \rangle \enspace,
\end{equation}
where the dual rot mover's plan $\boldsymbol\pi_\lambda^\star$ is given by:
\begin{equation}
\label{eq:drmp}
\boldsymbol\pi_\lambda^\star = \argmin_{\boldsymbol\pi \in \Pi(\p, \q)} \, \langle \boldsymbol\pi, \boldsymbol\gamma \rangle + \lambda \phi(\boldsymbol\pi) \enspace.
\end{equation}
\end{definition}

\begin{myremark}
For the sake of notation, we omit the dependence on $\p, \q, \boldsymbol\gamma, \phi$ in the index of dual rot mover's plans $\boldsymbol\pi_\lambda^\star$.
\end{myremark}

We proceed with the following proposition that enlightens the relation between the RMD and associated Bregman divergence.

\begin{proposition}
\label{prop:dbproj}
The dual rot mover's plan is the Bregman projection of $\boldsymbol\xi$ onto the transport polytope:
\begin{equation}
\label{rotbp}
\boldsymbol\pi_\lambda^\star = \argmin_{\boldsymbol\pi \in \Pi(\p, \q)} B_\phi(\boldsymbol\pi \Vert \boldsymbol\xi) \enspace.
\end{equation}
\end{proposition}

\begin{proof}
This is a consequence of the proof for Lemma~\ref{lemma:dual_rest_min}. Indeed, from the definition in \eqref{eq:drmp}, we see that the rot mover's plan also minimizes \eqref{eq:lemmaeq}. Therefore, it is the unique Bregman projection of $\boldsymbol\xi$ onto the transport polytope.
\end{proof}

We have a geometrical interpretation where the regularization shrinks the solution toward the matrix $\boldsymbol\xi'$ that has minimal Bregman information.

\begin{proposition}
\label{prop:dbigbball}
The dual rot mover's plan $\boldsymbol\pi_\lambda^\star$ can be obtained as:
\begin{equation}
\label{eq:drmpalt}
\boldsymbol\pi_\lambda^\star = \argmin_{\boldsymbol\pi \in \Pi(\p, \q)} \, \langle \boldsymbol\pi, \boldsymbol\gamma \rangle + \lambda B_\phi(\boldsymbol\pi \Vert \boldsymbol\xi') \enspace.
\end{equation}
\end{proposition}

\begin{proof}
Developing the Bregman divergence based on its definition~\eqref{eq:bdiv}, we have:
\begin{equation}
B_\phi(\boldsymbol\pi \Vert \boldsymbol\xi') = \phi(\boldsymbol\pi) - \phi(\boldsymbol\xi') - \langle \boldsymbol\pi - \boldsymbol\xi', \nabla\phi(\boldsymbol\xi') \rangle \enspace.
\end{equation}
Since $\nabla\phi(\boldsymbol\xi') = \zeros$, the last term with scalar product vanishes and we are left out with $\phi(\boldsymbol\pi)$ plus a constant term with respect to $\boldsymbol\pi$. Hence, we can replace $\phi(\boldsymbol\pi)$ by $B_\phi(\boldsymbol\pi \Vert \boldsymbol\xi')$ in the minimization~\eqref{eq:drmp} that defines $\boldsymbol\pi_\lambda^\star$.
\end{proof}

Under some additional conditions, this interpretation can also be seen as shrinking toward the transport plan $\boldsymbol\pi'$ with minimal Bregman information.

\begin{proposition}
\label{prop:dbball}
If $\boldsymbol\pi' \in \ri(\Pi(\p, \q))$, then the dual rot mover's plan $\boldsymbol\pi_\lambda^\star$ can be obtained as:
\begin{equation}
\boldsymbol\pi_\lambda^\star = \argmin_{\boldsymbol\pi \in \Pi(\p, \q)} \, \langle \boldsymbol\pi, \boldsymbol\gamma \rangle + \lambda B_\phi(\boldsymbol\pi \Vert \boldsymbol\pi') \enspace.
\end{equation}
\end{proposition}

\begin{proof}
If $\boldsymbol\pi' \in \ri(\Pi(\p, \q))$, then we have equality in the generalized Pythagorean theorem~\eqref{eq:pyth}, leading to:
\begin{equation}
B_\phi(\boldsymbol\pi \Vert \boldsymbol\xi') = B_\phi(\boldsymbol\pi \Vert \boldsymbol\pi') + B_\phi(\boldsymbol\pi' \Vert \boldsymbol\xi')\enspace.
\end{equation}
Since the last term is constant with respect to $\boldsymbol\pi$, we can replace $B_\phi(\boldsymbol\pi \Vert \boldsymbol\xi')$ by $B_\phi(\boldsymbol\pi \Vert \boldsymbol\pi')$ in the minimization~\eqref{eq:drmpalt} that characterizes $\boldsymbol\pi_\lambda^\star$.
\end{proof}

\begin{myremark}
The proposition also holds trivially when the global minimum is attained on the transport polytope, that is, when $\boldsymbol\xi' = \boldsymbol\pi'$.
\end{myremark}

\begin{corollary}
\label{corol:dbball}
Under assumptions (A), the dual rot mover's plan $\boldsymbol\pi_\lambda^\star$ can be obtained as:
\begin{equation}
\boldsymbol\pi_\lambda^\star = \argmin_{\boldsymbol\pi \in \Pi(\p, \q)} \, \langle \boldsymbol\pi, \boldsymbol\gamma \rangle + \lambda B_\phi(\boldsymbol\pi \Vert \boldsymbol\pi') \enspace.
\end{equation}
\end{corollary}

\begin{proof}
This is a result of $\boldsymbol\pi' \in \Pi(\p, \q) \cap \interior(\dom \phi) = \ri(\Pi(\p, \q))$ when $\dom \phi \subseteq \R_+^{d \times d}$, as shown in the proof of Corollary~\ref{corol:bball}.
\end{proof}

In the sequel, we also extend naturally the definition of the dual RMD for $\lambda = 0$ as the EMD. We then do not necessarily have uniqueness of dual rot mover's plans for $\lambda = 0$, and the geometrical interpretation in terms of a Bregman projection does not hold anymore for $\lambda = 0$. However, we have the following theorem based on duality theory that shows the equivalence between primal and dual ROT problems.

\begin{theorem}
\label{thm:duality}
For all $\alpha > 0$, there exists $\lambda \geq 0$ such that the primal and dual rot mover's distances are equal:
\begin{equation}
d_{\boldsymbol\gamma, \alpha, \phi}'(\p, \q) = d_{\boldsymbol\gamma, \lambda, \phi}(\p, \q) \enspace.
\end{equation}
Moreover, if $\alpha < \alpha'$, then a corresponding value is such that $\lambda > 0$, and the primal and dual rot mover's plans are unique and equal:
\begin{equation}
{\boldsymbol\pi'}\vphantom{\boldsymbol\pi}_\alpha^\star = \boldsymbol\pi_\lambda^\star \enspace.
\end{equation}
\end{theorem}

\begin{proof}
The primal problem can be seen as the minimization $p^\star$ of the cost $\langle \boldsymbol\pi, \boldsymbol\gamma \rangle$ on $\Pi(\p, \q)$ subject to $\phi(\boldsymbol\pi) - \phi(\boldsymbol\pi') - \alpha \leq 0$. The domain of this constrained convex problem is $\D = \Pi(\p, \q) \cap \dom \phi \neq \emptyset$. The Lagrangian on $\D \times \R$ is given by $\L(\boldsymbol\pi, \lambda) = \langle \boldsymbol\pi, \boldsymbol\gamma \rangle + \lambda (\phi(\boldsymbol\pi) - \phi(\boldsymbol\pi') - \alpha)$, and its minimization over $\D$ for a fixed $\lambda \geq 0$ has the same solutions $\boldsymbol\pi^\star$ as the dual problem. In addition, Slater's condition for convex problems, stating that there is a strictly feasible point in the relative interior of the domain, is verified as long as $\alpha > 0$. Indeed, we have $\ri(\D) = \ri(\Pi(\p, \q))$. The existence of a strictly feasible point $\phi(\boldsymbol\pi) < \phi(\boldsymbol\pi') + \alpha$ then holds by continuity of $\phi$ at $\boldsymbol\pi' \in \interior(\dom \phi)$. As a result, we have strong duality with a zero duality gap $p^\star = d^\star$, where $d^\star$ is the maximization of $g(\lambda)$ subject to $\lambda \geq 0$. Moreover, if $d^\star$ is finite, then it is attained at least once at a point $\lambda^\star$. This is the case since we already know that $p^\star$ is finite. Since $p^\star$ is also attained at least once at a point $\boldsymbol\pi^\star$ solution of the primal problem, we have the following chain:
\begin{align}
p^\star 	& = d^\star\\
		& = \min_{\boldsymbol\pi \in \D} \L(\boldsymbol\pi, \lambda^\star)\\
		& \leq \L(\boldsymbol\pi^\star, \lambda^\star)\\
		& = \langle \boldsymbol\pi^\star, \boldsymbol\gamma \rangle + \lambda^\star (\phi(\boldsymbol\pi^\star) - \phi(\boldsymbol\pi') - \alpha)\\
		& \leq \langle \boldsymbol\pi^\star, \boldsymbol\gamma \rangle\\
		& = p^\star \enspace.
\end{align}
Therefore, all inequalities are in fact equalities, $\boldsymbol\pi^\star$ also minimizes the Lagrangian over $\D$ and thus is a solution of the dual problem. In other words, the primal and dual RMD for $\alpha$ and $\lambda^\star$ are equal, and the primal solutions must be dual solutions too. For $\alpha < \alpha'$, the RMD has not reached the EMD yet, and thus we must have $\lambda^\star > 0$. Hence, the dual solution is unique, so that the primal solution is unique too and equal to the dual one.
\end{proof}

\begin{myremark}
Corresponding values of $\alpha$ and $\lambda$ depend on $\p, \q, \boldsymbol\gamma, \phi$. In addition, there might be multiple values of $\lambda$ that correspond to a given $\alpha$.
\end{myremark}

Again, the RMD does not verify the triangular inequality in general, and hence does not provide a true distance metric on $\Sigma_d$ even if $\boldsymbol\gamma$ is a distance matrix. Nevertheless, we still have the result that the RMD is symmetric as soon as $\phi$ is invariant by transposition, which holds for separable regularizers $\phi_{ij} = \phi$, and $\boldsymbol\gamma$ is symmetric. We also obtain properties for the dual RMD that are similar to the ones for the primal RMD.

\begin{property}
The dual rot mover's distance $d_{\boldsymbol\gamma, \lambda, \phi}(\p, \q)$ is an increasing and continuous function of $\lambda$.
\end{property}

\begin{proof}
The fact that it is increasing can be seen as follows. Let $0 \leq \lambda_1 < \lambda_2$. By construction, we have the following inequalities:
\begin{align}
\langle \boldsymbol\pi_{\lambda_1}^\star, \boldsymbol\gamma \rangle + \lambda_1 \phi(\boldsymbol\pi_{\lambda_1}^\star) & \leq \langle \boldsymbol\pi_{\lambda_2}^\star, \boldsymbol\gamma \rangle + \lambda_1 \phi(\boldsymbol\pi_{\lambda_2}^\star) \enspace,\\
\langle \boldsymbol\pi_{\lambda_2}^\star, \boldsymbol\gamma \rangle + \lambda_2 \phi(\boldsymbol\pi_{\lambda_2}^\star) & \leq \langle \boldsymbol\pi_{\lambda_1}^\star, \boldsymbol\gamma \rangle + \lambda_2 \phi(\boldsymbol\pi_{\lambda_1}^\star) \enspace.
\end{align}
Subtracting these two inequalities, we obtain that $\phi(\boldsymbol\pi_{\lambda_1}^\star) \geq \phi(\boldsymbol\pi_{\lambda_2}^\star)$. Reinserting this result in the first inequality, we finally get $\langle \boldsymbol\pi_{\lambda_1}^\star, \boldsymbol\gamma \rangle \leq \langle \boldsymbol\pi_{\lambda_2}^\star, \boldsymbol\gamma \rangle$. The continuity of the dual RMD results from that of the primal RMD. Let $\lambda \geq 0$, and choose an arbitrary dual rot mover's plan $\boldsymbol\pi_\lambda^\star$ and earth mover's plan $\boldsymbol\pi^\star$. On the one hand, we have $\langle \boldsymbol\pi^\star, \boldsymbol\gamma \rangle \leq \langle \boldsymbol\pi_\lambda^\star, \boldsymbol\gamma \rangle$. On the other hand, we have $\langle \boldsymbol\pi_\lambda^\star, \boldsymbol\gamma \rangle + \lambda \phi(\boldsymbol\pi_\lambda^\star) \leq \langle \boldsymbol\pi', \boldsymbol\gamma \rangle + \lambda \phi(\boldsymbol\pi')$, and thus $\langle \boldsymbol\pi_\lambda^\star, \boldsymbol\gamma \rangle \leq \langle \boldsymbol\pi', \boldsymbol\gamma \rangle + \lambda (\phi(\boldsymbol\pi') - \phi(\boldsymbol\pi_\lambda^\star)) \leq \langle \boldsymbol\pi', \boldsymbol\gamma \rangle$. Suppose we have a discontinuity of the dual RMD at $\lambda$. Then by monotonicity, there is a value $\langle \boldsymbol\pi^\star, \boldsymbol\gamma \rangle < d < \langle \boldsymbol\pi', \boldsymbol\gamma \rangle$ that is not in the image of the dual RMD. But $d$ is in the image of the primal RMD for a given $\alpha > 0$ by continuity. It means that $\nexists \lambda > 0$ such that $\langle \pi_\lambda^\star, C \rangle = d$, whereas, by continuity of the primal problem, we know that there exist $\alpha > 0$ such that $\langle \pi_\alpha^\star, C \rangle \geq d$. This is in contradiction with the duality result in Theorem~\ref{thm:duality}, which implies that the image of the primal RMD for $\alpha > 0$ must be included in that of the dual RMD for $\lambda \geq 0$.
\end{proof}

\begin{property}
\label{property:convergenceup}
When $\lambda$ tends to $+\infty$, the dual rot mover's distance converges to:
\begin{equation}
\lim_{\lambda \to +\infty} d_{\boldsymbol\gamma, \lambda, \phi}(\p, \q) = \langle \boldsymbol\pi', \boldsymbol\gamma \rangle \enspace,
\end{equation}
and the dual rot mover's plan converges in norm to the transport plan with minimal Bregman information:
\begin{equation}
\lim_{\lambda \to +\infty} \boldsymbol\pi_\lambda^\star = \boldsymbol\pi' \enspace.
\end{equation}
\end{property}

\begin{proof}
Let ${(\lambda_k)}_{k \in \N}$ be a sequence of positive numbers that tends to $+\infty$, and ${({\boldsymbol\pi}_{\lambda_k}^\star)}_{k \in \N}$ the associated rot mover's plans. By compactness of $\Pi(\p, \q)$, we can extract a subsequence of rot mover's plans that converges in norm to a point $\boldsymbol\pi^\star \in \Pi(\p, \q)$. For the sake of simplicity, we do not relabel this subsequence. By construction, we have $\langle \boldsymbol\pi_{\lambda_k}^\star, \boldsymbol\gamma \rangle + \lambda_k \phi(\boldsymbol\pi') \leq \langle \boldsymbol\pi_{\lambda_k}^\star, \boldsymbol\gamma \rangle + \lambda_k \phi(\boldsymbol\pi_{\lambda_k}^\star) \leq \langle \boldsymbol\pi', \boldsymbol\gamma \rangle + \lambda_k \phi(\boldsymbol\pi')$. The scalar products are bounded, so dividing the inequalities by $\lambda_k$ and taking the limit, we obtain that $\phi(\boldsymbol\pi_{\lambda_k}^\star)$ converges to $\phi(\boldsymbol\pi')$. By lower semi-continuity of $\phi$, we thus have $\phi(\boldsymbol\pi^\star) \leq \phi(\boldsymbol\pi')$. Since the global minimum of $\phi$ on $\Pi(\p, \q)$ is attained uniquely at $\boldsymbol\pi'$, we must have $\boldsymbol\pi^\star = \boldsymbol\pi'$, and the original sequence also converges in norm to $\boldsymbol\pi'$. Hence, the dual rot mover's plan $\boldsymbol\pi_\lambda$ converges in norm to $\boldsymbol\pi'$ when $\lambda$ tends to $+\infty$. By continuity of the total cost $\langle \cdot, \boldsymbol\gamma \rangle$ on $\R^{d \times d}$, $\langle \boldsymbol\pi_{\lambda_k}^\star, \boldsymbol\gamma \rangle$ converges to $\langle \boldsymbol\pi', \boldsymbol\gamma \rangle$. Hence, the limit of the RMD when $\lambda$ tends to $+\infty$ is $\langle \boldsymbol\pi', \boldsymbol\gamma \rangle$.
\end{proof}

\begin{property}
When $\lambda$ tends to $0$, the dual rot mover's distance converges to the earth mover's distance:
\begin{equation}
\lim_{\lambda \to 0} d_{\boldsymbol\gamma, \lambda, \phi}(\p, \q) = d_{\boldsymbol\gamma}(\p, \q) \enspace.
\end{equation}
\end{property}

\begin{proof}
This is a direct consequence of the dual RMD being continuous at $\lambda = 0$.
\end{proof}

\begin{property}
\label{property:convergencedown}
If ${[0, 1)}^{d \times d} \subseteq \dom \phi$ and $\phi$ is strictly convex on ${[0, 1)}^{d \times d}$, then the dual rot mover's plan converges in norm when $\lambda$ tends to $0$ to the earth mover's plan $\boldsymbol\pi_0^\star$ with minimal Bregman information:
\begin{equation}
\lim_{\lambda \to 0} \boldsymbol\pi_\lambda^\star = \boldsymbol\pi_0^\star \enspace.
\end{equation}
\end{property}

\begin{proof}
Let ${(\lambda_k)}_{k \in \N}$ be a sequence of positive numbers that converges to $0$, and ${({\boldsymbol\pi}_{\lambda_k}^\star)}_{k \in \N}$ the associated rot mover's plans. By compactness of $\Pi(\p, \q)$, we can extract a subsequence of rot mover's plans that converges in norm to a point $\boldsymbol\pi^\star \in \Pi(\p, \q)$. For the sake of simplicity, we do not relabel this subsequence. By construction, we have $\langle \boldsymbol\pi_0^\star, \boldsymbol\gamma \rangle + \lambda_k \phi(\boldsymbol\pi_{\lambda_k}^\star) \leq \langle \boldsymbol\pi_{\lambda_k}^\star, \boldsymbol\gamma \rangle + \lambda_k \phi(\boldsymbol\pi_{\lambda_k}^\star) \leq \langle \boldsymbol\pi_0^\star, \boldsymbol\gamma \rangle + \lambda_k \phi(\boldsymbol\pi_0^\star)$. The regularizer $\phi$ is continuous on the polytope $\Pi(\p, \q) \subseteq \dom \phi$, so taking the limit, we obtain that $\langle \boldsymbol\pi_{\lambda_k}^\star, \boldsymbol\gamma \rangle$ converges to $\langle \boldsymbol\pi_0^\star, \boldsymbol\gamma \rangle$. Therefore, $\boldsymbol\pi^\star$ must be an earth mover's plan. Now dividing by $\lambda_k$ and taking the limit, we obtain that $\phi(\boldsymbol\pi^\star) \leq \phi(\boldsymbol\pi_0^\star)$. Since $\boldsymbol\pi_0^\star$ is the unique earth mover's plan with minimal Bregman information, we must have $\boldsymbol\pi^\star = \boldsymbol\pi_0^\star$.
\end{proof}

\subsection{Geometrical Insights}
\label{subsec:geomoverview}

Our primal and dual formulations enlighten some intricate relations between optimal transportation theory~\cite{Villani2009} and information geometry~\cite{Amari2000}, where Bregman divergences are known to possess a dually flat structure with a generalized Pythagorean theorem for information projections. A schematic view of the underlying geometry for ROT problems is represented in Figure~\ref{fig:geometry_rot}, and can be discussed as follows.

\begin{figure}[t!]
\centering
\includegraphics[width=0.9\textwidth]{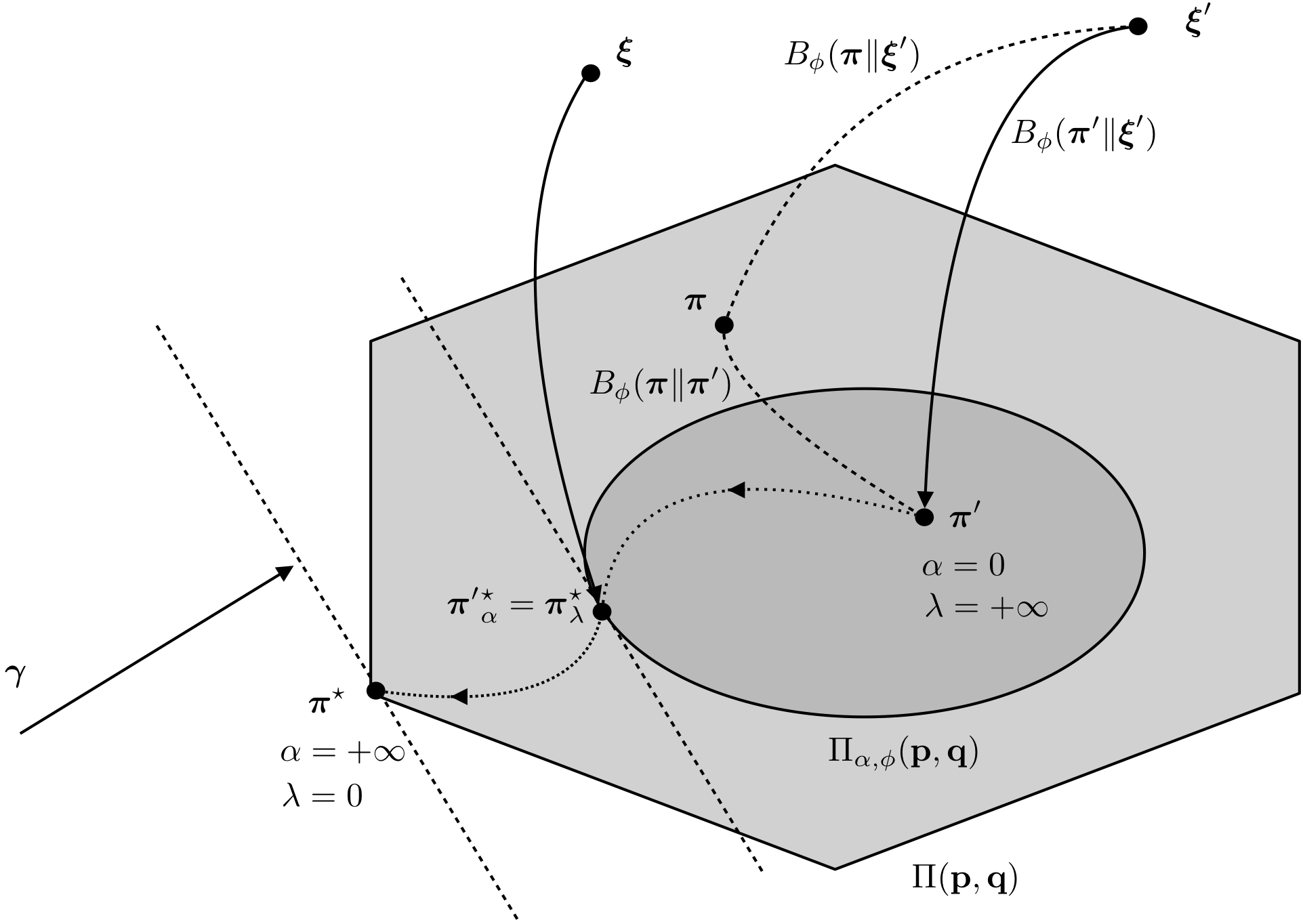}
\caption{Geometry of regularized optimal transport.}
\label{fig:geometry_rot}
\end{figure}

Our constructions start from the global minimizer $\boldsymbol\xi'$ of the regularizer $\phi$ (Lemma~\ref{lemma:primal_glob_min}). The Bregman projection $\boldsymbol\pi'$ of $\boldsymbol\xi'$ onto the transport polytope $\Pi(\p, \q)$ has minimal Bregman information on $\Pi(\p, \q)$ (Lemma~\ref{lemma:primal_rest_min}). The linear cost restricted to the regularized transport polytope $\Pi_{\alpha, \phi}(\p, \q)$ also attains its global minimum (Lemma~\ref{lemma:rtp}). Such a minimizer ${\boldsymbol\pi'}\vphantom{\boldsymbol\pi}_\alpha^\star$ is a primal rot mover's plan (Definition~\ref{def:prmd}). We can interpret $\Pi_{\alpha, \phi}(\p, \q)$ as the intersection of $\Pi(\p, \q)$ with the Bregman ball of radius $B_\phi(\boldsymbol\pi' \Vert \boldsymbol\xi') + \alpha$ and center $\boldsymbol\xi'$ (Proposition~\ref{prop:bigbball}). In certain cases, $\Pi_{\alpha, \phi}(\p, \q)$ is also the intersection of $\Pi(\p, \q)$ with the Bregman ball of radius $\alpha$ and center $\boldsymbol\pi'$, as a result of the generalized Pythagorean theorem $B_\phi(\boldsymbol\pi \Vert \boldsymbol\xi') = B_\phi(\boldsymbol\pi \Vert \boldsymbol\pi') + B_\phi(\boldsymbol\pi' \Vert \boldsymbol\xi')$ (Proposition~\ref{prop:bball}, Corollary~\ref{corol:bball}). All in all, this enforces the solutions to have small enough Bregman information, by constraining them to lie close to the matrix $\boldsymbol\xi'$ or transport plan $\boldsymbol\pi'$ with minimal Bregman information.

In our developments, we next introduce the global minimizer $\boldsymbol\xi$ of the regularized cost (Lemma~\ref{lemma:xi}). The regularized cost restricted to $\Pi(\p, \q)$ also attains its global minimum uniquely (Lemma~\ref{lemma:dual_rest_min}). This minimizer defines the dual rot mover's plan $\boldsymbol\pi_\lambda^\star$ (Definition~\ref{def:drmd}). Actually, $\boldsymbol\pi_\lambda^\star$ can be seen as the Bregman projection of $\boldsymbol\xi$ onto $\Pi(\p, \q)$ (Proposition~\ref{prop:dbproj}). The regularization by the Bregman information is also equivalent to regularizing the solution toward $\boldsymbol\xi'$ (Proposition~\ref{prop:dbigbball}). In some cases, this can also be seen as regularizing toward $\boldsymbol\pi'$, as a result of the generalized Pythagorean theorem $B_\phi(\boldsymbol\pi \Vert \boldsymbol\xi') = B_\phi(\boldsymbol\pi \Vert \boldsymbol\pi') + B_\phi(\boldsymbol\pi' \Vert \boldsymbol\xi')$ (Proposition~\ref{prop:dbball}, Corollary~\ref{corol:dbball}). Again, this enforces the solutions to have small enough Bregman information, by shrinking them toward the matrix $\boldsymbol\xi'$ or transport plan $\boldsymbol\pi'$ with minimal Bregman information.

We have duality between the primal and dual formulations, so that primal and dual rot mover's plans follow the same path on $\Pi(\p, \q)$ from no regularization ($\alpha = +\infty$, $\lambda = 0$) to full regularization ($\alpha = 0$, $\lambda = +\infty$) (Theorem~\ref{thm:duality}). In the limit of no regularization, we obviously retrieve an earth mover's plan $\boldsymbol\pi^\star$ for the cost matrix $\boldsymbol\gamma$. By duality, it is also intuitive that the additional constraint for the primal formulation, seen in the equivalent forms of $\phi(\boldsymbol\pi)$, $B_\phi(\boldsymbol\pi \Vert \boldsymbol\xi')$ or $B_\phi(\boldsymbol\pi \Vert \boldsymbol\pi')$, leads to an analog penalty for the dual formulation in the same respective form.

Since $\boldsymbol\xi' = \ones$, $\boldsymbol\pi' = \p \q^\top$, $\boldsymbol\xi = \exp(-\boldsymbol\gamma / \lambda)$, for minus the Boltzmann-Shannon entropy and Kullback-Leibler divergence, we retrieve the existing results discussed in Section~\ref{subsec:background} as a specific case~\cite{Cuturi2013,Benamou2015}. In addition, we can readily generalize the estimation of contingency tables with fixed marginals to a matrix nearness problem based on other divergences than the Kullback-Leibler divergence~\cite{Dhillon2007}. Given a rough estimate $\boldsymbol\xi \in \interior(\dom\phi)$, a contingency table with fixed marginals $\p, \q$ can be estimated by Bregman projection of $\boldsymbol\xi$ onto $\Pi(\p, \q)$:
\begin{equation}
\boldsymbol\pi^\star = \argmin_{\boldsymbol\pi \in \Pi(\p, \q)} B_\phi(\boldsymbol\pi \Vert \boldsymbol\xi) \enspace.
\end{equation}
This simply amounts to solving a dual ROT problem with an arbitrary penalty $\lambda > 0$ and a cost matrix $\boldsymbol\gamma = -\lambda \nabla\phi(\boldsymbol\xi)$.

Finally, since Bregman divergences are invariant under adding an affine term to their generator, it is straightforward to generalize ROT problems by shrinking toward an arbitrary prior matrix $\boldsymbol\xi, \boldsymbol\xi' \in \interior(\dom \phi)$, or transport plan $\boldsymbol\pi' \in \Pi(\p, \q)$. This is indeed equivalent to translating the regularizer by the appropriate amount $\phi(\boldsymbol\pi) + \langle \boldsymbol\pi, \boldsymbol\delta \rangle$, so that the global minimizer is now attained at the desired point. Equivalently, this amounts to translating the cost matrix as $\boldsymbol\gamma + \lambda \boldsymbol\delta$ instead.

\section{Algorithmic Derivations}
\label{sec:alg}

In this section, we introduce algorithmic methods to solve ROT problems. We focus without lack of generality on the dual problem, which can be solved efficiently via alternate Bregman projections. The primal problem can then easily be solved for $0 < \alpha < \alpha'$ by a bisection search on $\lambda > 0$. For $\alpha = 0$, we could simply use alternate Bregman projections to project $\nabla\psi(\zeros)$ instead of $\nabla\psi(-\boldsymbol\gamma / \lambda)$ in virtue of Lemmas~\ref{lemma:primal_glob_min} and~\ref{lemma:primal_rest_min}, which actually corresponds to the special case $\boldsymbol\gamma = \zeros$ in our algorithms, though this is not really relevant in practice since this completely removes the linear influence of the total cost from the ROT problem. In the limit $\alpha \geq \alpha'$, a classical OT solver such as the network simplex can directly be used. We first study the underlying Bregman projections in their generic form (Section~\ref{subsec:genproj}) and specifically develop the case of separable divergences (Section~\ref{subsec:sepproj}). We then derive the two generic schemes of ASA (Section~\ref{subsec:asa}) and NASA (Section~\ref{subsec:nasa}) to solve dual ROT problems, depending on whether the domain of the smooth convex regularizer lies within the non-negative orthant or not. We also enhance both algorithms in the separable case with a sparse extension (Section~\ref{subsec:sparse}), and finally discuss some practical considerations of our methods (Section~\ref{subsec:practice}). To simplify notations, we omit the penalty value $\lambda$ in the index and simply write $\boldsymbol\pi^\star$ for the rot mover's plan. 

\subsection{Generic Projections}
\label{subsec:genproj}

The closed convex transport polytope $\Pi(\p, \q)$ is the intersection of the non-negative orthant:
\begin{equation}
\C_0 = \R_+^{d \times d} \enspace,
\end{equation}
which is a polyhedral subset, with two affine subspaces:
\begin{align}
\C_1 & = \{\boldsymbol\pi \in \R^{d \times d} \colon \boldsymbol\pi \ones = \p\} \enspace,\\
\C_2 & = \{\boldsymbol\pi \in \R^{d \times d} \colon \boldsymbol\pi^\top \ones = \q\} \enspace.
\end{align}
The Bregman projection $\boldsymbol\pi^\star$ onto $\Pi(\p, \q)$ can then be obtained by alternate Bregman projections onto $\C_0, \C_1, \C_2$, where we expect that these latter projections are easier to compute.

On the one hand, the Karush-Kuhn-Tucker conditions for Bregman projection $\boldsymbol\pi_0^\star$ of a given matrix $\overline{\boldsymbol\pi} \in \interior(\dom \phi)$ onto $\C_0$ are necessary and sufficient, and write as follows:
\begin{align}
\boldsymbol\pi_0^\star \geq \zeros \enspace,\\
\nabla\phi(\boldsymbol\pi_0^\star) - \nabla\phi(\overline{\boldsymbol\pi}) \geq \zeros \enspace,\\
(\nabla\phi(\boldsymbol\pi_0^\star) - \nabla\phi(\overline{\boldsymbol\pi})) \odot \boldsymbol\pi_0^\star = \zeros \enspace.
\end{align}
While these conditions are nontrivial to solve in general, we shall see that they admit an elegant solver specific to the non-separable squared Mahalanobis distances defined in~\eqref{eq:mahal} and generated by the quadratic form in~\eqref{eq:aqc}. In addition, they also greatly simplify for separable divergences, which encompass all other divergences used in this paper.

On the other hand, the Lagrangians with Lagrange multipliers $\boldsymbol\mu, \boldsymbol\nu \in \R^d$ for the Bregman projections $\boldsymbol\pi_1^\star$ and $\boldsymbol\pi_2^\star$ of a given matrix $\overline{\boldsymbol\pi} \in \interior(\dom \phi)$ onto $\C_1$ and $\C_2$ respectively write as follows:
\begin{align}
\L_1(\boldsymbol\pi, \boldsymbol\mu) & = \phi(\boldsymbol\pi) - \langle \boldsymbol\pi, \nabla\phi(\overline{\boldsymbol\pi}) \rangle + \boldsymbol\mu^\top (\boldsymbol\pi \ones - \p) \enspace,\\
\L_2(\boldsymbol\pi, \boldsymbol\nu) & = \phi(\boldsymbol\pi) - \langle \boldsymbol\pi, \nabla\phi(\overline{\boldsymbol\pi}) \rangle + \boldsymbol\nu^\top (\boldsymbol\pi^\top \ones - \q) \enspace.
\end{align}
Their gradients are given on $\interior(\dom \phi)$ by:
\begin{align}
\nabla\L_1(\boldsymbol\pi, \boldsymbol\mu)	& = \nabla\phi(\boldsymbol\pi) - \nabla\phi(\overline{\boldsymbol\pi}) + \boldsymbol\mu \ones^\top \enspace,\\
\nabla\L_2(\boldsymbol\pi, \boldsymbol\nu) & = \nabla\phi(\boldsymbol\pi) - \nabla\phi(\overline{\boldsymbol\pi}) + \ones \boldsymbol\nu^\top \enspace,
\end{align}
and vanish at $\boldsymbol\pi^\star_1, \boldsymbol\pi^\star_2 \in \interior(\dom \phi)$ if and only if:
\begin{align}
\boldsymbol\pi^\star_1 & = \nabla\psi(\nabla\phi(\overline{\boldsymbol\pi}) - \boldsymbol\mu \ones^\top) \enspace,\\
\boldsymbol\pi^\star_2 & = \nabla\psi(\nabla\phi(\overline{\boldsymbol\pi}) - \ones \boldsymbol\nu^\top) \enspace.
\end{align}
By duality, the Bregman projections onto $\C_1, \C_2$ are thus equivalent to finding the unique vectors $\boldsymbol\mu, \boldsymbol\nu$, such that the rows of $\boldsymbol\pi^\star_1$ sum up to $\p$, respectively the columns of $\boldsymbol\pi^\star_2$ sum up to $\q$:
\begin{align}
\nabla\psi(\nabla\phi(\overline{\boldsymbol\pi}) - \boldsymbol\mu \ones^\top) \ones 				& = \p\enspace,\\
{\nabla\psi(\nabla\phi(\overline{\boldsymbol\pi}) - \ones \boldsymbol\nu^\top)}^\top \ones & = \q \enspace.
\end{align}
Similarly, solving for the Lagrange multipliers is an expensive problem in general, since the search space is of dimension $d$ and we evaluate matrix functions of size $d \times d$. This is because a given entry $\mu_i, \nu_j$ can actually modify any entry of the $d \times d$ matrix functions being evaluated. Again, we shall see that they can nevertheless be computed efficiently for separable divergences as well as the non-separable Mahalanobis distances.

\subsection{Separable Case}
\label{subsec:sepproj}

Assuming that the regularizer $\phi$ is separable, the underlying Bregman projections can be computed more efficiently. To keep notations simple, we focus on separable divergences with same element-wise regularizer, and thus chiefly omit the indices $\phi_{ij} = \phi$. We emphasize, however, that it is straightforward to apply all our methods for separable divergences with different element-wise regularizers, which notably enables weighting a given element-wise regularizer.

In case of separability, the Karush-Kuhn-Tucker conditions for projection onto $\C_0$ simplify to provide a closed-form solution on primal parameters:
\begin{equation}
\pi_{0, ij}^\star = \max\{0, \overline{\pi}_{ij}\} \enspace.
\end{equation}
Since $\phi'$ is increasing, this is equivalent on dual parameters to:
\begin{equation}
\theta_{0, ij}^\star = \max\{\phi'(0), \overline{\theta}_{ij}\} \enspace.
\end{equation}

Now turning to projections onto $\C_1, \C_2$ for primal parameters $\pi_{1, ij}^\star, \pi_{2, ij}^\star$, we can divide the initial problems into $d$ parallel subproblems in search space of dimension $1$ each. This is much more efficient to solve than in the non-separable case. This can be summarized as looking for $d$ separate Lagrange multipliers $\mu_i$, respectively $\nu_j$, such that:
\begin{align}
\sum_{j = 1}^d \psi'(\overline{\theta}_{ij} - \mu_i)	& = p_i \enspace,\\
\sum_{i = 1}^d \psi'(\overline{\theta}_{ij} - \nu_j)	& = q_j \enspace.
\end{align}
Finding the optimal values $\mu_i, \nu_j \in \R$ through $\psi'$ and the sums over rows or columns, however, is still nontrivial in general.

An analytical solution can be obtained in specific cases. Intuitively, we need to factor $\mu_i, \nu_j$ out of $\psi'$ as additive or multiplicative terms. This is related to Pexider's functional equations, which hold only for functions with a linear form $\psi'(\theta) = a \theta + b$, or exponential form $\psi'(\theta) = a \exp(b \theta)$, with $a, b \in \R$. This leads to regularizers with a quadratic form $\phi(\pi) = a \pi^2 + b \pi + c$, or entropic form $\phi(\pi) = a \pi \log \pi + b \pi + c$, with $a, b, c \in \R$. The constants $a, b$ actually only scale and translate the cost matrix, whereas the constant $c$ has no effect. Referring to Table \ref{tab:assump}, the quadratic case holds under assumptions (B), and thus requires Dykstra's algorithm for alternate Bregman projections with correction terms to ensure non-negativity by projection onto the polyhedral non-negative orthant. The entropic case holds under assumptions (A), using the POCS technique for alternate Bregman projection with no correction terms since the non-negativity is already ensured by the domain of the regularizer. The latter case reduces to the regularization of \cite{Cuturi2013} and \cite{Benamou2015}, so that we actually end up with the Sinkhorn-Knopp algorithm. Hence, the Euclidean norm associated to the squared Euclidean distance, and the entropic case associated to the Kullback-Leibler divergence, are reasonably the only two existing analytical schemes to find the sum constraint projections. For other ROT problems, available solvers for line search can be employed instead.

For simplicity, we assume hereafter that $\psi$ is twice continuously differentiable with $\psi''$ positive and $\psi'$ verifying the necessary and sufficient condition~\eqref{eq:nasc} on its whole domain. Therefore, we can use the Newton-Raphson method with guarantees of global convergence. This encompasses most of the common regularizers, and notably all regularizers used in this paper except from the Fermi-Dirac entropy, $\ell_p$ norms and Hellinger distance. When the condition~\eqref{eq:nasc} for global convergence is not met on the whole domain, it is still possible to apply the Newton-Raphson method after careful initialization, so as to restrict to a smaller interval where the condition holds. This is discussed in more detail with practical examples for the Fermi-Dirac entropy, $\ell_p$ norms and Hellinger distance in Section~\ref{sec:examples}, where the first-order derivatives are increasing convex on half of the domain and increasing concave on the other half. When the second-order derivatives do not exist, are not continuous or vanish at some points, a similar strategy can be applied. This is again discussed for the $\ell_p$ norms in Section~\ref{sec:examples}, where the second-order derivative is undefined or vanishes at $0$ depending on the value of the parameter. If such an initialization is not possible, then a bisection search can always be applied instead of the Newton-Raphson method.

To apply the Newton-Raphson method, we exploit the following functions:
\begin{align}
f(\mu_i) & = -\sum_{j = 1}^d \psi'(\overline{\theta}_{ij} - \mu_i) \enspace,\\
g(\nu_j) & = -\sum_{i = 1}^d \psi'(\overline{\theta}_{ij} - \nu_j) \enspace,
\end{align}
defined respectively on the open intervals $(\hat{\theta}_i - \overline{\theta}, +\infty)$ and $(\check{\theta}_j - \overline{\theta}, +\infty)$, where $0 < \overline{\theta} \leq +\infty$ is such that $\dom \psi = (-\infty, \overline{\theta})$, and $\hat{\theta}_i = \max\{\overline{\theta}_{ij}\}_{1 \leq j \leq d}$, $\check{\theta}_j = \max\{\overline{\theta}_{ij}\}_{1 \leq i \leq d}$. Their continuous derivatives are given by:
\begin{align}
f'(\mu_i) & = \sum_{j = 1}^d \psi''(\overline{\theta}_{ij} - \mu_i) \enspace,\\
g'(\nu_j) & = \sum_{i = 1}^d \psi''(\overline{\theta}_{ij} - \nu_j) \enspace,
\end{align}
and are positive, so that $f, g$ are strictly increasing on their whole domain, and thus on any closed interval with endpoints consisting of a feasible point and a solution. By construction, $f, g$ also verify the necessary and sufficient condition~\eqref{eq:nasc} for global convergence, and we know that there are unique solutions to $f(\mu_i) = -p_i$ and $g(\nu_j) = -q_j$. Hence, the Newton-Raphson updates:
\begin{align}
\mu_i & \leftarrow \mu_i + \frac{\sum_{j = 1}^d \psi'(\overline{\theta}_{ij} - \mu_i) - p_i}{\sum_{j = 1}^d \psi''(\overline{\theta}_{ij} - \mu_i)} \enspace,\\
\nu_j & \leftarrow \nu_j + \frac{\sum_{i = 1}^d \psi'(\overline{\theta}_{ij} - \nu_j) - q_j}{\sum_{i = 1}^d \psi''(\overline{\theta}_{ij} - \nu_j)} \enspace,
\end{align}
converge to the optimal solutions with a quadratic rate for any feasible starting points. By construction, we also know that initialization can be done with $\mu_i \leftarrow 0$, $\nu_j \leftarrow 0$. To avoid storing the intermediate Lagrange multipliers, the updates can then directly be written on dual parameters:
\begin{align}
\theta_{1, ij}^\star & \leftarrow \theta_{1, ij}^\star - \frac{\sum_{j = 1}^d \psi'(\theta_{1, ij}^\star) - p_i}{\sum_{j = 1}^d \psi''(\theta_{1, ij}^\star)} \enspace,\\
\theta_{2, ij}^\star & \leftarrow \theta_{2, ij}^\star - \frac{\sum_{i = 1}^d \psi'(\theta_{2, ij}^\star) - q_j}{\sum_{i = 1}^d \psi''(\theta_{2, ij}^\star)} \enspace,
\end{align}
after initialization by $\theta_{1, ij}^\star \leftarrow \overline{\theta}_{ij}$, $\theta_{2, ij}^\star \leftarrow \overline{\theta}_{ij}$.

\subsection{Alternate Scaling Algorithm}
\label{subsec:asa}

Under assumptions (A), we can drop the non-negative constraint since it is already ensured by $\dom \phi \subseteq \R_+^{d \times d}$ (Table \ref{tab:assump}). The POCS technique in its basic form~\eqref{pocs} then states that the projection of $\boldsymbol\xi$ onto $\Pi(\p, \q)$ can be obtained by alternate Bregman projections onto the affine subspaces $\C_1$ and $\C_2$ with linear convergence. Clearly, the underlying control mapping takes each output value an infinite number of times. Since we have just two sets, the only possible alternative in the control mapping is to swap the order of projections starting from $\C_2$ instead of $\C_1$, which actually amounts to swapping the input distributions $\p, \q$ and transposing the cost matrix $\boldsymbol\gamma$, to obtain the transposed of the rot mover's plan. We thus focus on the first choice without lack of generality.

Starting from $\boldsymbol\xi$ and writing the successive vectors $\boldsymbol\mu^{(k)}, \boldsymbol\nu^{(k)}$ along iterations, we have the following sequence:
\begin{align}
\nabla\psi(-\boldsymbol\gamma / \lambda) 	& \to \nabla\psi\left(-\boldsymbol\gamma / \lambda - \boldsymbol\mu^{(1)} \ones^\top\right)\\
															& \to \nabla\psi\left(-\boldsymbol\gamma / \lambda - \boldsymbol\mu^{(1)} \ones^\top - \ones \boldsymbol\nu^{(1)\top}\right)\\
															& \to \dotso \\
															& \to \nabla\psi\left(-\boldsymbol\gamma / \lambda - \boldsymbol\mu^{(1)} \ones^\top - \ones \boldsymbol\nu^{(1)\top} - \dotsb - \boldsymbol\mu^{(k)} \ones^\top\right)\\
															& \to \nabla\psi\left(-\boldsymbol\gamma / \lambda - \boldsymbol\mu^{(1)} \ones^\top - \ones \boldsymbol\nu^{(1)\top} - \dotsb - \boldsymbol\mu^{(k)} \ones^\top - \ones \boldsymbol\nu^{(k)\top}\right)\\
															& \to \dotso \\
															& \to \boldsymbol\pi^\star \enspace.
\end{align}
In other terms, we obtain the rot mover's plan $\boldsymbol\pi^\star$ by scaling iteratively the rows and columns of the successive estimates through $\nabla\psi$. An efficient algorithm, called ASA, is to store a unique $d \times d$ matrix in dual parameter space and update it by alternating the projections in primal parameter space (Algorithm~\ref{alg:asa}). The updates have a complexity in $O(d^2)$ once the vectors $\boldsymbol\mu, \boldsymbol\nu$ are obtained.

\begin{algorithm}[t!]
\caption{Alternate scaling algorithm.}
\label{alg:asa}
\begin{algorithmic}
\STATE $\boldsymbol\theta^\star \leftarrow - \boldsymbol\gamma / \lambda$
\REPEAT
\STATE $\boldsymbol\theta^\star \leftarrow \boldsymbol\theta^\star - \boldsymbol\mu \ones^\top$, where $\boldsymbol\mu$ uniquely solves $\nabla\psi(\boldsymbol\theta^\star - \boldsymbol\mu \ones^\top) \ones = \p$
\STATE $\boldsymbol\theta^\star \leftarrow \boldsymbol\theta^\star - \ones \boldsymbol\nu^\top$, where $\boldsymbol\nu$ uniquely solves ${\nabla\psi(\boldsymbol\theta^\star - \ones \boldsymbol\nu^\top)}^\top \ones = \q$
\UNTIL{convergence}
\STATE $\boldsymbol\pi^\star \leftarrow \nabla\psi(\boldsymbol\theta^\star)$
\end{algorithmic}
\end{algorithm}

In the separable case, the projections can be obtained by iterating the respective Newton-Raphson update steps, which can be written compactly with matrix and vector operations (Algorithm~\ref{alg:sasa}). The complexity for the updates are now clearly in $O(d^2)$. In more detail, each update step features one vector row or column replication, one vector element-wise division, one vector subtraction, one matrix subtraction, two matrix row or column sums, and two element-wise matrix function evaluations. Because of separability, we can expect the required number of iterations for convergence in the different loops to be independent of the data dimension, and thus expect a quadratic empirical complexity as well.

\begin{algorithm}[t!]
\caption{Alternate scaling algorithm in the separable case.}
\label{alg:sasa}
\begin{algorithmic}
\STATE $\boldsymbol\theta^\star \leftarrow - \boldsymbol\gamma / \lambda$
\REPEAT
\REPEAT
\STATE $\boldsymbol\theta^\star \leftarrow \boldsymbol\theta^\star - \frac{\psi'(\boldsymbol\theta^\star) \ones - \p}{\psi''(\boldsymbol\theta^\star) \ones} \, \ones^\top$
\UNTIL{convergence}
\REPEAT
\STATE $\boldsymbol\theta^\star \leftarrow \boldsymbol\theta^\star - \ones \, \frac{\ones^\top \psi'(\boldsymbol\theta^\star) - \q^\top}{\ones^\top \psi''(\boldsymbol\theta^\star)}$
\UNTIL{convergence}
\UNTIL{convergence}
\STATE $\boldsymbol\pi^\star \leftarrow \psi'(\boldsymbol\theta^\star)$
\end{algorithmic}
\end{algorithm}

\subsection{Non-negative Alternate Scaling Algorithm}
\label{subsec:nasa}

Under assumptions (B), we must now include the non-negative constraint since $\dom \phi \nsubseteq \R_+^{d \times d}$ (Table~\ref{tab:assump}). We suggest to ensure non-negativity of each update, and thus follow a cycle of projections onto $\C_0, \C_1, \C_0, \C_2$. The underlying control mapping is a fortiori essentially cyclic. For practical reasons, we also ensure non-negativity of the output solution with a final projection onto $\C_0$. Again, swapping the order of projections onto $\C_1, \C_2$ is equivalent to swapping the input distributions $\p, \q$ and transposing the cost matrix $\boldsymbol\gamma$ to obtain the transposed of the rot mover's plan. Other control mappings could also be exploited, for example by ensuring non-negativity every two or more sum constraint projections. We do not discuss such variants here and focus on the above-mentioned sequence. The non-negative orthant being polyhedral but not affine, we also need to incorporate correction terms $\boldsymbol\vartheta, \boldsymbol\varrho, \boldsymbol\varsigma$ for all three projections. In more detail, the projections are computed after correction so that we do not directly project the obtained updates $\boldsymbol\theta^\star$ but the corrected updates $\overline{\boldsymbol\theta} = \boldsymbol\theta^\star + \boldsymbol\vartheta$, $\overline{\boldsymbol\theta} = \boldsymbol\theta^\star + \boldsymbol\varrho$, and $\overline{\boldsymbol\theta} = \boldsymbol\theta^\star + \boldsymbol\varsigma$ for the respective subsets. The correction terms are also updated as the difference $\overline{\boldsymbol\theta} - \boldsymbol\theta^\star$ between the projected point and its projection. Dykstra's algorithm~\eqref{dykstra} for Bregman divergences with corrections~\eqref{correction} then guarantees that the projection of $\boldsymbol\xi$ onto $\Pi(\p, \q)$ is obtained with linear convergence.

A general algorithm, called NASA, is to store $d \times d$ matrices for projected points, projections and correction terms in dual parameter space, update them accordingly and finally go back to primal parameter space (Algorithm~\ref{alg:nasa}). The updates have a complexity in $O(d^2)$ once the Karush-Kuhn-Tucker conditions are solved or Lagrange multipliers $\boldsymbol\mu, \boldsymbol\nu$ are obtained.

\begin{algorithm}[t!]
\caption{Non-negative alternate scaling algorithm.}
\label{alg:nasa}
\begin{algorithmic}
\STATE $\boldsymbol\theta^\star \leftarrow - \boldsymbol\gamma / \lambda$
\STATE $\boldsymbol\vartheta \leftarrow \zeros$
\STATE $\boldsymbol\varrho \leftarrow \zeros$
\STATE $\boldsymbol\varsigma \leftarrow \zeros$
\STATE $\overline{\boldsymbol\theta} \leftarrow \boldsymbol\theta^\star + \boldsymbol\vartheta$
\STATE $\boldsymbol\theta^\star \leftarrow \boldsymbol\theta$, where $\boldsymbol\theta$ uniquely solves $\nabla\psi(\boldsymbol\theta) \geq \zeros$, $\boldsymbol\theta \geq \overline{\boldsymbol\theta}$, $(\boldsymbol\theta - \overline{\boldsymbol\theta}) \odot \nabla\psi(\boldsymbol\theta) = \zeros$
\STATE $\boldsymbol\vartheta \leftarrow \overline{\boldsymbol\theta} - \boldsymbol\theta^\star$
\REPEAT
\STATE $\overline{\boldsymbol\theta} \leftarrow \boldsymbol\theta^\star + \boldsymbol\varrho$
\STATE $\boldsymbol\theta^\star \leftarrow \overline{\boldsymbol\theta} - \boldsymbol\mu \ones^\top$, where $\boldsymbol\mu$ uniquely solves $\nabla\psi(\overline{\boldsymbol\theta} - \boldsymbol\mu \ones^\top) \ones = \p$
\STATE $\boldsymbol\varrho \leftarrow \overline{\boldsymbol\theta} - \boldsymbol\theta^\star$
\STATE $\overline{\boldsymbol\theta} \leftarrow \boldsymbol\theta^\star + \boldsymbol\vartheta$
\STATE $\boldsymbol\theta^\star \leftarrow \boldsymbol\theta$, where $\boldsymbol\theta$ uniquely solves $\nabla\psi(\boldsymbol\theta) \geq \zeros$, $\boldsymbol\theta \geq \overline{\boldsymbol\theta}$, $(\boldsymbol\theta - \overline{\boldsymbol\theta}) \odot \nabla\psi(\boldsymbol\theta) = \zeros$
\STATE $\boldsymbol\vartheta \leftarrow \overline{\boldsymbol\theta} - \boldsymbol\theta^\star$
\STATE $\overline{\boldsymbol\theta} \leftarrow \boldsymbol\theta^\star + \boldsymbol\varsigma$
\STATE $\boldsymbol\theta^\star \leftarrow \overline{\boldsymbol\theta} - \ones \boldsymbol\nu^\top$, where $\boldsymbol\nu$ uniquely solves ${\nabla\psi(\overline{\boldsymbol\theta} - \ones \boldsymbol\nu^\top)}^\top \ones = \q$
\STATE $\boldsymbol\varsigma \leftarrow \overline{\boldsymbol\theta} - \boldsymbol\theta^\star$
\STATE $\overline{\boldsymbol\theta} \leftarrow \boldsymbol\theta^\star + \boldsymbol\vartheta$
\STATE $\boldsymbol\theta^\star \leftarrow \boldsymbol\theta$, where $\boldsymbol\theta$ uniquely solves $\nabla\psi(\boldsymbol\theta) \geq \zeros$, $\boldsymbol\theta \geq \overline{\boldsymbol\theta}$, $(\boldsymbol\theta - \overline{\boldsymbol\theta}) \odot \nabla\psi(\boldsymbol\theta) = \zeros$
\STATE $\boldsymbol\vartheta \leftarrow \overline{\boldsymbol\theta} - \boldsymbol\theta^\star$
\UNTIL{convergence}
\STATE $\boldsymbol\pi^\star \leftarrow \nabla\psi(\boldsymbol\theta^\star)$
\end{algorithmic}
\end{algorithm}

In the separable case, the non-negativity constraint can be obtained analytically and the sequence of updates greatly simplifies. Starting from $\boldsymbol\xi$ and writing the successive vectors $\boldsymbol\mu^{(k)}, \boldsymbol\nu^{(k)}$ along iterations, we have:
\begin{align*}
\psi'(-\boldsymbol\gamma / \lambda)						& \to \psi'\big(\max\{\phi'(\zeros), -\boldsymbol\gamma / \lambda\}\big)\\
 												& \to \psi'\left(\max\{\phi'(\zeros), -\boldsymbol\gamma / \lambda\} - \boldsymbol\mu^{(1)} \ones^\top\right)\\
 												& \to \psi'\left(\max\{\phi'(\zeros), -\boldsymbol\gamma / \lambda - \boldsymbol\mu^{(1)} \ones^\top\}\right)\\
 												& \to \psi'\left(\max\{\phi'(\zeros), -\boldsymbol\gamma / \lambda - \boldsymbol\mu^{(1)} \ones^\top\} - \ones \boldsymbol\nu^{(1)\top}\right)\\
												& \to \psi'\left(\max\{\phi'(\zeros), -\boldsymbol\gamma / \lambda - \boldsymbol\mu^{(1)} \ones^\top - \ones \boldsymbol\nu^{(1)\top}\}\right)\\
												& \to \psi'\left(\max\{\phi'(\zeros), -\boldsymbol\gamma / \lambda - \boldsymbol\mu^{(1)} \ones^\top - \ones \boldsymbol\nu^{(1)\top}\} + \boldsymbol\mu^{(1)} \ones^\top - \boldsymbol\mu^{(2)} \ones^\top\right)\\
												& \to \psi'\left(\max\{\phi'(\zeros), -\boldsymbol\gamma / \lambda - \boldsymbol\mu^{(2)} \ones^\top - \ones \boldsymbol\nu^{(1)\top}\}\right)\\
												& \to \psi'\left(\max\{\phi'(\zeros), -\boldsymbol\gamma / \lambda - \boldsymbol\mu^{(2)} \ones^\top - \ones \boldsymbol\nu^{(1)\top}\} + \ones \boldsymbol\nu^{(1)\top} - \ones \boldsymbol\nu^{(2)\top}\right)\\
												& \to \psi'\left(\max\{\phi'(\zeros), -\boldsymbol\gamma / \lambda - \boldsymbol\mu^{(2)} \ones^\top - \ones \boldsymbol\nu^{(2)\top}\}\right)\\
												& \to \dotso \\
												& \to \psi'\left(\max\{\phi'(\zeros), -\boldsymbol\gamma / \lambda - \boldsymbol\mu^{(k)} \ones^\top - \ones \boldsymbol\nu^{(k)\top}\}\right)\\
												& \to \psi'\left(\max\{\phi'(\zeros), -\boldsymbol\gamma / \lambda - \boldsymbol\mu^{(k)} \ones^\top - \ones \boldsymbol\nu^{(k)\top}\} + \boldsymbol\mu^{(k)} \ones^\top - \boldsymbol\mu^{(k + 1)} \ones^\top\right)\\
												& \to \psi'\left(\max\{\phi'(\zeros), -\boldsymbol\gamma / \lambda - \boldsymbol\mu^{(k + 1)} \ones^\top - \ones \boldsymbol\nu^{(k)\top}\}\right)\\
												& \to \psi'\left(\max\{\phi'(\zeros), -\boldsymbol\gamma / \lambda - \boldsymbol\mu^{(k + 1)} \ones^\top - \ones \boldsymbol\nu^{(k)\top}\} + \ones \boldsymbol\nu^{(k)\top} - \ones \boldsymbol\nu^{(k + 1)\top}\right)\\
												& \to \psi'\left(\max\{\phi'(\zeros), -\boldsymbol\gamma / \lambda - \boldsymbol\mu^{(k + 1)} \ones^\top - \ones \boldsymbol\nu^{(k + 1)\top}\}\right)\\
												& \to \dotso \\
												& \to \boldsymbol\pi^\star \enspace.
\end{align*}
An efficient algorithm then exploits the differences $\boldsymbol\tau^{(k)} = \boldsymbol\mu^{(k)} - \boldsymbol\mu^{(k - 1)}$ and $\boldsymbol\sigma^{(k)} = \boldsymbol\nu^{(k)} - \boldsymbol\nu^{(k - 1)}$ to scale the rows and columns (Algorithm~\ref{alg:snasa}). We store $d \times d$ matrices as well as difference vectors instead of correction matrices. The algorithm can then be interpreted as producing interleaved updates between the projections according to the max operator and according to the respective scalings. The updates in NASA now clearly have a complexity in $O(d^2)$ when using the Newton-Raphson method for scaling, with similar matrix and vector operations to ASA in the separable case, and an expected empirical complexity that is quadratic.

\begin{algorithm}[t!]
\caption{Non-negative alternate scaling algorithm in the separable case.}
\label{alg:snasa}
\begin{algorithmic}
\STATE $\widetilde{\boldsymbol\theta} \leftarrow - \boldsymbol\gamma / \lambda$
\STATE $\boldsymbol\theta^\star \leftarrow \max\{\phi'(\zeros), \widetilde{\boldsymbol\theta}\}$
\REPEAT
\STATE $\boldsymbol\tau \leftarrow \zeros$
\REPEAT
\STATE $\boldsymbol\tau \leftarrow \boldsymbol\tau + \frac{\psi'(\boldsymbol\theta^\star - \boldsymbol\tau \ones^\top) \ones - \p}{\psi''(\boldsymbol\theta^\star - \boldsymbol\tau \ones^\top) \ones}$
\UNTIL{convergence}
\STATE $\widetilde{\boldsymbol\theta} \leftarrow \widetilde{\boldsymbol\theta} - \boldsymbol\tau \ones^\top$
\STATE $\boldsymbol\theta^\star \leftarrow \max\{\phi'(\zeros), \widetilde{\boldsymbol\theta}\}$
\STATE $\boldsymbol\sigma \leftarrow \zeros$
\REPEAT
\STATE $\boldsymbol\sigma \leftarrow \boldsymbol\sigma + \frac{\ones^\top \psi'(\boldsymbol\theta^\star - \ones \boldsymbol\sigma^\top) - \q^\top}{\ones^\top \psi''(\boldsymbol\theta^\star - \ones \boldsymbol\sigma^\top)}$
\UNTIL{convergence}
\STATE $\widetilde{\boldsymbol\theta} \leftarrow \widetilde{\boldsymbol\theta} - \ones \boldsymbol\sigma^\top$
\STATE $\boldsymbol\theta^\star \leftarrow \max\{\phi'(\zeros), \widetilde{\boldsymbol\theta}\}$
\UNTIL{convergence}
\STATE $\boldsymbol\pi^\star \leftarrow \psi'(\boldsymbol\theta^\star)$
\end{algorithmic}
\end{algorithm}

\subsection{Sparse Extension}
\label{subsec:sparse}

In the separable case, it is possible to develop a sparse extension of both our methods ASA and NASA. Storing and updating full $d \times d$ matrices becomes expensive with the data dimension. Instead, we allow for infinite entries in the cost matrix $\boldsymbol\gamma$, meaning that the transport of mass between certain bins is proscribed. As a result, the corresponding entries of $\boldsymbol\pi^\star$ must be null. Eventually, we can drop all these entries so that we just need to store and update the remaining ones. The RMD via the Frobenius inner product $\langle \boldsymbol\pi^\star, \boldsymbol\gamma \rangle$ is then computed without accounting for discarded entries, or equivalently by setting indefinite element-wise products $0 \times \infty = 0$ by convention, so it naturally costs nothing to move no mass on a path that is forbidden. This leads to an expected complexity in $O(r)$, where $r$ is the number of finite entries in $\boldsymbol\gamma$. Typically, $r$ can be chosen in the order of magnitude of $d$, so as to obtain a linear instead of quadratic empirical complexity.

In practice, both ASA and NASA are compatible with this strategy. We always have $\lim_{\theta \to -\infty} \psi'(\theta) = 0$ under assumptions (A) for ASA. Under assumptions (B) for NASA, this limit might be finite or infinite but is necessarily negative, so also leads to $0$ after enforcing non-negativity by projection onto the non-negative orthant. As a result, the obtained sequence of projections preserves the desired zeros in both algorithms, and an infinite element-wise cost does lead to no mass transport at all between the corresponding bins. In theory, we can understand this extension in light of the dual formulation seen as a Bregman projection in~\eqref{rotbp}. Under assumptions (B), we always have $0 \in \interior(\dom \phi)$ and thus $B_\phi(0 \Vert 0) = 0$. Hence, Dykstra's algorithm is readily applicable in the sparse version. Under assumptions (A), however, we have $0 \notin \interior(\dom \phi)$, and even sometimes $0 \notin \dom \phi$ as for the Itakura-Saito divergence. We can nonetheless extend the domain of the element-wise divergence at the origin by continuity on the diagonal, that is, by setting it null as $B_\phi(0 \Vert 0) = 0$. This is akin to considering absolutely continuous measures, also known as dominated measures, and Radon-Nikodym derivatives to generalize the definition of Bregman divergences. \cite{Kurras2015} then showed that the POCS method still holds with this convention by introducing a notion of locally affine spaces.

With such a sparse extension, however, we must take care that a sparse solution does exist, meaning that there is a transport plan in the transport polytope that has the desired zeros. For example, if all entries of $\boldsymbol\gamma$ are infinite, then there are obviously no possible sparse solutions since we enforce all entries of the plan to be null. A necessary condition for the existence of a sparse solution is that for any entry $q_j$, all entries $p_k$ from which we are allowed to transport mass must provide enough total mass to fill $q_j$ completely. Similarly, for any entry $p_i$, all entries $q_k$ to which we are allowed to transport mass must require enough total mass to empty $p_i$ completely. Unfortunately, sufficient conditions are not so intuitive. \cite[Theorem~4.1]{Idel2016} studied such problems thoroughly and elucidated several necessary and sufficient conditions for sparse solutions to exist, but these conditions are nontrivial to use from in practice. \cite{Kurras2015} advocates trying first to compute a solution with the desired sparsity, and if no solution can be found, then gradually reduce sparsity until a solution is found. This might still speed up computation drastically because of the linear instead of quadratic complexity. Lastly, we remark that it is not evident to propose a sparse extension for the non-separable case in general, since a given entry of $\boldsymbol\gamma$ might influence all entries of $\boldsymbol\pi^\star$.

\subsection{Practical Considerations}
\label{subsec:practice}

As noticed by \cite{Cuturi2013} and \cite{Benamou2015}, the Sinkhorn-Knopp algorithm might fail to converge because of numerical instability when the penalty $\lambda$ gets small. In particular, unless taking special care of numerical stabilization~\cite{Schmitzer2016b}, a direct limitation is the machine precision under which some entries of $\exp(-\boldsymbol\gamma / \lambda)$ are represented as zeros in memory. Such issues occur similarly for other regularizations, notably via the representation $\nabla\psi(-\boldsymbol\gamma / \lambda)$ of the unconstrained solution to project. Therefore, the proposed methods are actually competitive in a range where the penalty $\lambda$ is not too small, and for which the rot mover's plan $\boldsymbol\pi^\star$ exhibits a significant amount of smoothing. Hence, we do not target the same problems as traditional schemes such as interior point methods or the network simplex.

In addition, the different Bregman projections in our algorithms are most of the time approximate up to a given tolerance depending on the termination criterion used for convergence. Exceptions occur for the sum constraints with the Euclidean distance or Kullback-Leibler divergence, as well as the non-negativity constraints in the separable case, which are obtained analytically. A natural question to raise is then whether our algorithms still converge when the projections are approximate only. However, this is relatively hard to answer in theory. We did not observe in practice any problem of convergence when using sufficiently good approximations. Furthermore, first approximations can be quite rough without affecting convergence as long as final approximations are good enough. Sometimes, even alternating a single or two steps of the Newton-Raphson method throughout the main iterations the algorithm still works, though this is not systematic. Thus, we advocate for safety to use a tight tolerance for the auxiliary projections.

We also observed numerical instability of the Newton-Raphson updates for separable divergences under assumptions (A). This is due to the denominator being based on $\psi''$ with limit $\lim_{\theta \to -\infty} \psi''(\theta) = 0$, that is, for entries $\pi$ close to zero. It is possible, however, to make the updates of $\mu_i, \nu_j$ much more stable in practice by using the max truncation operator, despite theoretical guarantees of convergence without it. Specifically, we know that the entries $\pi_{1, ij}^\star$ must lie between $0$ and $p_i$, and $\pi_{2, ij}^\star$ between $0$ and $q_j$. Hence, we can lower bound $\mu_i$ and $\nu_j$ by $\hat{\theta}_i - \phi'(p_i)$ and $\hat{\theta}_j - \phi'(q_j)$, respectively. Interestingly, this also speeds up the convergence of the updates significantly when the initialization by $0$ is far from the actual solution.

A possible termination criterion for the main and auxiliary iterations is to compute the marginal difference between the updated matrix and $\p, \q$. In the auxiliary iterations for the two scaling projections, we compare the sums of rows or columns to $\p$ or $\q$ respectively, and in the main iterations of the algorithm, we compare both marginals simultaneously. Typically, we can use the $\ell_p$ (quasi-)norm with $0 < p \leq +\infty$ to assess the marginal difference, and the auxiliary tolerance should be at least the main one for sufficient precision of the approximations. Two alternative quantities in absolute or relative scales can also be used for termination, either the variation with $\ell_p$ (quasi-)norm in the updated matrix or in the updated distance. Here the auxiliary tolerance should be at least the square of the main one. This seems reasonable for $\boldsymbol\pi^\star$ given the quadratic rate of convergence for the Newton-Raphson method versus the linear one for alternate Bregman projections, as well as for $\langle \boldsymbol\pi^\star, \boldsymbol\gamma \rangle$ under the Cauchy-Schwarz inequality. In all cases, convergence can be checked either after each iteration or after a given number of iterations to reduce the underlying cost of computing the termination criterion. We can also fix a maximum number of main and auxiliary iterations to limit the overall running time.

Regarding implementation, the matrix and vector operations used for ASA and NASA in the separable case are well-suited for fast calculation on a GPU and for processing of multiple input distributions in parallel. By working directly in the primal parameter space, the Sinkhorn-Knopp algorithm is also readily suited for dealing with sparse plans, based on existing libraries. In more general ROT problems, however, a specific library should be written for the sparse extension because null entries in the transport plan are not represented by null entries in the dual parameter space, so that tailored data structures and operations for such matrices need to be coded. Therefore, we only implemented and will focus in our experiments on the non-sparse version of our methods.

Finally, although we implicitly assumed throughout that the entries of $\p$ and $\q$ are strictly comprised between $0$ and $1$ for theoretical issues, it is often possible in practice to deal explicitly with null or unit entries in the input distributions. Intuitively, no mass can be moved from or to a null entry, so the transport plans have null rows and columns for the corresponding null entries of $\p$ and $\q$, respectively. In the separable case, we can thus simply remove these entries, solve the reduced ROT problem, and reinsert the corresponding null entries in the rot mover's plan $\boldsymbol\pi^\star$. The same reasoning as for the sparse extension can be made to show that our two algorithms still hold with this strategy from a theoretical standpoint. In the non-separable case, however, this is not as straightforward again because the influences of the different entries of $\boldsymbol\pi^\star$ are interleaved through the regularizer $\phi$. Nonetheless, as long as we have ${[0, 1)}^{d \times d} \subset \interior(\dom \phi)$ under assumptions (B), then we have $\Pi(\p, \q) \subset \interior(\dom \phi)$ and we can apply NASA without modification. This is notably the case for the Mahalanobis distances whose domain is $\R^{d \times d}$. For a non-separable regularizer under assumptions (A), it is not easy to account for null entries because the constraint qualification $\Pi(\p, \q) \cap \interior(\dom \phi) \neq \emptyset$ never holds due to mandatory null entries in the transport plans. Nevertheless, common regularizers under assumptions (A), including the ones used in this paper, are separable in general. Lastly, it is direct to cope with unit entries in $\p$ or $\q$ in all cases, since the transport polytope then reduces to a singleton, so that there is a unique transport plan $\p\q^\top$ which is the rot mover's plan.

\section{Classical Regularizers and Divergences}
\label{sec:examples}

In this section, we discuss the specificities of the ASA (Algorithm~\ref{alg:asa}) and NASA (Algorithm~\ref{alg:nasa}) methods to solve ROT problems for classical regularizers and associated divergences. We start with several separable regularizers under assumptions (A), based on the Boltzmann-Shannon entropy related to the Kullback-Leibler divergence (\texttt{BSKL}, Section~\ref{subsec:kl}), the Burg entropy related to the Itakura-Saito divergence (\texttt{BIS}, Section~\ref{subsec:is}), and the Fermi-Dirac entropy related to a logistic loss function (\texttt{FDLOG}, Section~\ref{subsec:log}), as well as the parametric families of $\beta$\nobreakdash-potentials related to the $\beta$\nobreakdash-divergences (\texttt{BETA}, Section~\ref{subsec:beta}). We then discuss the separable $\ell_p$ quasi-norms (\texttt{LPQN}, Section~\ref{subsec:lpqn}), which require a slight adaptation of assumptions (A). We also consider separable regularizers under assumptions (B) related to $\ell_p$ norms (\texttt{LPN}, Section~\ref{subsec:lpn}), as well as the Euclidean norm related to the Euclidean distance (\texttt{EUC}, Section~\ref{subsec:euc}) and the Hellinger distance (\texttt{HELL}, Section~\ref{subsec:hell}). Finally, we study a non-separable regularizer under assumptions (B) via quadratic forms in relation to Mahalanobis distances (Section~\ref{subsec:mahal}). We plot all separable regularizers in Figure~\ref{fig:reg}. All regularizers and their corresponding divergences are also summed up in Table~\ref{tab:regdiv}. Lastly, we provide in Table~\ref{tab:deriv} the related terms based on derivatives that are needed to instantiate the separable versions of ASA (Algorithm~\ref{alg:sasa}) or NASA (Algorithm~\ref{alg:snasa}) accordingly.

\begin{table}[t!]
\centering
\begin{tabular}{llll}
\toprule
$\phi(\pi)$ / $\phi(\boldsymbol\pi)$ 													&
$B_\phi(\pi \Vert \xi)$ / $B_\phi(\boldsymbol\pi \Vert \boldsymbol\xi)$							&
$\dom \phi$ 																	&
$\dom \psi$ 																	\\
\midrule
\textit{Boltzmann-Shannon entropy} 													&
\textit{Kullback-Leibler divergence} 													&
																			&
																			\\ \vspace{0.25cm}
$\pi \log \pi - \pi + 1$ 																&
$\pi \log \frac{\pi}{\xi} - \pi + \xi$														&
$\R_+$ 																		&
$\R$																			\\
\textit{Burg entropy}																&
\textit{Itakura-Saito divergence}														&
																			&
																			\\ \vspace{0.25cm}
$\pi - \log \pi - 1$ 																&
$\frac{\pi}{\xi} - \log \frac{\pi}{\xi} - 1$													&
$\R_{++}$ 																	&
$(-\infty, 1)$ 																	\\
\textit{Fermi-Dirac entropy}														&
\textit{Logistic loss function}														&
																			&
																			\\ \vspace{0.25cm}
$\pi \log \pi + (1 - \pi) \log (1 - \pi)$ 													&
$\pi \log \frac{\pi}{\xi} + (1 - \pi) \log \frac{1 - \pi}{1 - \xi}$									&
$[0, 1]$ 																		&
$\R$ 																		\\
\textit{$\beta$\nobreakdash-potentials $(0 < \beta < 1)$}									&
\textit{$\beta$\nobreakdash-divergences}												&
																			&
																			\\ \vspace{0.25cm}
$\frac{1}{\beta (\beta - 1)} (\pi^\beta - \beta \pi + \beta - 1)$ 								&
$\frac{1}{\beta (\beta - 1)} (\pi^\beta + (\beta - 1) \xi^\beta - \beta \pi \xi^{\beta - 1})$				&
$\R_+$ 																		&
$(-\infty, \frac{1}{1 - \beta})$														\\
\textit{$\ell_p$ quasi-norms $(0 < p < 1)$}												&
																			&
																			&
																			\\ \vspace{0.25cm}
$-\pi^p$ 																		&
$-\pi^p + p \pi \xi^{p - 1} - (p - 1) \xi^p$												&
$\R_+$ 																		&
$\R_{--}$ 																		\\
\textit{$\ell_p$ norms $(1 < p < +\infty)$}												&
																			&
																			&
																			\\ \vspace{0.25cm}
${|\pi|}^p$ 																		&
${|\pi|}^p - p \pi \, \sgn(\xi) {|\xi|}^{p - 1} + (p - 1) {|\xi|}^p$									&
$\R$ 																		&
$\R$ 																		\\
\textit{Euclidean norm}															&
\textit{Euclidean distance}															&
																			&
																			\\ \vspace{0.25cm}
$\frac{1}{2} \pi^2$ 																&
$\frac{1}{2} {(\pi - \xi)}^2$															&
$\R$ 																		&
$\R$ 																		\\
\textit{Hellinger distance}															&
																			&
																			&
																			\\ \vspace{0.25cm}
$-{(1 - \pi^2)}^{\frac{1}{2}}$ 														&
$(1 - \pi \xi) {(1 - \xi^2)}^{-\frac{1}{2}} - {(1 - \pi^2)}^{\frac{1}{2}}$								&
$[-1, 1]$ 																		&
$\R$ 																		\\
\textit{Quadratic forms $(\P \succ \zeros)$}											&
\textit{Mahalanobis distances}														&
																			&
																			\\ \vspace{0.25cm}
$\frac{1}{2} {\vect(\boldsymbol\pi)}^\top \P \vect(\boldsymbol\pi)$							&
$\frac{1}{2} {\vect(\boldsymbol\pi - \boldsymbol\xi)}^\top \P \vect(\boldsymbol\pi - \boldsymbol\xi)$	&
$\R^{d \times d}$ 																&
$\R^{d \times d}$ 																\\
\bottomrule
\end{tabular}
\caption{Convex regularizers and associated Bregman divergences.}
\label{tab:regdiv}
\end{table}

\begin{table}[t!]
\centering
\begin{tabular}{llll}
\toprule
$\phi(\pi)$											&
$\phi'(\pi)$ 										&
$\psi'(\theta)$										&
$\psi''(\theta)$										\\
\midrule
\textit{Boltzmann-Shannon entropy} 						&
												&
												&
												\\ \vspace{0.25cm}
$\pi \log \pi - \pi + 1$									&
$\log \pi$ 											&
$\exp \theta$ 										&
$\exp \theta$ 										\\
\textit{Burg entropy} 									&
												&
												&
												\\ \vspace{0.25cm}
$\pi - \log \pi - 1$									&
$1 - \pi^{-1}$ 										&
${(1 -\theta)}^{-1}$ 									&
${(1 - \theta)}^{-2}$									\\
\textit{Fermi-Dirac entropy}							&
												&
												&
												\\ \vspace{0.25cm}
$\pi \log \pi + (1 - \pi) \log (1 - \pi)$ 						&
$\log \frac{\pi}{1 - \pi}$ 								&
$\frac{\exp \theta}{(1 + \exp \theta)}$ 					&
$\frac{\exp \theta}{{(1 + \exp \theta)}^2}$					\\
\textit{$\beta$\nobreakdash-potentials $(0 < \beta < 1)$} 		&
												&
												&
												\\ \vspace{0.25cm}
$\frac{1}{\beta (\beta - 1)} (\pi^\beta - \beta \pi + \beta - 1)$	&
$\frac{1}{\beta - 1} (\pi^{\beta - 1} - 1)$ 					&
${((\beta - 1) \theta + 1)}^{\frac{1}{\beta - 1}}$ 				&
${((\beta - 1) \theta + 1)}^{\frac{1}{\beta - 1} - 1}$			\\
\textit{$\ell_p$ quasi-norms $(0 < p < 1)$}					&
												&
												&
												\\ \vspace{0.25cm}
$-\pi^p$ 											&
$-p \pi^{p - 1}$ 										&
$p^{-\frac{1}{p - 1}} {(-\theta)}^{\frac{1}{p - 1}}$ 				&
$- \frac{p^{-\frac{1}{p - 1}}}{p - 1} {(-\theta)}^{\frac{1}{p - 1} - 1}$	\\
\textit{$\ell_p$ norms $(1 < p < +\infty)$}					&
												&
												&
												\\ \vspace{0.25cm}
${|\pi|}^p$ 											&
$p \, \sgn(\pi) {|\pi|}^{p - 1}$ 							&
$p^{-\frac{1}{p - 1}} \, \sgn(\theta) {|\theta|}^{\frac{1}{p - 1}}$	&
$\frac{p^{-\frac{1}{p - 1}}}{p - 1} {|\theta|}^{\frac{1}{p - 1} - 1}$ 	\\
\textit{Euclidean norm}								&
												&
												&
												\\ \vspace{0.25cm}
$\frac{1}{2} \pi^2$									&
$\pi$ 											&
$\theta$ 											&
$1$												\\
\textit{Hellinger distance}								&
												&
												&
												\\ \vspace{0.25cm}
$-{(1 - \pi^2)}^{\frac{1}{2}}$ 							&
$\pi {(1 - \pi^2)}^{-\frac{1}{2}}$ 							&
$\theta {(1 + \theta^2)}^{-\frac{1}{2}}$ 					&
${(1 + \theta^2)}^{-\frac{3}{2}}$ 							\\
\bottomrule
\end{tabular}
\caption{Separable regularizers and related terms based on derivatives.}
\label{tab:deriv}
\end{table}

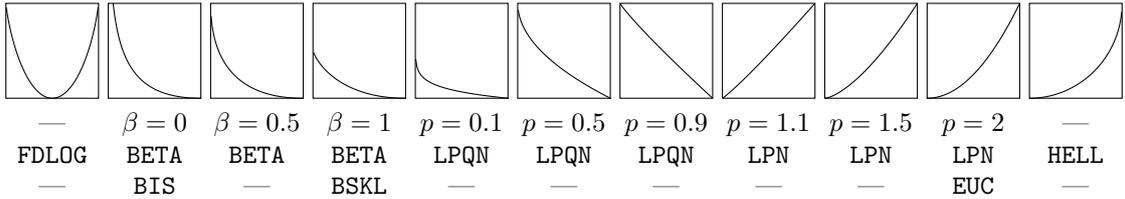
\begin{figure}[t!]
\centering
\begin{tikzpicture}
    \begin{axis}[xmin=0,xmax=1,ymin={ln(0.5)},ymax=0,ticks=none,samples=1000]
    \addplot[black](x,{x*ln(x)+(1-x)*ln(1-x)});
\end{axis}
\end{tikzpicture}
\begin{tikzpicture}
    \begin{axis}[xmin=0,xmax=1,ymin=0,ymax=2,ticks=none,samples=1000]
    \addplot[black](x,{x-ln(x)-1});
\end{axis}
\end{tikzpicture}
\begin{tikzpicture}
    \begin{axis}[xmin=0,xmax=1,ymin=0,ymax=2,ticks=none,samples=1000]
    \addplot[black](x,{(x^0.5-0.5*x+0.5-1)/(0.5*(0.5-1))});
\end{axis}
\end{tikzpicture}
\begin{tikzpicture}
    \begin{axis}[xmin=0,xmax=1,ymin=0,ymax=2,ticks=none,samples=1000]
    \addplot[black](x,{x*ln(x)-x+1});
\end{axis}
\end{tikzpicture}
\begin{tikzpicture}
    \begin{axis}[xmin=0,xmax=1,ymin=-1,ymax=0,ticks=none,samples=1000]
    \addplot[black](x,{-x^0.1});
\end{axis}
\end{tikzpicture}
\begin{tikzpicture}
    \begin{axis}[xmin=0,xmax=1,ymin=-1,ymax=0,ticks=none,samples=1000]
    \addplot[black](x,{-x^0.5});
\end{axis}
\end{tikzpicture}
\begin{tikzpicture}
    \begin{axis}[xmin=0,xmax=1,ymin=-1,ymax=0,ticks=none,samples=1000]
    \addplot[black](x,{-x^0.9});
\end{axis}
\end{tikzpicture}
\begin{tikzpicture}
    \begin{axis}[xmin=0,xmax=1,ymin=0,ymax=1,ticks=none,samples=1000]
    \addplot[black](x,{x^1.1});
\end{axis}
\end{tikzpicture}
\begin{tikzpicture}
    \begin{axis}[xmin=0,xmax=1,ymin=0,ymax=1,ticks=none,samples=1000]
    \addplot[black](x,{x^1.5});
\end{axis}
\end{tikzpicture}
\begin{tikzpicture}
    \begin{axis}[xmin=0,xmax=1,ymin=0,ymax=1,ticks=none,samples=1000]
    \addplot[black](x,{x^2});
\end{axis}
\end{tikzpicture}
\begin{tikzpicture}
    \begin{axis}[xmin=0,xmax=1,ymin=-1,ymax=0,ticks=none,samples=1000]
    \addplot[black](x,{-(1-x^2)^0.5});
\end{axis}
\end{tikzpicture}
\\
\makebox[0.084\textwidth]{---}
\makebox[0.084\textwidth]{$\beta = 0$}
\makebox[0.084\textwidth]{$\beta = 0.5$}
\makebox[0.084\textwidth]{$\beta = 1$}
\makebox[0.084\textwidth]{$p = 0.1$}
\makebox[0.084\textwidth]{$p = 0.5$}
\makebox[0.084\textwidth]{$p = 0.9$}
\makebox[0.084\textwidth]{$p = 1.1$}
\makebox[0.084\textwidth]{$p = 1.5$}
\makebox[0.084\textwidth]{$p = 2$}
\makebox[0.084\textwidth]{---}\\
\makebox[0.084\textwidth]{\texttt{FDLOG}}
\makebox[0.084\textwidth]{\texttt{BETA}}
\makebox[0.084\textwidth]{\texttt{BETA}}
\makebox[0.084\textwidth]{\texttt{BETA}}
\makebox[0.084\textwidth]{\texttt{LPQN}}
\makebox[0.084\textwidth]{\texttt{LPQN}}
\makebox[0.084\textwidth]{\texttt{LPQN}}
\makebox[0.084\textwidth]{\texttt{LPN}}
\makebox[0.084\textwidth]{\texttt{LPN}}
\makebox[0.084\textwidth]{\texttt{LPN}}
\makebox[0.084\textwidth]{\texttt{HELL}}\\
\makebox[0.084\textwidth]{---}
\makebox[0.084\textwidth]{\texttt{BIS}}
\makebox[0.084\textwidth]{---}
\makebox[0.084\textwidth]{\texttt{BSKL}}
\makebox[0.084\textwidth]{---}
\makebox[0.084\textwidth]{---}
\makebox[0.084\textwidth]{---}
\makebox[0.084\textwidth]{---}
\makebox[0.084\textwidth]{---}
\makebox[0.084\textwidth]{\texttt{EUC}}
\makebox[0.084\textwidth]{---}\\
\caption{Separable regularizers on $(0, 1)$.}
\label{fig:reg}
\end{figure}

\subsection{Boltzmann-Shannon Entropy and Kullback-Leibler Divergence}
\label{subsec:kl}

Assumptions (A) hold for minus the Boltzmann-Shannon entropy $\pi \log \pi - \pi + 1$ associated to the Kullback-Leibler divergence. Hence, the ROT problem can be solved with the ASA scheme. In addition, the updates in the POCS technique can be written analytically, leading to the Sinkhorn-Knopp algorithm. Specifically, the two projections amount to normalizing in turn the rows and columns of $\boldsymbol\pi^\star$ so that they sum up to $\p$ and $\q$ respectively:
\begin{align}
\boldsymbol\pi^\star & \leftarrow \diag\left(\frac{\p}{\boldsymbol\pi^\star \ones}\right) \boldsymbol\pi^\star \enspace,\\
\boldsymbol\pi^\star & \leftarrow \boldsymbol\pi^\star \diag\left(\frac{\q}{{\boldsymbol\pi^\star}^\top \ones}\right) \enspace.
\end{align}

This can be optimized by remarking that the iterates ${\boldsymbol\pi^\star}^{(k)}$ after each couple of projections verify:
\begin{equation}
{\boldsymbol\pi^\star}^{(k)} = \diag(\u^{(k)}) \boldsymbol\xi \diag(\v^{(k)}) \enspace,
\end{equation}
where $\boldsymbol\xi = \exp(-\boldsymbol\gamma / \lambda)$, and vectors $\u^{(k)}, \v^{(k)}$ satisfy the following recursion:
\begin{align}
\u^{(k)} & = \frac{\p}{\boldsymbol\xi \v^{(k - 1)}} \enspace,\\
\v^{(k)} & = \frac{\q}{\boldsymbol\xi^\top \u^{(k)}} \enspace,
\end{align}
with convention $\v^{(0)} = \ones$. This allows a fast implementation by performing only matrix-vector multiplications using a fixed matrix $\boldsymbol\xi = \exp(-\boldsymbol\gamma / \lambda)$. We can further save one element-wise vector multiplication per update:
\begin{align}
\u & \leftarrow \frac{\ones}{\diag\left(\frac{\ones}{\p}\right) \boldsymbol\xi \, \v} \enspace,\\
\v & \leftarrow \frac{\ones}{\diag\left(\frac{\ones}{\q}\right) \boldsymbol\xi^\top \, \u} \enspace,
\end{align}
where the matrices $\diag\left(\frac{\ones}{\p}\right) \boldsymbol\xi$ and $\diag\left(\frac{\ones}{\q}\right) \boldsymbol\xi^\top$ are precomputed and stored.

\subsection{Burg Entropy and Itakura-Saito Divergence}
\label{subsec:is}

Assumptions (A) also hold for minus the Burg entropy $\pi - \log \pi - 1$ associated to the Itakura-Saito divergence, so the ROT problem can be solved with the ASA scheme. Eventually, the Newton-Raphson steps to update the alternate projections in POCS technique can be written as follows:
\begin{align}
\boldsymbol\theta^\star & \leftarrow \boldsymbol\theta^\star - \frac{{(1 - \boldsymbol\theta^\star)}^{-1} \, \ones - \p}{{(1 - \boldsymbol\theta^\star)}^{-2} \, \ones} \, \ones^\top \enspace,\\
\boldsymbol\theta^\star & \leftarrow \boldsymbol\theta^\star - \ones \, \frac{\ones^\top \, {(1 - \boldsymbol\theta^\star)}^{-1} - \q^\top}{\ones^\top \, {(1 - \boldsymbol\theta^\star)}^{-2}} \enspace.
\end{align}

Each step can be optimized by computing first an element-wise matrix inverse ${(1 - \boldsymbol\theta^\star)}^{-1}$ for the numerator, and then performing an element-wise matrix multiplication of this matrix by itself to obtain a matrix for the denominator instead of applying an additional element-wise matrix power. Since $\psi'$ is convex and strictly increasing with $\psi''$ positive everywhere, the convergence of the updates is guaranteed.

\subsection{Fermi-Dirac Entropy and Logistic Loss Function}
\label{subsec:log}

Assumptions (A) again hold for minus the Fermi-Dirac entropy $\pi \log \pi + (1 - \pi) \log (1 - \pi)$, also known as bit entropy, associated to a logistic loss function. The ROT problem can thus be solved with the ASA scheme, and the Newton-Raphson steps to update the alternate projections in the POCS technique can be written as follows:
\begin{align}
\boldsymbol\theta^\star & \leftarrow \boldsymbol\theta^\star - \frac{\frac{\exp \boldsymbol\theta^\star}{1 + \exp \boldsymbol\theta^\star} \, \ones - \p}{\frac{\exp \boldsymbol\theta^\star}{{(1 + \exp \boldsymbol\theta^\star)}^2} \, \ones} \, \ones^\top \enspace,\\
\boldsymbol\theta^\star & \leftarrow \boldsymbol\theta^\star - \ones \, \frac{\frac{\exp \boldsymbol\theta^\star}{1 + \exp \boldsymbol\theta^\star} - \q^\top}{\ones^\top \, \frac{\exp \boldsymbol\theta^\star}{{(1 + \exp \boldsymbol\theta^\star)}^2}} \enspace.
\end{align}

Each step can be optimized by storing first the element-wise matrix exponential $\exp \boldsymbol\theta^\star$, then applying an element-wise matrix division by the temporary matrix $1 + \exp \boldsymbol\theta^\star$ to obtain a matrix for the numerator, and lastly performing an element-wise matrix division of these two matrices to obtain a matrix for the denominator and thus save an additional element-wise matrix power as well as several element-wise matrix exponentials. However, even if $\psi'$ is strictly increasing with $\psi''$ positive everywhere, $\psi'$ is neither convex nor concave and does not verify the necessary and sufficient condition~\eqref{eq:nasc} for global convergence of the Newton-Raphson method.

Nevertheless, $\psi'$ is convex on $\R_-$ and concave on $\R_+$. It thus divides for a given $1 \leq i \leq d$, respectively $1 \leq j \leq d$, the real line into at most $d + 1$ intervals $-\infty < \hat{\theta}_i^{(1)} \leq \hat{\theta}_i^{(2)} \leq \dotsb \leq \hat{\theta}_i^{(d - 1)} \leq \hat{\theta}_i^{(d)} < +\infty$, respectively $-\infty < \check{\theta}_j^{(1)} \leq \check{\theta}_j^{(2)} \leq \dotsb \leq \check{\theta}_j^{(d - 1)} \leq \check{\theta}_j^{(d)} < +\infty$, with the values ${(\hat{\theta}_i^{(k)})}_{1 \leq k \leq d}$ from row $i$ of $\overline{\boldsymbol\theta}$, respectively ${(\check{\theta}_j^{(k)})}_{1 \leq k \leq d}$ from column $j$ of $\overline{\boldsymbol\theta}$, sorted in increasing order. On each of these intervals, the necessary and sufficient condition~\eqref{eq:nasc} is verified since we can decompose $f(\mu_i)$, respectively $g(\nu_j)$, as the sum of an increasing convex and an increasing concave function. Hence, we have global convergence on the interval that contains the solution. It is further possible to restrict the search to the two last intervals only. Indeed, we have $\sum_{j = 1}^d \psi'(\overline{\theta}_{ij} - \hat{\theta}_i^{(d - 1)}) \geq \psi'(\hat{\theta}_i^{(d - 1)} - \hat{\theta}_i^{(d - 1)}) + \psi'(\hat{\theta}_i^{(d)} - \hat{\theta}_i^{(d - 1)}) \geq 2 \psi'(0) = 1$, so that $\hat{\theta}_i^{(d - 1)} < \mu_i < +\infty$. Similarly, we have $\sum_{i = 1}^d \psi'(\overline{\theta}_{ij} - \check{\theta}_j^{(d - 1)}) \geq \psi'(\check{\theta}_j^{(d - 1)} - \check{\theta}_j^{(d - 1)}) + \psi'(\check{\theta}_j^{(d)} - \check{\theta}_j^{(d - 1)}) \geq 2 \psi'(0) = 1$, so that $\check{\theta}_j^{(d - 1)} < \nu_j < +\infty$. As a result, it suffices to initialize $\mu_i$ with $\hat{\theta}_i = \hat{\theta}_i^{(d)} = \max{\{\overline{\theta}_{ij}\}}_{1 \leq j \leq d}$, respectively $\nu_j$ with $\check{\theta}_j = \check{\theta}_j^{(d)} = \max{\{\overline{\theta}_{ij}\}}_{1 \leq i \leq d}$, to guarantee convergence of the updates.

\subsection[Beta-potentials and beta-divergences]{$\beta$-potentials and $\beta$-divergences}
\label{subsec:beta}

Assumptions (A) hold for the $\beta$-potentials $(\pi^\beta - \beta \pi + \beta - 1) / (\beta (\beta - 1))$ with $0 < \beta < 1$, associated to the so-called $\beta$-divergences. Hence, the ROT problem can be solved with the ASA scheme, and the Newton-Raphson steps to update the alternate projections in the POCS technique can be written as follows:
\begin{align}
\boldsymbol\theta^\star & \leftarrow \boldsymbol\theta^\star - \frac{{((\beta - 1) \boldsymbol\theta^\star + 1)}^{\frac{1}{\beta - 1}} \, \ones - \p}{{((\beta - 1) \boldsymbol\theta^\star + 1)}^{\frac{1}{\beta - 1} - 1} \, \ones} \, \ones^\top \enspace,\\
\boldsymbol\theta^\star & \leftarrow \boldsymbol\theta^\star - \ones \, \frac{\ones^\top \, {((\beta - 1) \boldsymbol\theta^\star + 1)}^{\frac{1}{\beta - 1}} - \q^\top}{\ones^\top \, {((\beta - 1) \boldsymbol\theta^\star + 1)}^{\frac{1}{\beta - 1} - 1}} \enspace.
\end{align}

Each step can be optimized by computing first the temporary matrix $(\beta - 1) \boldsymbol\theta^\star + 1$, then applying an element-wise matrix power of $1 / (\beta - 1) - 1$ to this temporary matrix to obtain a matrix for the denominator, and lastly performing an element-wise matrix multiplication of these two matrices to obtain a matrix for the numerator and thus save one element-wise matrix power. Since $\psi'$ is convex and strictly increasing with $\psi''$ positive, the convergence of the updates is guaranteed.

Interestingly, the regularizer tends to minus the Burg and Boltzmann-Shannon entropies in the limit $\beta = 0$ and $\beta = 1$, respectively. Therefore, the $\beta$-divergences interpolate between the Itakura-Saito and Kullback-Leibler divergences. We finally remark that the regularizer can also be defined for other values of the parameter $\beta$ using the same formula, but do not verify assumptions (A) for these values.

\subsection[Lp quasi-norms]{$\ell_p$ quasi-norms}
\label{subsec:lpqn}

Considering regularizers $-\pi^p$ with $0 < p < 1$, all assumptions (A) are verified except from (A5) since $\R_-^{d \times d} \not\subset \dom \psi = \R_{--}^{d \times d}$. Hence, our primal formulation does not hold here because $\zeros \notin \dom \nabla\psi$. However, it is straightforward to check that our dual formulation for ROT problems with the ASA scheme can still be applied as long as the cost matrix $\boldsymbol\gamma$ does not have null entries so that $-\boldsymbol\gamma / \lambda \in \dom \nabla\psi$. Eventually, the Newton-Raphson steps to update the alternate projections in the POCS technique can be written as follows:
\begin{align}
\boldsymbol\theta^\star & \leftarrow \boldsymbol\theta^\star + \frac{{(-\boldsymbol\theta^\star)}^{\frac{1}{p - 1}} \, \ones - p^{\frac{1}{p - 1}} \, \p}{\frac{1}{p - 1} {(-\boldsymbol\theta^\star)}^{\frac{1}{p - 1} - 1} \, \ones} \, \ones^\top \enspace,\\
\boldsymbol\theta^\star & \leftarrow \boldsymbol\theta^\star + \ones \, \frac{\ones^\top \, {(-\boldsymbol\theta^\star)}^{\frac{1}{p - 1}} - p^{\frac{1}{p - 1}} \, \q^\top}{\frac{1}{p - 1} \, \ones^\top \, {(-\boldsymbol\theta^\star)}^{\frac{1}{p - 1} - 1}} \enspace.
\end{align}

Each step can be optimized by computing first the temporary matrix $-\boldsymbol\theta^\star$, then applying an element-wise matrix power of $1 / (p - 1) - 1$ to obtain a matrix for the denominator, and lastly performing an element-wise matrix multiplication of these two matrices to obtain a matrix for the numerator and thus save one element-wise matrix power. Since $\psi'$ is convex and strictly increasing with $\psi''$ positive everywhere, the convergence of the updates is guaranteed.

\subsection[Lp norms]{$\ell_p$ norms}
\label{subsec:lpn}

Assumptions (B) hold for the $\ell_p$ norms $|\pi|^p$ with $1 < p < +\infty$, so the ROT problem can be solved with the NASA scheme. For $p \neq 2$, the Newton-Raphson steps to update the alternate projections in Dykstra's algorithm can be written as follows:
\begin{align}
\boldsymbol\tau & \leftarrow \boldsymbol\tau + \frac{\left\{\sgn(\boldsymbol\theta^\star - \boldsymbol\tau \ones^\top) \odot {|\boldsymbol\theta^\star - \boldsymbol\tau \ones^\top|}^{\frac{1}{p - 1}}\right\} \ones - p^{\frac{1}{p - 1}} \, \p}{\frac{1}{p - 1} {|\boldsymbol\theta^\star - \boldsymbol\tau \ones^\top|}^{\frac{1}{p - 1} - 1} \, \ones} \enspace,\\
\boldsymbol\sigma & \leftarrow \boldsymbol\sigma + \frac{\ones^\top \left\{\sgn(\boldsymbol\theta^\star - \boldsymbol\tau \ones^\top) \odot {|\boldsymbol\theta^\star - \boldsymbol\tau \ones^\top|}^{\frac{1}{p - 1}}\right\} - p^{\frac{1}{p - 1}} \, \q^\top}{\frac{1}{p - 1} \, \ones^\top \, {|\boldsymbol\theta^\star - \boldsymbol\tau \ones^\top|}^{\frac{1}{p - 1} - 1}} \enspace.
\end{align}

Denoting $\overline{\boldsymbol\theta} = \boldsymbol\theta^\star - \boldsymbol\tau \ones^\top$ or $\overline{\boldsymbol\theta} = \boldsymbol\theta^\star - \ones \boldsymbol\sigma^\top$ in the respective updates, each step can be optimized by computing first the temporary matrix $|\overline{\boldsymbol\theta}|$, then applying an element-wise matrix power of $1 / (p - 1) - 1$ to obtain a matrix for the denominator, and lastly performing an element-wise matrix multiplication of these two matrices and of $\sgn{\overline{\boldsymbol\theta}}$ to obtain a matrix for the numerator and thus save one element-wise matrix power as well as several vector replications and matrix subtractions. However, even if $\psi'$ is strictly increasing with $\psi'' > 0$ on $\R^*$, $\psi'$ is neither convex nor concave and does not verify the necessary and sufficient condition~\eqref{eq:nasc} for global convergence of the Newton-Raphson method. Moreover, $\psi''$ vanishes at $0$ for $p < 2$, and $\psi'$ is not differentiable at $0$ for $p > 2$.

Nevertheless, $\psi'$ is concave on $\R_-$ and convex on $\R_+$ for $p < 2$, as well as convex on $\R_-$ and concave on $\R_+$ for $p > 2$. It thus divides for a given $1 \leq i \leq d$, respectively $1 \leq j \leq d$, the real line into at most $d + 1$ intervals $-\infty < \hat{\theta}_i^{(1)} \leq \hat{\theta}_i^{(2)} \leq \dotsb \leq \hat{\theta}_i^{(d - 1)} \leq \hat{\theta}_i^{(d)} < +\infty$, respectively $-\infty < \check{\theta}_j^{(1)} \leq \check{\theta}_j^{(2)} \leq \dotsb \leq \check{\theta}_j^{(d - 1)} \leq \check{\theta}_j^{(d)} < +\infty$, with the values ${(\hat{\theta}_i^{(k)})}_{1 \leq k \leq d}$ from row $i$ of $\overline{\boldsymbol\theta}$, respectively ${(\check{\theta}_j^{(k)})}_{1 \leq k \leq d}$ from column $j$ of $\overline{\boldsymbol\theta}$, sorted in increasing order. The necessary and sufficient condition~\eqref{eq:nasc} is verified on the interior of each of these intervals since we can decompose $f(\mu_i)$, respectively $g(\nu_j)$, as the sum of an increasing convex and an increasing concave function. Hence, we have global convergence on the interior of the interval that contains the solution. In both cases, we must remove the finite endpoints to ensure differentiability of $\psi'$ and positivity of $\psi''$. It is also further possible to prune the last interval from the search. Indeed, we have $\sum_{j = 1}^d \psi'(\overline{\theta}_{ij} - \hat{\theta}_i^{(d)}) \leq \sum_{j = 1}^d \psi'(0) = 0$, so that $\mu_i < \hat{\theta}_i = \hat{\theta}_i^{(d)} = \max{\{\overline{\theta}_{ij}\}}_{1 \leq j \leq d}$. Similarly, we have $\sum_{i = 1}^d \psi'(\overline{\theta}_{ij} - \check{\theta}_j^{(d)}) \leq \sum_{i = 1}^d \psi'(0) = 0$, so that $\nu_j < \check{\theta}_j = \check{\theta}_j^{(d)} = \max{\{\overline{\theta}_{ij}\}}_{1 \leq i \leq d}$. Lastly, we can restrict the first interval with a finite lower bound instead. Indeed, we have $\sum_{j = 1}^d \psi'(\overline{\theta}_{ij} - \hat{\theta}_i^{(1)} + \phi'(p_i / d)) \geq \sum_{j = 1}^d \psi'(\phi'(p_i / d)) = p_i$, so that $\mu_i \geq \hat{\theta}_i^{(1)} - \phi'(p_i / d)$. Similarly, we have $\sum_{i = 1}^d \psi'(\overline{\theta}_{ij} - \check{\theta}_j^{(1)} + \phi'(q_j / d)) \geq \sum_{i = 1}^d \psi'(\phi'(q_j / d)) = q_j$, so that $\nu_j \geq \check{\theta}_j^{(1)} - \phi'(q_j / d)$. As a result, we can perform at most $d$ binary searches in parallel to determine within which of the remaining bounded intervals the solutions $\mu_i$, respectively $\nu_j$, lie. Initialization is then done with the midpoint to guarantee convergence of the updates. A given search thus requires a worst-case logarithmic number of tests, each of which requires a linear number of operations, for a total complexity in $O(d^2 \log d)$ instead of $O(d^2)$ if no such binary search were needed.

Now for $p = 2$, the regularizer specializes to the Euclidean norm, leading to the squared Euclidean distance as the associated divergence. In addition, the formula for $\psi''$ still holds with the convention $0^0 = 1$, and $\psi''$ is actually constant equal to $1/2$. Eventually, the projections can be written in closed form, and we can resort to the analytical algorithm derived in the next example specifically for the Euclidean distance, after doubling the penalty $\lambda$ to account for the regularizer being halved.

\subsection{Euclidean Norm and Euclidean Distance}
\label{subsec:euc}

Assumptions (B) hold for half the Euclidean norm $\pi^2 / 2$ associated to half the squared Euclidean distance. Therefore, the ROT problem can be solved with the NASA scheme, where Dykstra's algorithm can actually be written in closed form. Specifically, the non-negative projection reduces to:
\begin{equation}
\boldsymbol\pi^\star \leftarrow \max\{\zeros, \widetilde{\boldsymbol\pi}\} \enspace,
\end{equation}
and is interleaved with the scaling projections which amount to offsetting the rows and columns of $\widetilde{\boldsymbol\pi}$ by an amount such that the rows and columns of $\boldsymbol\pi^\star$ sum up to $\p$ and $\q$ respectively:
\begin{align}
\widetilde{\boldsymbol\pi} & \leftarrow \widetilde{\boldsymbol\pi} - \frac{1}{d} (\boldsymbol\pi^\star \ones - \p) \, \ones^\top \enspace,\\
\widetilde{\boldsymbol\pi} & \leftarrow \widetilde{\boldsymbol\pi} - \frac{1}{d} \, \ones \, (\ones^\top \boldsymbol\pi^\star - \q^\top) \enspace.
\end{align}

As a remark, we notice that half the squared Euclidean distance can be seen as a $\beta$\nobreakdash-divergence using the provided formula for $\beta = 2$. However, the $\beta$\nobreakdash-divergence generated is not of Legendre type because the domain is restricted to $\R_+$, whereas it could actually be extended to $\R$ so that the regularizer would then be of Legendre type. This is why we fall under assumptions (B) rather than assumptions (A) in this case.

\subsection{Hellinger Distance}
\label{subsec:hell}

Assumptions (B) hold for the regularizer $-{(1 - \pi^2)}^{\frac{1}{2}}$ akin to a Hellinger distance. Hence, the ROT problem can be solved with the NASA scheme, and the Newton-Raphson steps to update the alternate projections in Dykstra's algorithm can be written as follows:
\begin{align}
\boldsymbol\tau & \leftarrow \boldsymbol\tau + \frac{\left\{(\boldsymbol\theta^\star - \boldsymbol\tau \ones^\top) \odot {\left(1 + {(\boldsymbol\theta^\star - \boldsymbol\tau \ones^\top)}^2\right)}^{-\frac{1}{2}}\right\} \ones - \p}{{\left(1 + {(\boldsymbol\theta^\star - \boldsymbol\tau \ones^\top)}^2\right)}^{-\frac{3}{2}} \, \ones} \enspace,\\
\boldsymbol\sigma & \leftarrow \boldsymbol\sigma + \frac{\ones^\top \left\{(\boldsymbol\theta^\star - \ones \boldsymbol\sigma^\top) \odot {\left(1 + {(\boldsymbol\theta^\star - \ones \boldsymbol\sigma^\top)}^2\right)}^{-\frac{1}{2}}\right\} - \q^\top}{\ones^\top \, {\left(1 + {(\boldsymbol\theta^\star - \ones \boldsymbol\sigma^\top)}^2\right)}^{-\frac{3}{2}}} \enspace.
\end{align}

Denoting $\overline{\boldsymbol\theta} = \boldsymbol\theta^\star - \boldsymbol\tau \ones^\top$ or $\overline{\boldsymbol\theta} = \boldsymbol\theta^\star - \ones \boldsymbol\sigma^\top$ in the respective updates, each step can be optimized by computing first the temporary matrix $1 / (1 + \overline{\boldsymbol\theta}\vphantom{\boldsymbol\theta}^2) $, then applying an element-wise matrix square root to this temporary matrix, performing an element-wise matrix multiplication of these two matrices to obtain a matrix for the denominator, and lastly an element-wise matrix multiplication of the temporary matrix with $\overline{\boldsymbol\theta}$ to obtain a matrix for the numerator and thus save one element-wise matrix power as well as several vector replications and matrix subtractions. However, even if $\psi'$ is strictly increasing with $\psi''$ positive everywhere, $\psi'$ is neither convex nor concave and does not verify the necessary and sufficient condition~\eqref{eq:nasc} for global convergence of the Newton-Raphson method.

Nevertheless, $\psi'$ is convex on $\R_-$ and concave on $\R_+$. It thus divides for a given $1 \leq i \leq d$, respectively $1 \leq j \leq d$, the real line into at most $d + 1$ intervals $-\infty < \hat{\theta}_i^{(1)} \leq \hat{\theta}_i^{(2)} \leq \dotsb \leq \hat{\theta}_i^{(d - 1)} \leq \hat{\theta}_i^{(d)} < +\infty$, respectively $-\infty < \check{\theta}_j^{(1)} \leq \check{\theta}_j^{(2)} \leq \dotsb \leq \check{\theta}_j^{(d - 1)} \leq \check{\theta}_j^{(d)} < +\infty$, with the values ${(\hat{\theta}_i^{(k)})}_{1 \leq k \leq d}$ from row $i$ of $\overline{\boldsymbol\theta}$, respectively ${(\check{\theta}_j^{(k)})}_{1 \leq k \leq d}$ from column $j$ of $\overline{\boldsymbol\theta}$, sorted in increasing order. On each of these intervals, the necessary and sufficient condition~\eqref{eq:nasc} is verified since we can decompose $f(\mu_i)$, respectively $g(\nu_j)$, as the sum of an increasing convex and an increasing concave function. Hence, we have global convergence on the interval that contains the solution. It is further possible to prune the last interval from the search. Indeed, we have $\sum_{j = 1}^d \psi'(\overline{\theta}_{ij} - \hat{\theta}_i^{(d)}) \leq \sum_{j = 1}^d \psi'(0) = 0$, so that $\mu_i < \hat{\theta}_i = \hat{\theta}_i^{(d)} = \max{\{\overline{\theta}_{ij}\}}_{1 \leq j \leq d}$. Similarly, we have $\sum_{i = 1}^d \psi'(\overline{\theta}_{ij} - \check{\theta}_j^{(d)}) \leq \sum_{i = 1}^d \psi'(0) = 0$, so that $\nu_j < \check{\theta}_j = \check{\theta}_j^{(d)} = \max{\{\overline{\theta}_{ij}\}}_{1 \leq i \leq d}$. Lastly, we can restrict the first interval with a finite lower bound instead. Indeed, we have $\sum_{j = 1}^d \psi'(\overline{\theta}_{ij} - \hat{\theta}_i^{(1)} + \phi'(p_i / d)) \geq \sum_{j = 1}^d \psi'(\phi'(p_i / d)) = p_i$, so that $\mu_i \geq \hat{\theta}_i^{(1)} - \phi'(p_i / d)$. Similarly, we have $\sum_{i = 1}^d \psi'(\overline{\theta}_{ij} - \check{\theta}_j^{(1)} + \phi'(q_j / d)) \geq \sum_{i = 1}^d \psi'(\phi'(q_j / d)) = q_j$, so that $\nu_j \geq \check{\theta}_j^{(1)} - \phi'(q_j / d)$. As a result, we can perform $d$ binary searches in parallel to determine within which of the remaining intervals the solutions $\mu_i$, respectively $\nu_j$, lie. Initialization is then done with the midpoint to guarantee convergence of the updates. A given search requires a worst-case logarithmic number of tests, each of which requires a linear number of operations, for a total complexity in $O(d^2 \log d)$ instead of $O(d^2)$ if no such binary search were needed.

\subsection{Quadratic Forms and Mahalanobis Distances}
\label{subsec:mahal}

Assumptions (B) hold for the quadratic forms $(1 / 2) \, {\vect(\boldsymbol\pi)}^\top \P \vect(\boldsymbol\pi)$ with positive-definite matrix $\P \in \R^{d^2 \times d^2}$, associated to the Mahalanobis distances, so the ROT problem can be solved with the NASA scheme. For a diagonal matrix $\P$, the regularizer is separable and the Newton-Raphson steps to update the alternate projections in Dykstra's algorithm are similar to that for the Euclidean distance with appropriate weights. For a non-diagonal matrix $\P$, however, the regularizer is not separable anymore and we must resort to the generic NASA scheme.

In this general case, the scaling projections amount to convex quadratic programs with linear equality constraints. They can be solved using classical techniques such as the range-space and null-space approaches, Krylov subspace methods or active set strategies. The non-negative projection reduces to a convex quadratic program with a linear inequality constraint. It can be solved elegantly with an iterative algorithm for non-negative quadratic programming proposed by \cite{Sha2007} using multiplicative updates with a complexity in $O(d^4)$. All in all, we recommend using a sparse matrix $\P$ with a block-diagonal structure and an order of magnitude of $d^2$ non-null entries, so as to obtain a quadratic instead of quartic empirical complexity.

\section{Experimental Results}
\label{sec:experiments}

In this section, we present the results of our methods on different experiments. We first design an synthetic test to showcase the behavior of different regularizers and penalties on the output solutions or computational times (Section~\ref{subsec:synth}). We then consider a pattern recognition application to audio scene classification on a real-world dataset (Section~\ref{subsec:audio}).

\subsection{Synthetic Data}
\label{subsec:synth}

We start by visualizing the effects of different regularizers $\phi$ and varying penalties $\lambda$ on synthetic data. For the input distributions, we discretize and normalize continuous densities on a uniform grid ${(x_i)}_{1 \leq i \leq d}$ of $[0, 1]$ with dimension $d = 256$. We use for $\p$ a univariate normal with mean $0.5$ and variance $0.2$, and for $\q$ a mixture of two normals with equal weights, respective means $0.25$ and $0.75$, and same variance $0.1$. We set the cost matrix $\boldsymbol\gamma$ as the squared Euclidean distance $\gamma_{ij} = {(x_i - x_j)}^2$ on the grid. The input distributions (bottom left and top right), cost matrix (top left) and unique earth mover's plan (bottom right) computed for classical OT using the solver of~\cite{Rubner2000} with standard settings, are shown in Figure~\ref{fig:data}.

\begin{figure}[t!]
\centering
\includegraphics[width=4cm, height=2cm]{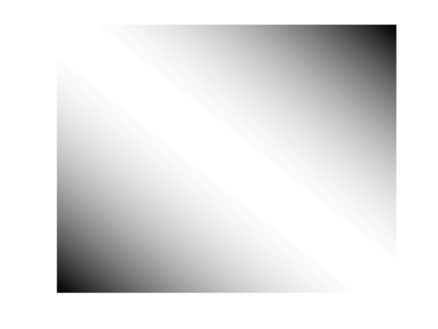}
\includegraphics[width=4cm, height=2cm]{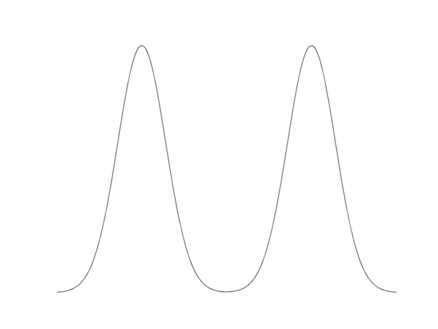}\\
\includegraphics[width=2cm, height=4cm, angle=90]{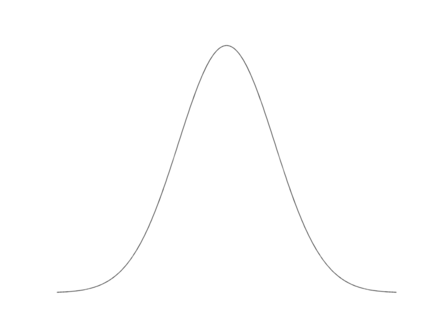}
\includegraphics[width=4cm, height=2cm]{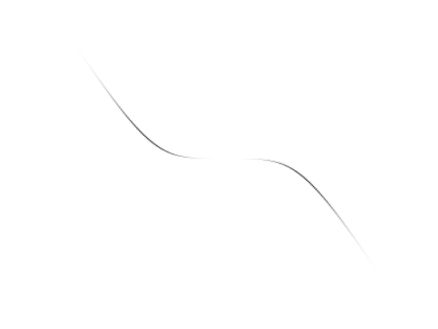}
\caption{Earth mover's plan $\boldsymbol\pi^\star$ for the cost matrix $\boldsymbol\gamma$ and input distributions $\p, \q$.}
\label{fig:data}
\end{figure}

We test all separable regularizers $\phi$ introduced in Section~\ref{sec:examples}. Because these regularizers have different ranges in the sensible values of the rot mover's plans $\boldsymbol\pi^\star$, we manually tune the penalties $\lambda$ so that they feature similar amounts of regularization. For ease of comparison, we set $\lambda = \overline{\lambda} \, \lambda'$, with $\overline{\lambda}$ constant for each $\phi$, and $\lambda'$ varying similarly for all $\phi$. The limit case when $\lambda$ tends to infinity is simply obtained by setting $\boldsymbol\gamma / \lambda = \zeros$ in the algorithms, except from $\ell_p$ quasi\nobreakdash-norms for which we use $\lambda = 10^{10}$. The null values of $\boldsymbol\gamma$ are also fixed to $10^{-12}$ for $\ell_p$ quasi\nobreakdash-norms. We do not limit the number of iterations in the different algorithms, and use a small tolerance of $10^{-8}$ for convergence with the $\ell_\infty$ norm on the marginal difference checked after each iteration as a termination criterion.

The rot mover's plans obtained for ROT for $d = 256$ with the different regularizers and penalties are visualized in Figure~\ref{fig:gamma}. We first observe that all rot mover's plans converge to the earth mover's plan for low values of the penalty as shown theoretically in Property~\ref{property:convergencedown}. Nevertheless, the rot mover's plans exhibit different shapes depending on the regularizers for intermediary and large values of the penalty. In the limit when the penalty grows to infinity, we obtain the transport plan with minimal Bregman information as shown theoretically in Property~\ref{property:convergenceup}. In particular, this leads to $\p\q^\top$ with an ellipsoidal shape for {\texttt{BSKL}} (Boltzmann-Shannon entropy and Kullback-Leibler divergence), meaning that the mass is relatively spread among neighbor bins. The same pattern is observed for {\texttt{FDLOG}} (Fermi-Dirac entropy and logistic loss function), which can be explained in this synthetic example by the rot mover's plans having low values and the two regularizers being equivalent up to a constant in the neighborhood of zero. The profile gets more rectangular for {\texttt{BIS}} (Burg entropy and Itakura-Saito divergence), implying that the mass is even more spread across the different bins. Using an intermediary value $\beta = 0.5$ in {\texttt{BETA}} ($\beta$\nobreakdash-potentials and $\beta$\nobreakdash-divergences) allows the interpolation between these two limits of a rectangle for $\beta = 1$ and an ellipsoid for $\beta = 0$, so that the parameter $\beta$ actually helps to control the spread of mass in the regularization. We observe similar results for {\texttt{LPQN}} ($\ell_p$ quasi-norms) with an ellipsoid for $p = 0.9$, a rectangle for $p = 0.1$, and a shape in between for $p = 0.5$. When the power parameter further increases in {\texttt{LPN}} ($\ell_p$ norms), we obtain new shapes that feature less spread of mass. These shapes for $p = 1.1$ and $p = 1.5$ now interpolate up to a lozenge for $p = 2$ in {\texttt{EUC}} (Euclidean norm and Euclidean distance), so that the parameter $p$ also provides control on the spread of mass. A similar diamond profile is obtained for {\texttt{HELL}} (Hellinger distance), which is due again to the rot mover's plans having low values and the two regularizers being equivalent up to a constant in the neighborhood of zero. Lastly, we remark that varying the penalty between the two extremes allows a smooth interpolation of the earth mover's plan and optimal plan with minimal Bregman information, while keeping similar shapes and effects in terms of spreading of mass.

\begin{figure}[t!]
\centering
\makebox[0.5cm]{}
\makebox[0.5cm]{\rotatebox{90}{$\phi$}}
\makebox[0.5cm]{}
\makebox[0.5cm]{\rotatebox{90}{$\overline{\lambda}$}\rotatebox{45}{/}$\lambda'$}
\makebox[2.4cm]{$10^{-2}$}
\makebox[2.4cm]{$10^{-1}$}
\makebox[2.4cm]{$10^{+0}$}
\makebox[2.4cm]{$10^{+1}$}
\makebox[2.4cm]{$+\infty$}\\
\vspace{0.5cm}
\makebox[0.5cm]{\rotatebox{90}{\makebox[1.1cm]{---}}}
\makebox[0.5cm]{\rotatebox{90}{\makebox[1.1cm]{\texttt{FDLOG}}}}
\makebox[0.5cm]{\rotatebox{90}{\makebox[1.1cm]{---}}}
\makebox[0.5cm]{\rotatebox{90}{\makebox[1.1cm]{$10^{-2}$}}}
\includegraphics[width=2.2cm,height=1.1cm]{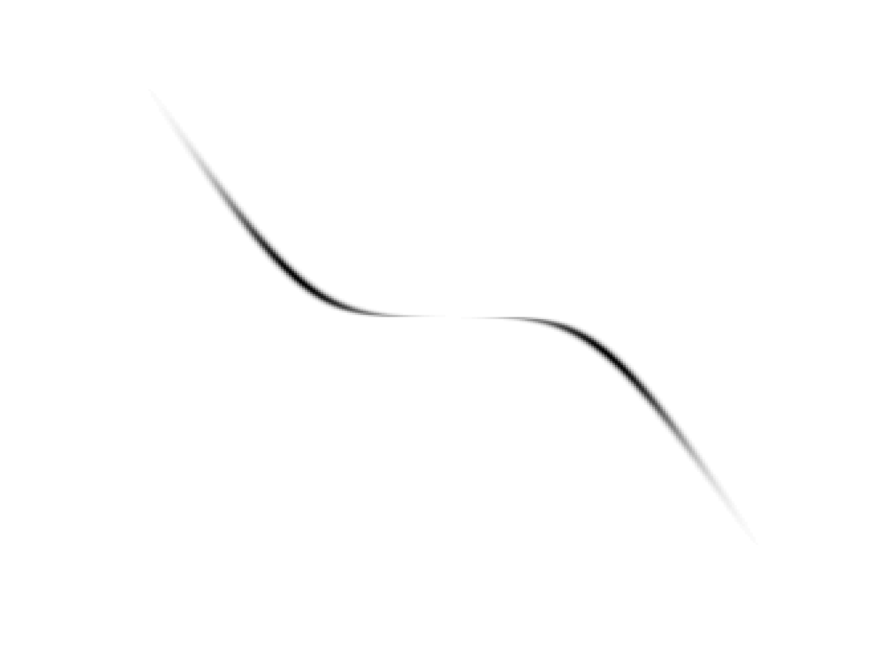}
\includegraphics[width=2.2cm,height=1.1cm]{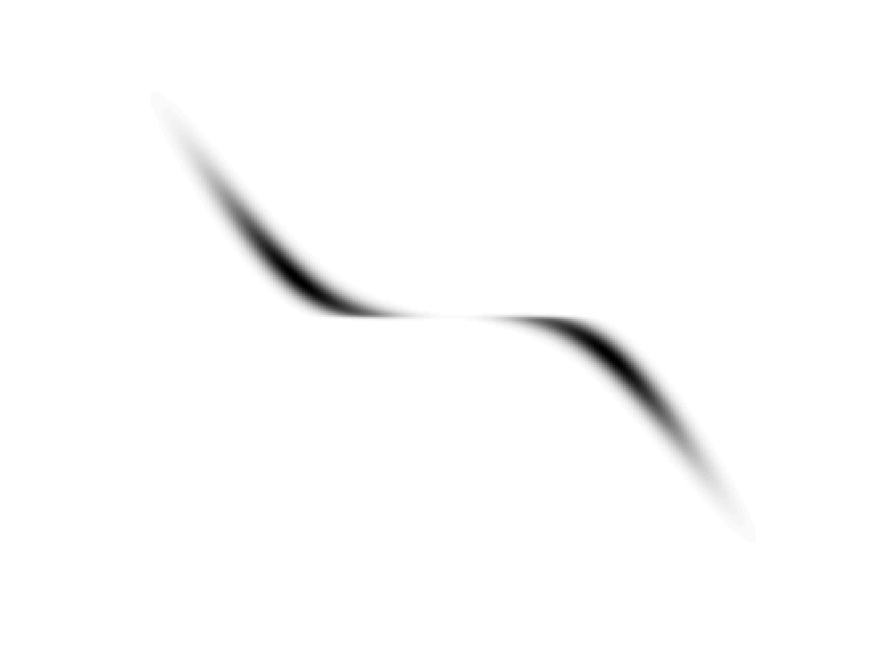}
\includegraphics[width=2.2cm,height=1.1cm]{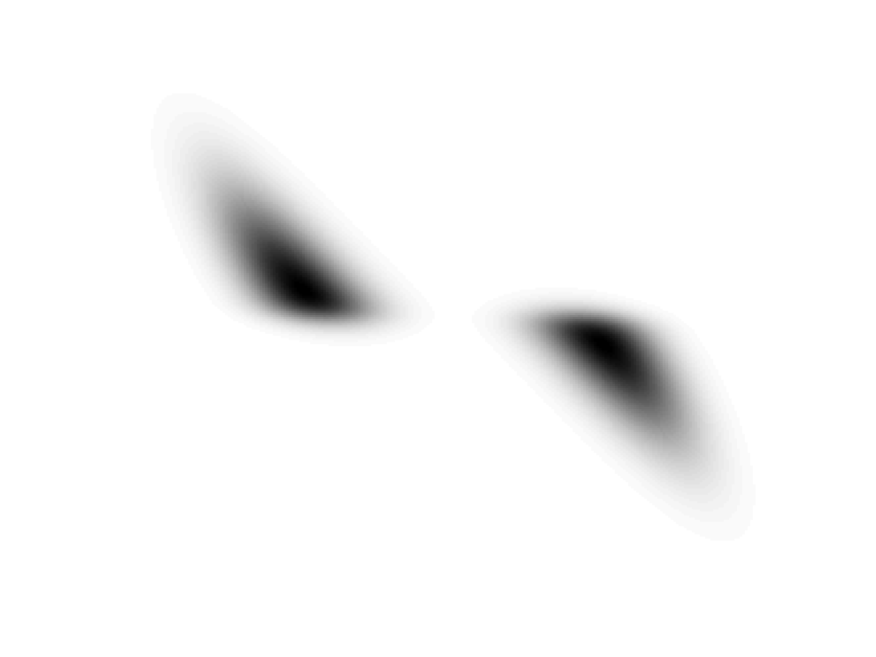}
\includegraphics[width=2.2cm,height=1.1cm]{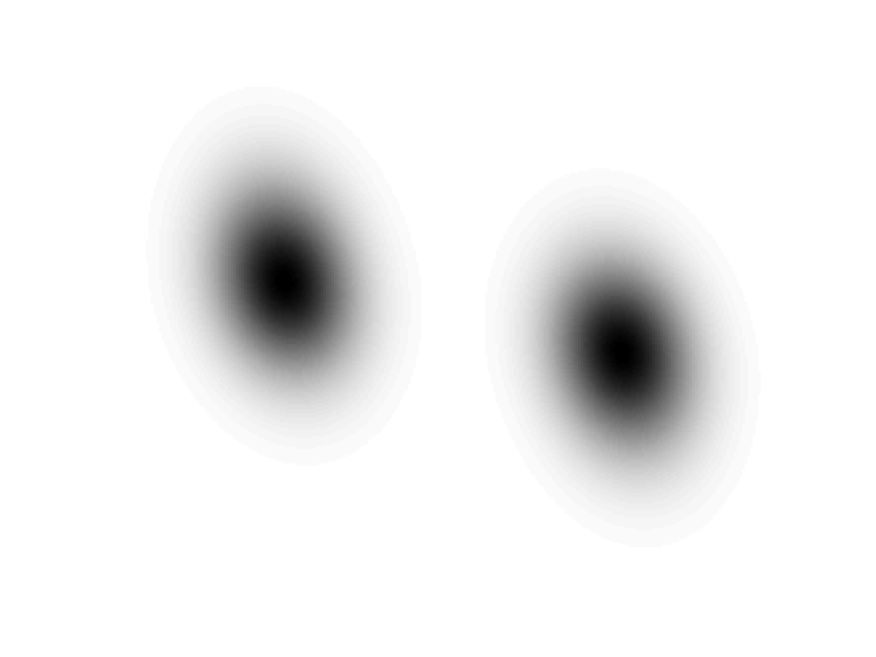}
\includegraphics[width=2.2cm,height=1.1cm]{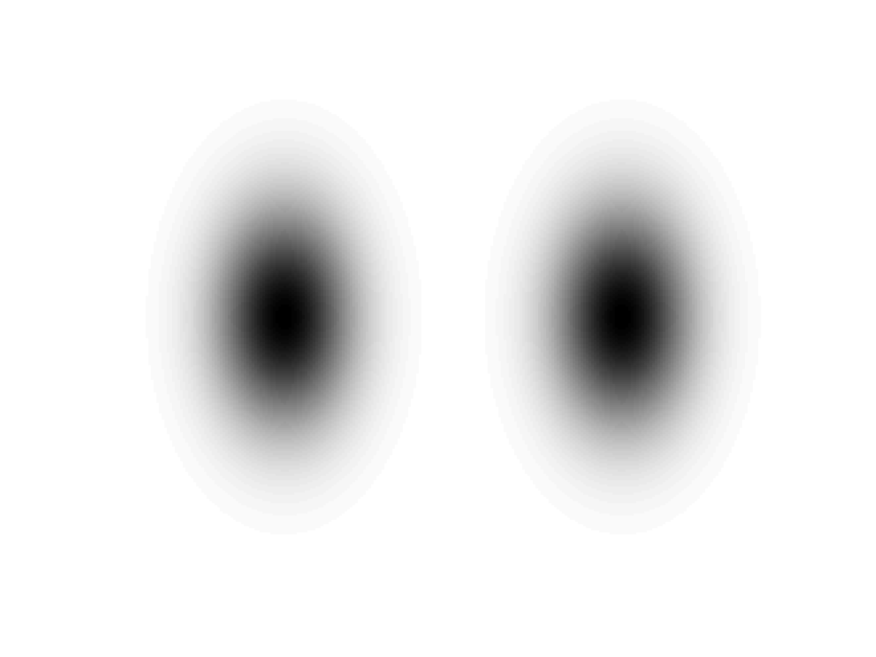}\\
\vspace{0.5cm}
\makebox[0.5cm]{\rotatebox{90}{\makebox[1.1cm]{\texttt{BSKL}}}}
\makebox[0.5cm]{\rotatebox{90}{\makebox[1.1cm]{\texttt{BETA}}}}
\makebox[0.5cm]{\rotatebox{90}{\makebox[1.1cm]{$\beta=1$}}}
\makebox[0.5cm]{\rotatebox{90}{\makebox[1.1cm]{$10^{-2}$}}}
\includegraphics[width=2.2cm,height=1.1cm]{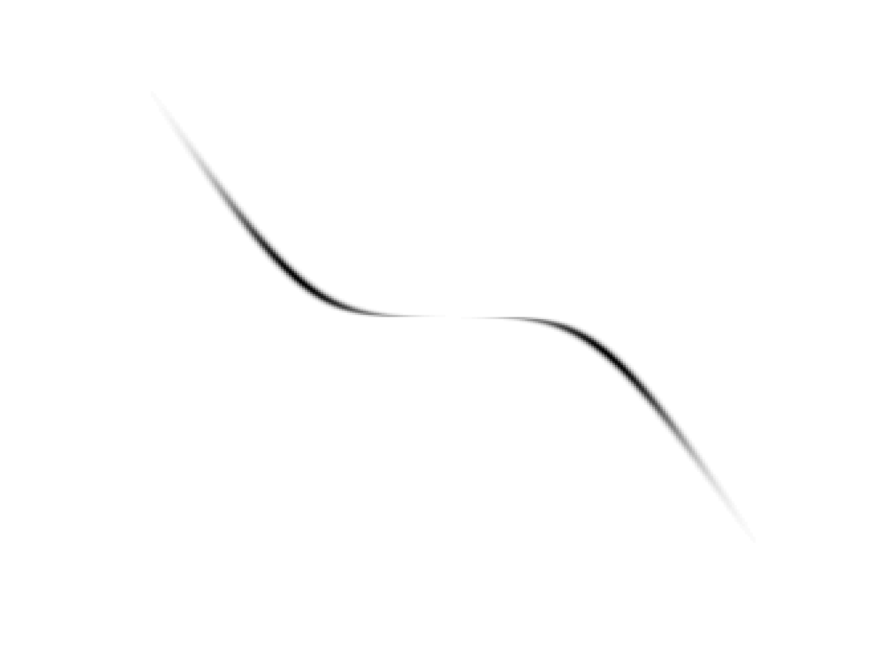}
\includegraphics[width=2.2cm,height=1.1cm]{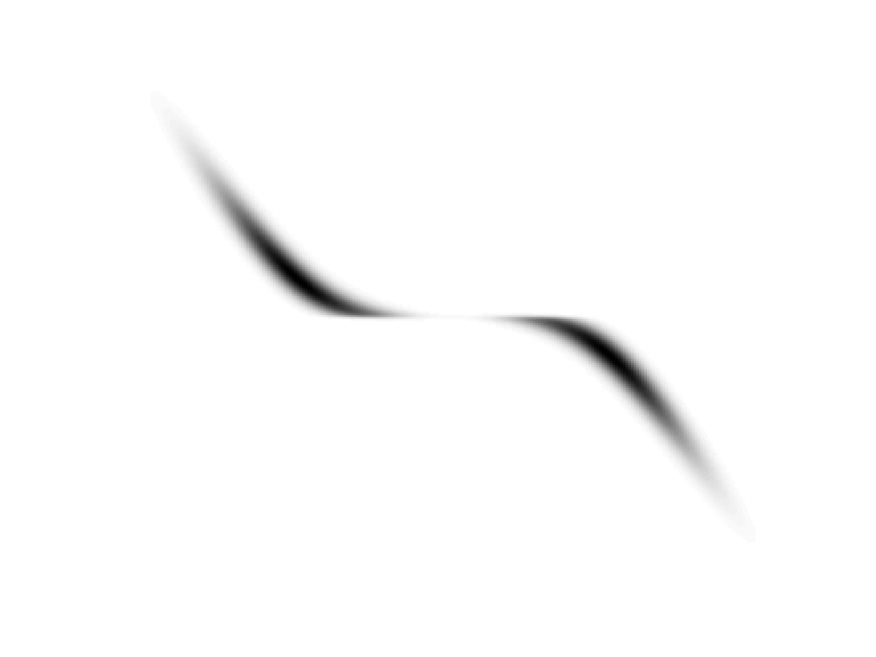}
\includegraphics[width=2.2cm,height=1.1cm]{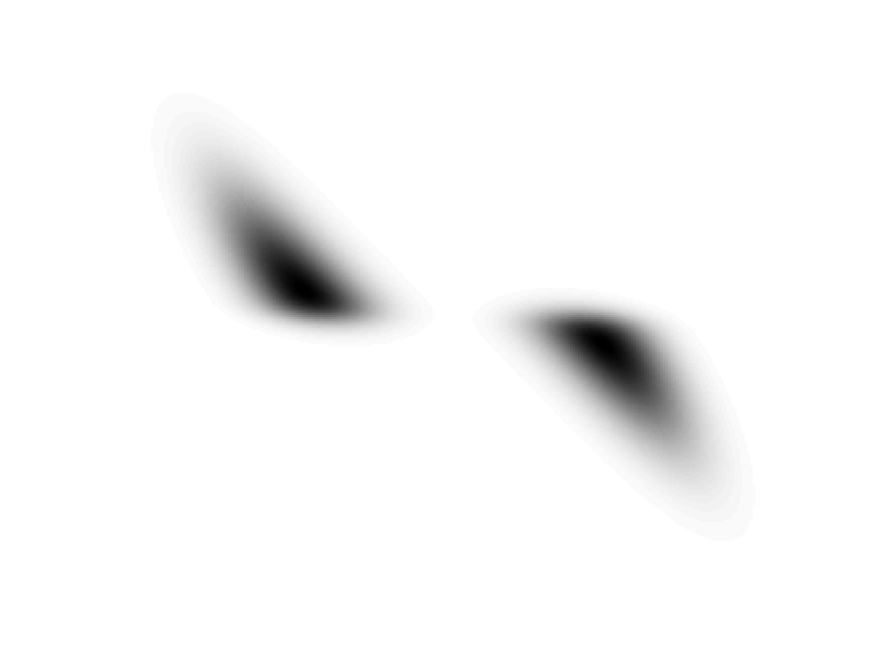}
\includegraphics[width=2.2cm,height=1.1cm]{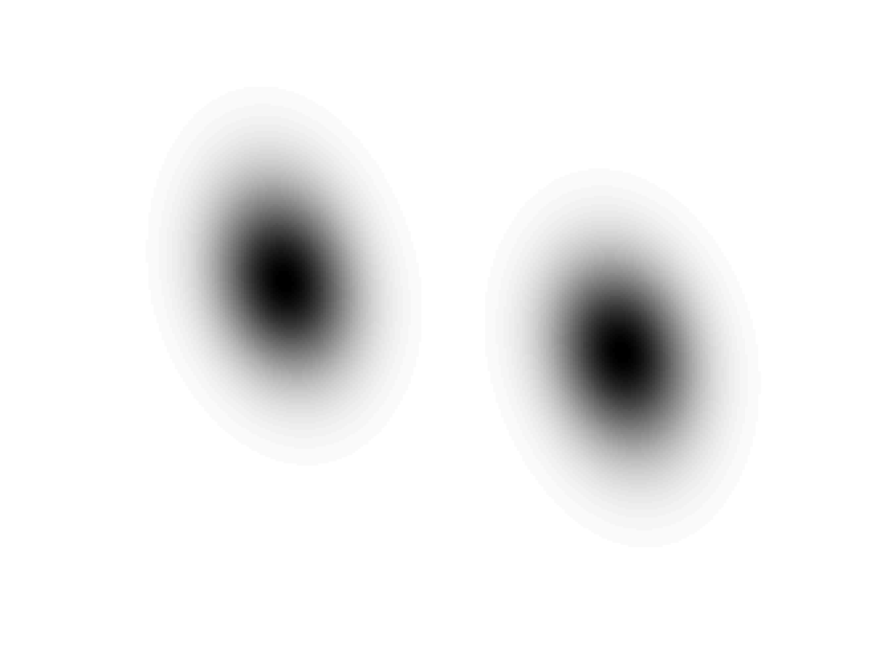}
\includegraphics[width=2.2cm,height=1.1cm]{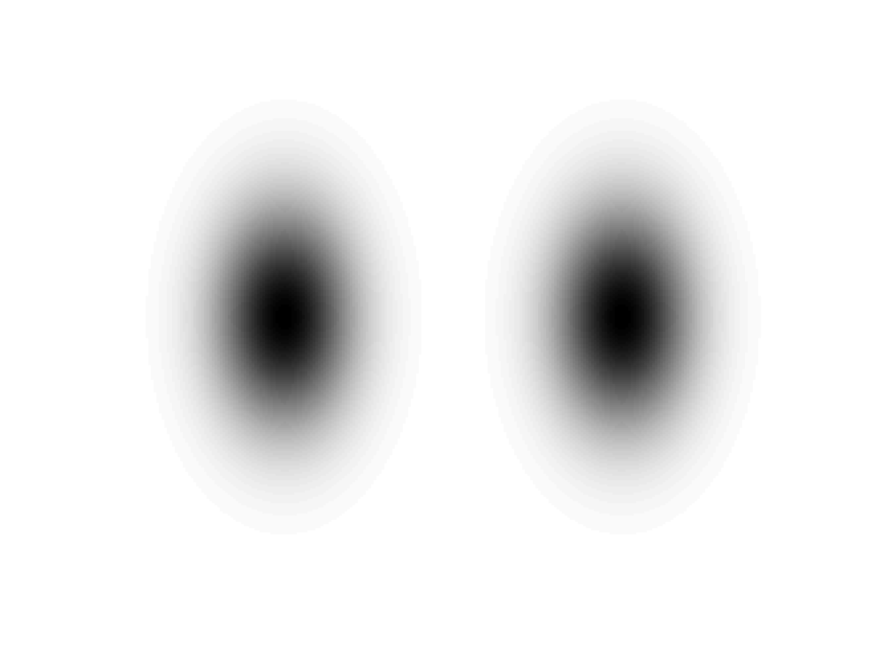}\\
\vspace{0.5cm}
\makebox[0.5cm]{\rotatebox{90}{\makebox[1.1cm]{---}}}
\makebox[0.5cm]{\rotatebox{90}{\makebox[1.1cm]{\texttt{BETA}}}}
\makebox[0.5cm]{\rotatebox{90}{\makebox[1.1cm]{$\beta=0.5$}}}
\makebox[0.5cm]{\rotatebox{90}{\makebox[1.1cm]{$10^{-4}$}}}
\includegraphics[width=2.2cm,height=1.1cm]{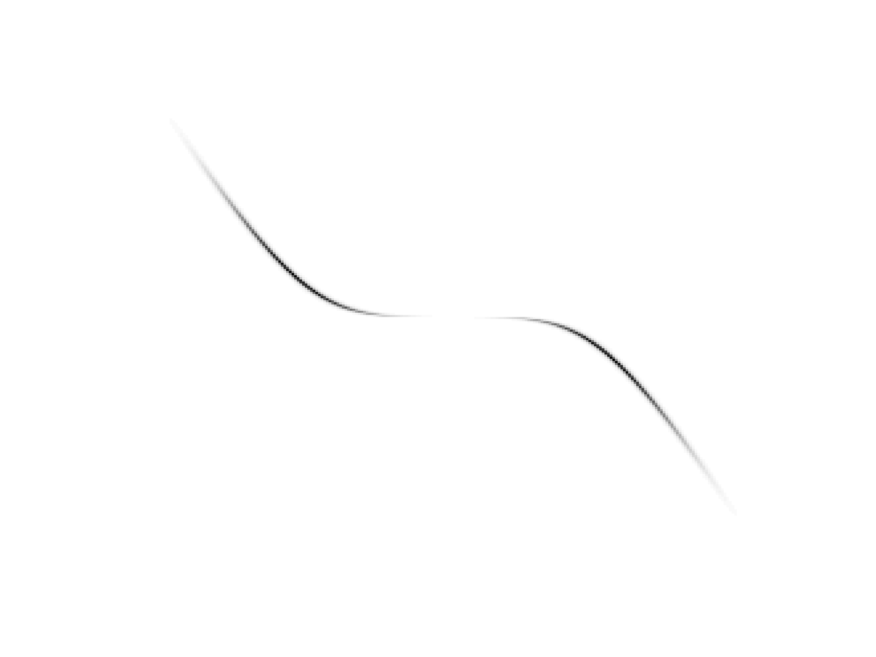}
\includegraphics[width=2.2cm,height=1.1cm]{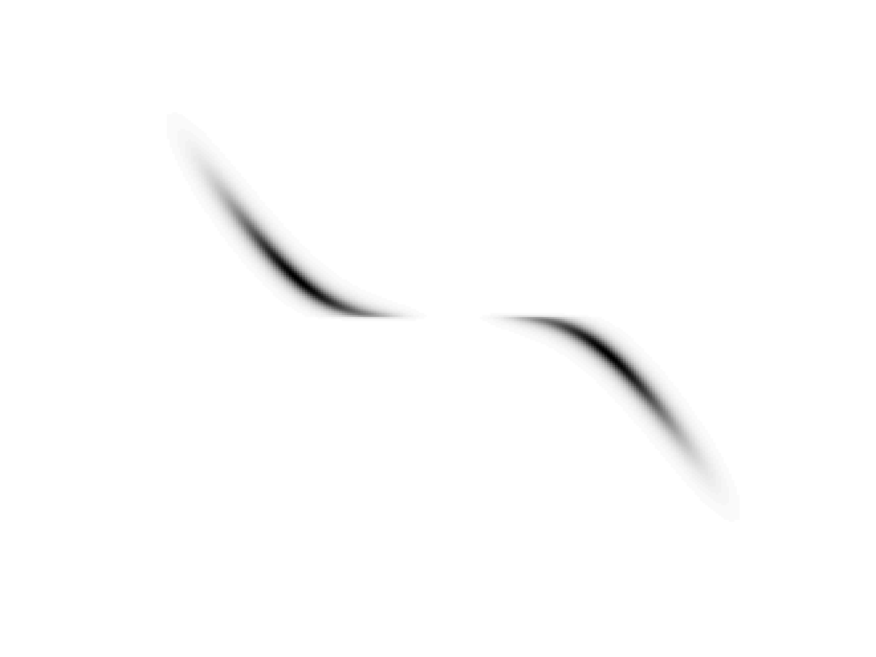}
\includegraphics[width=2.2cm,height=1.1cm]{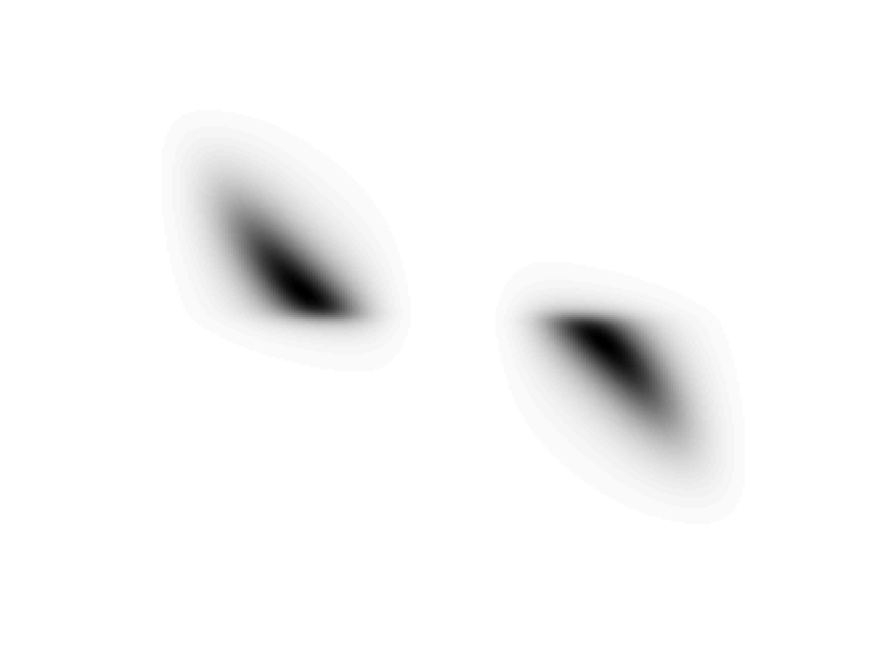}
\includegraphics[width=2.2cm,height=1.1cm]{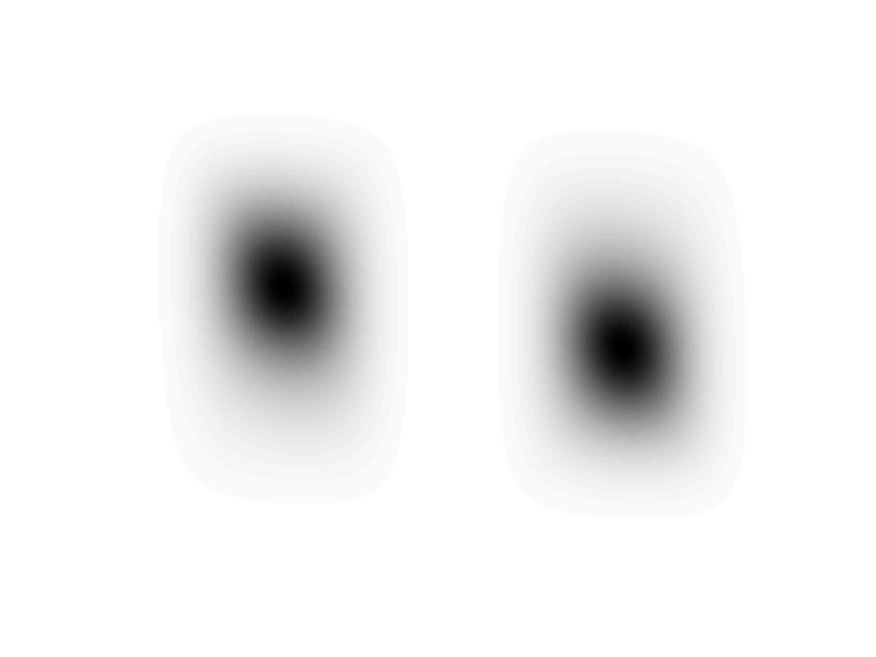}
\includegraphics[width=2.2cm,height=1.1cm]{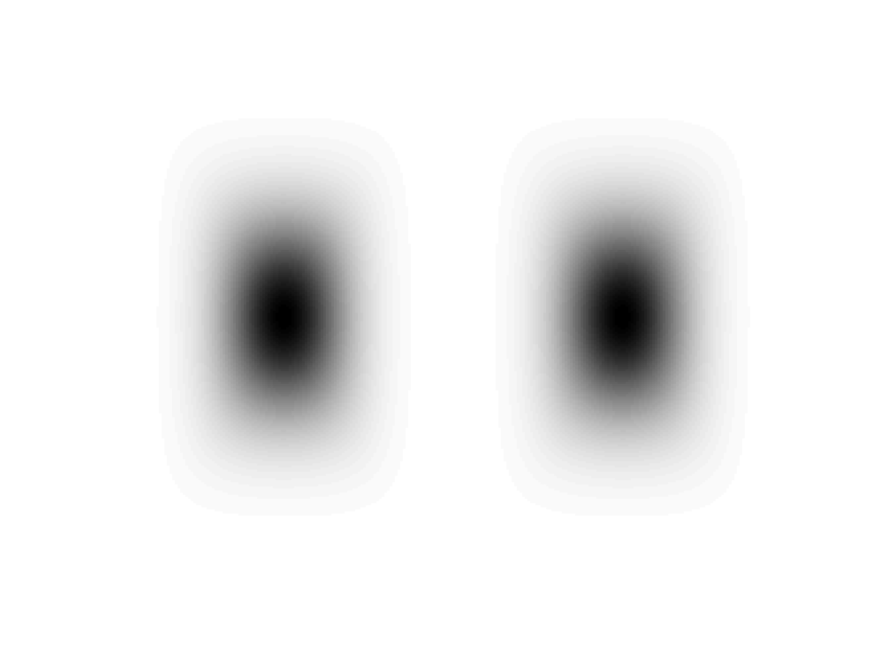}\\
\vspace{0.5cm}
\makebox[0.5cm]{\rotatebox{90}{\makebox[1.1cm]{\texttt{BIS}}}}
\makebox[0.5cm]{\rotatebox{90}{\makebox[1.1cm]{\texttt{BETA}}}}
\makebox[0.5cm]{\rotatebox{90}{\makebox[1.1cm]{$\beta=0$}}}
\makebox[0.5cm]{\rotatebox{90}{\makebox[1.1cm]{$10^{-6}$}}}
\includegraphics[width=2.2cm,height=1.1cm]{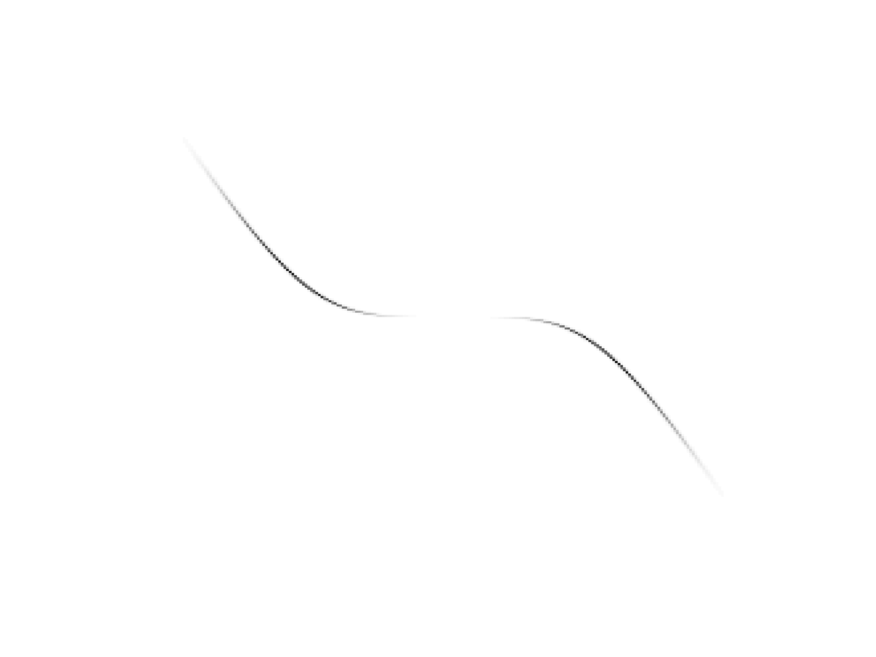}
\includegraphics[width=2.2cm,height=1.1cm]{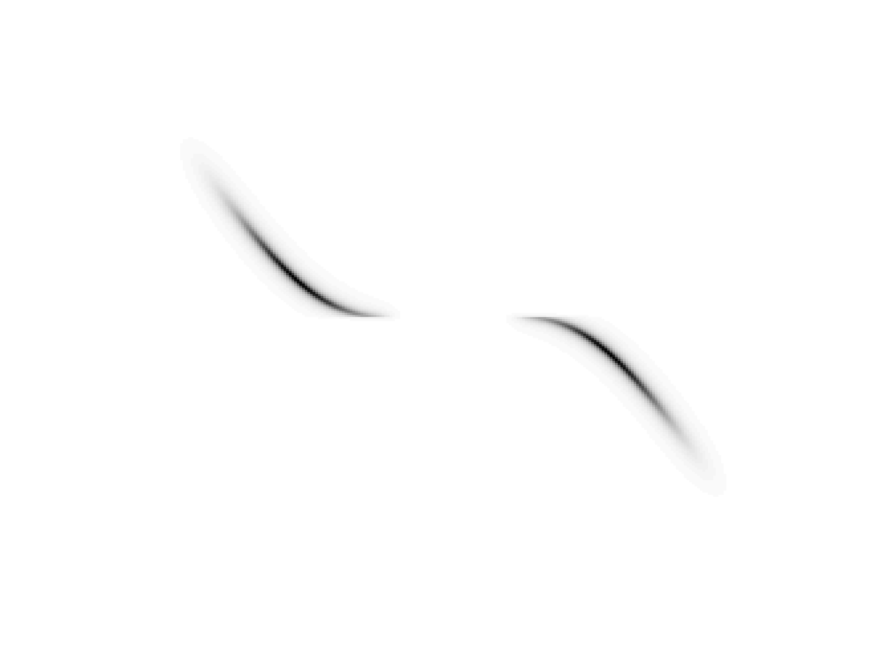}
\includegraphics[width=2.2cm,height=1.1cm]{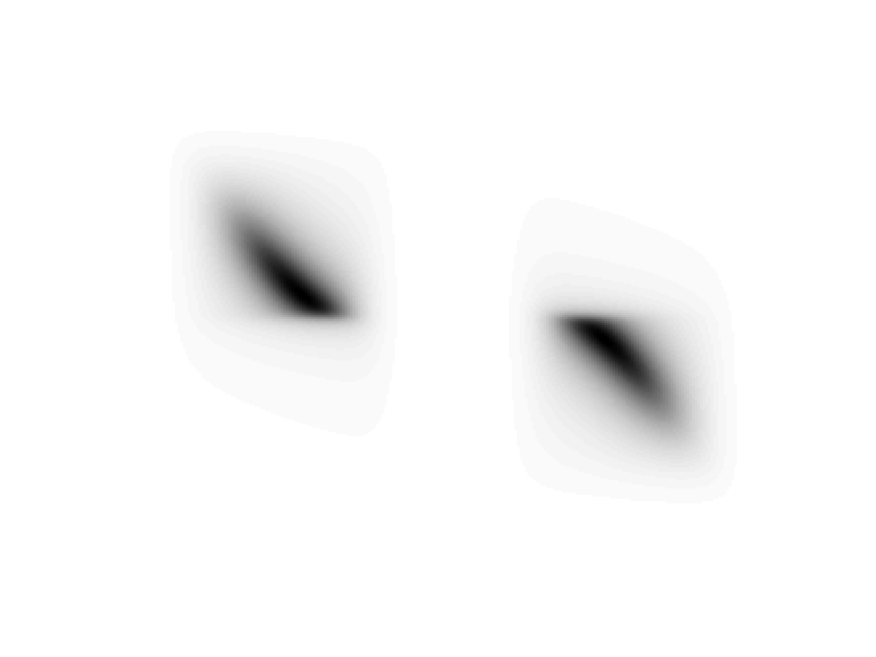}
\includegraphics[width=2.2cm,height=1.1cm]{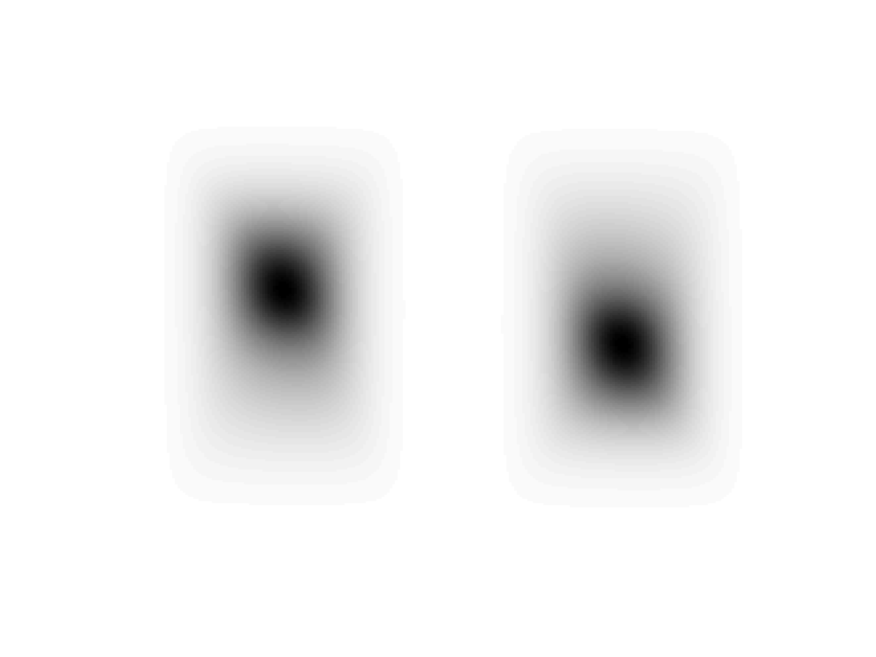}
\includegraphics[width=2.2cm,height=1.1cm]{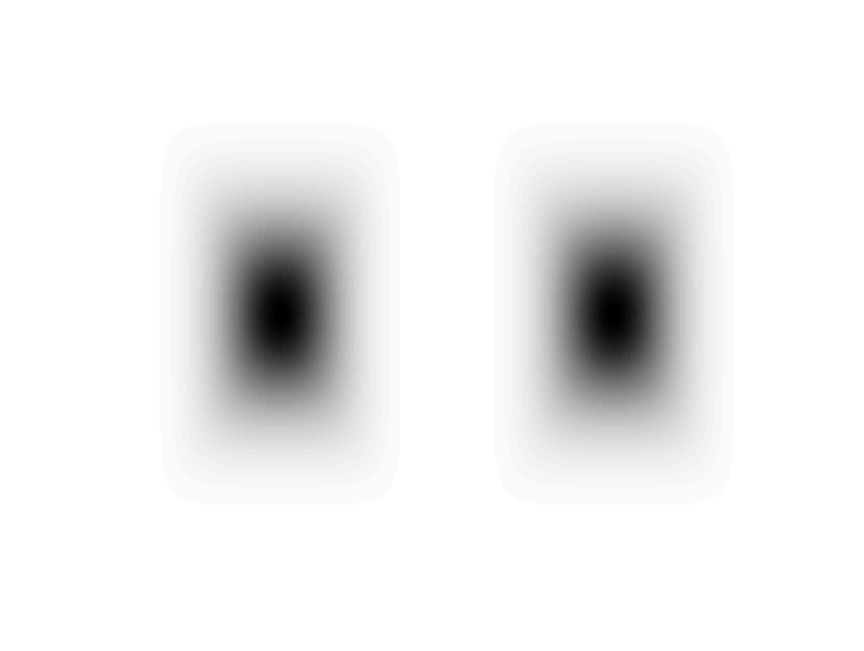}\\
\vspace{0.5cm}
\makebox[0.5cm]{\rotatebox{90}{\makebox[1.1cm]{---}}}
\makebox[0.5cm]{\rotatebox{90}{\makebox[1.1cm]{\texttt{LPQN}}}}
\makebox[0.5cm]{\rotatebox{90}{\makebox[1.1cm]{$p=0.1$}}}
\makebox[0.5cm]{\rotatebox{90}{\makebox[1.1cm]{$10^{-4}$}}}
\includegraphics[width=2.2cm,height=1.1cm]{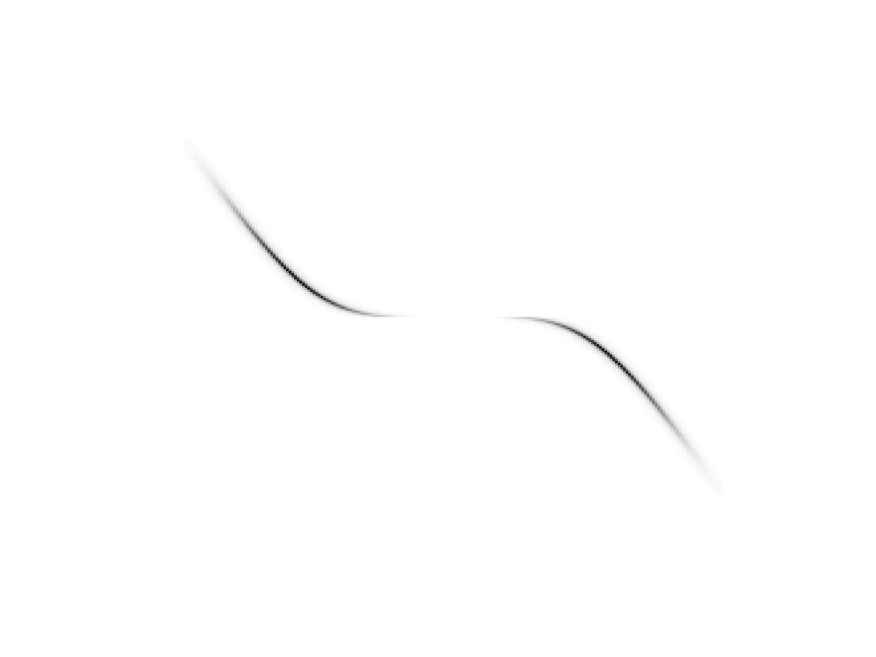}
\includegraphics[width=2.2cm,height=1.1cm]{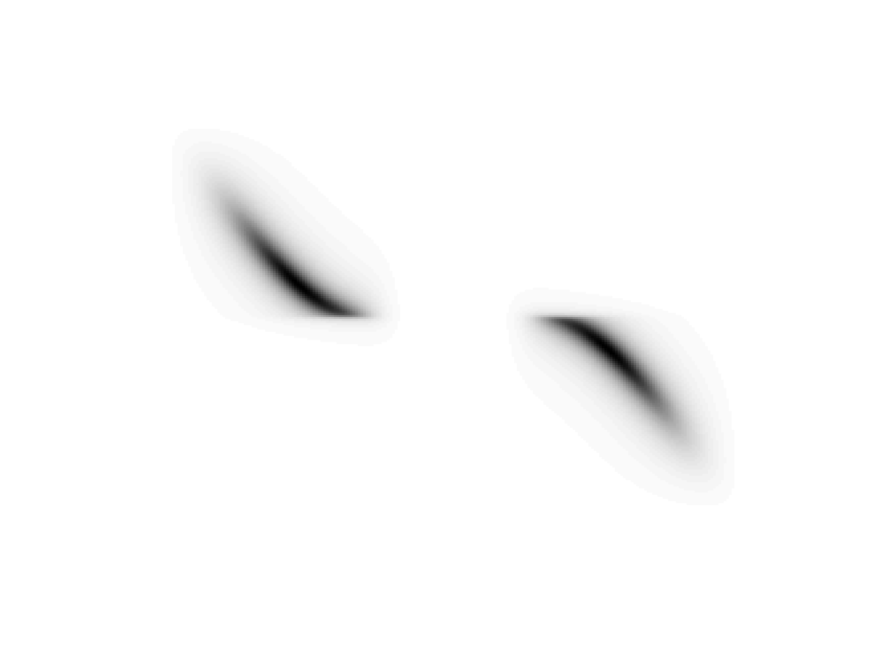}
\includegraphics[width=2.2cm,height=1.1cm]{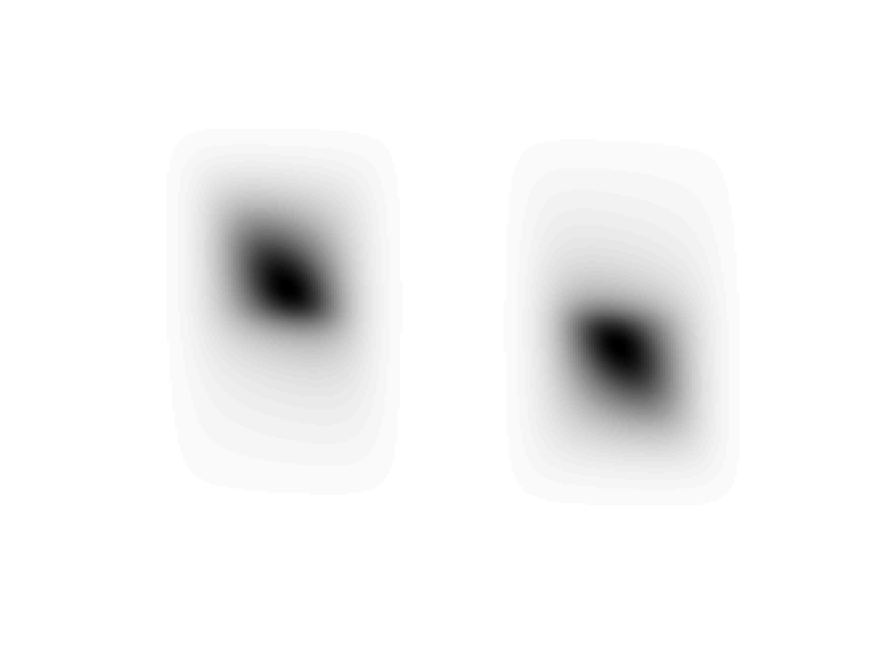}
\includegraphics[width=2.2cm,height=1.1cm]{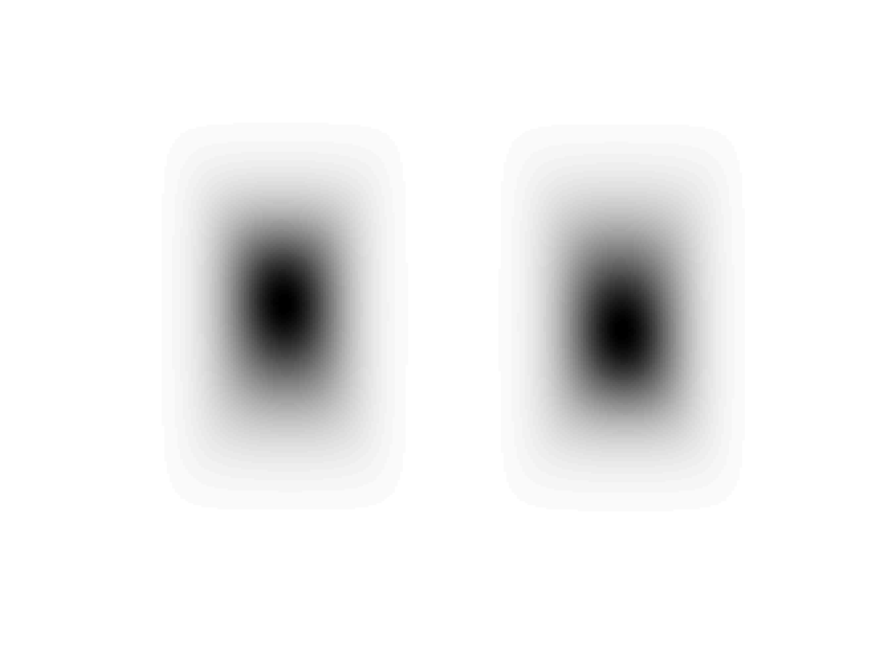}
\includegraphics[width=2.2cm,height=1.1cm]{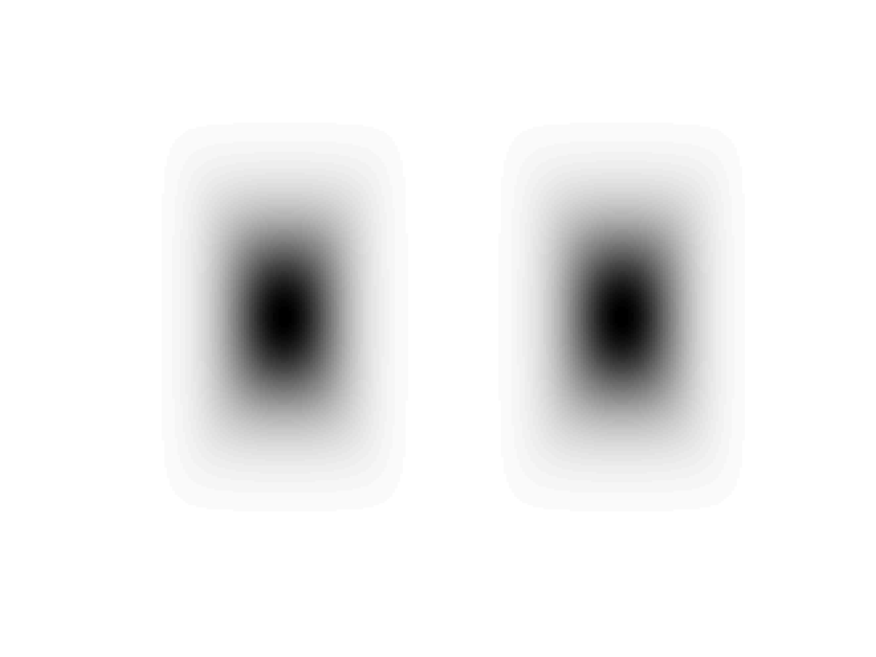}\\
\vspace{0.5cm}
\makebox[0.5cm]{\rotatebox{90}{\makebox[1.1cm]{---}}}
\makebox[0.5cm]{\rotatebox{90}{\makebox[1.1cm]{\texttt{LPQN}}}}
\makebox[0.5cm]{\rotatebox{90}{\makebox[1.1cm]{$p=0.5$}}}
\makebox[0.5cm]{\rotatebox{90}{\makebox[1.1cm]{$10^{-3}$}}}
\includegraphics[width=2.2cm,height=1.1cm]{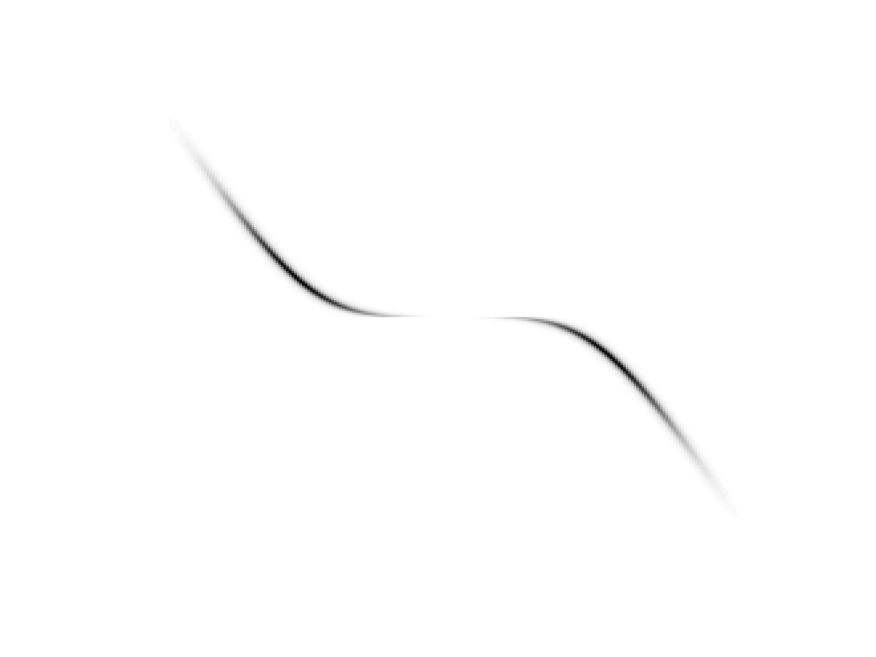}
\includegraphics[width=2.2cm,height=1.1cm]{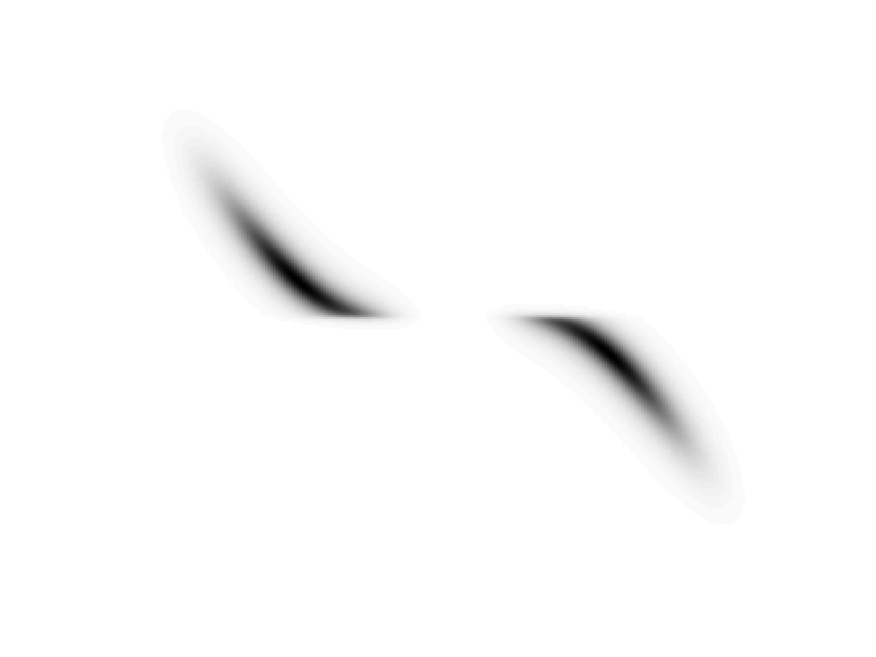}
\includegraphics[width=2.2cm,height=1.1cm]{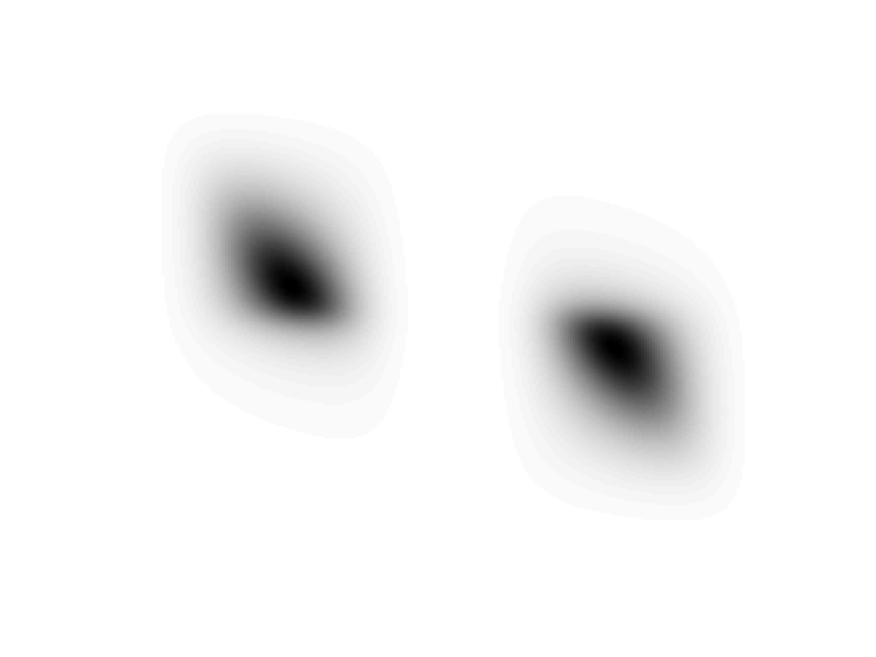}
\includegraphics[width=2.2cm,height=1.1cm]{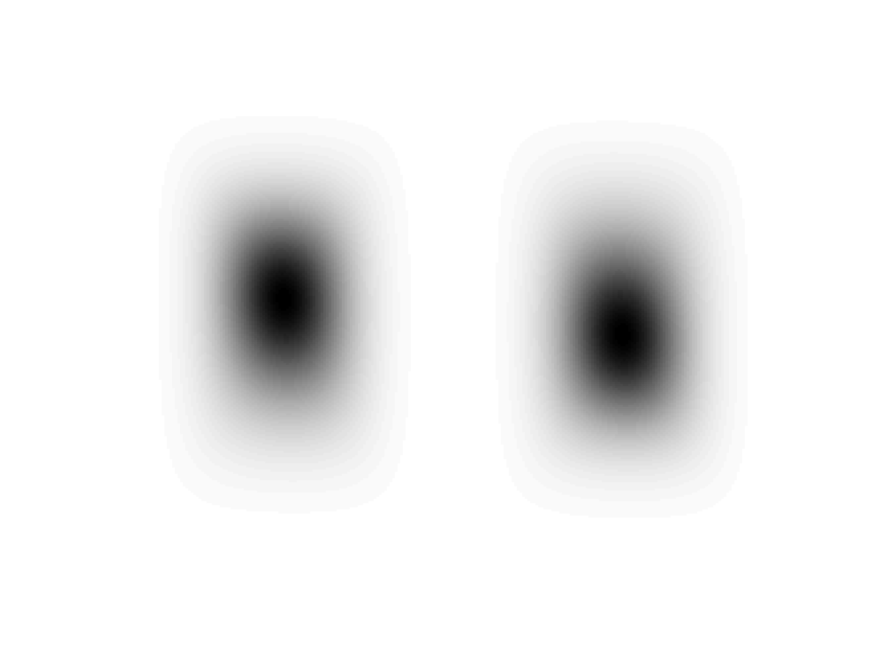}
\includegraphics[width=2.2cm,height=1.1cm]{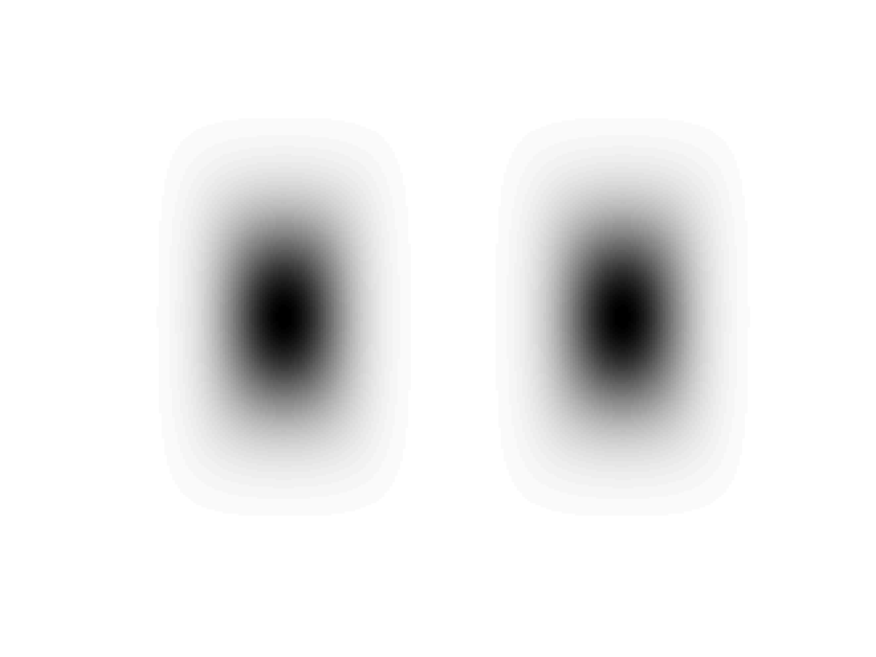}\\
\vspace{0.5cm}
\makebox[0.5cm]{\rotatebox{90}{\makebox[1.1cm]{---}}}
\makebox[0.5cm]{\rotatebox{90}{\makebox[1.1cm]{\texttt{LPQN}}}}
\makebox[0.5cm]{\rotatebox{90}{\makebox[1.1cm]{$p=0.9$}}}
\makebox[0.5cm]{\rotatebox{90}{\makebox[1.1cm]{$10^{-1}$}}}
\includegraphics[width=2.2cm,height=1.1cm]{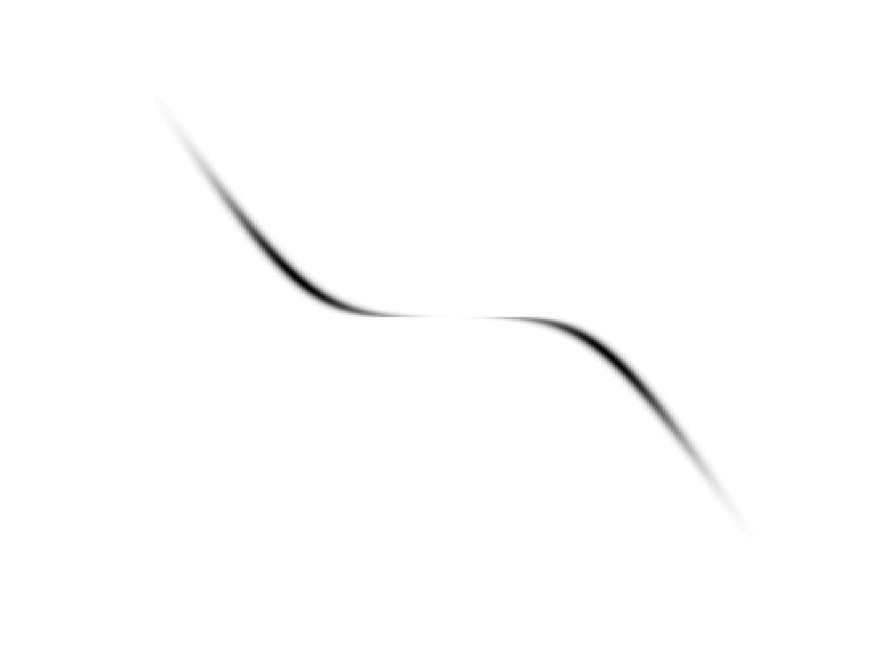}
\includegraphics[width=2.2cm,height=1.1cm]{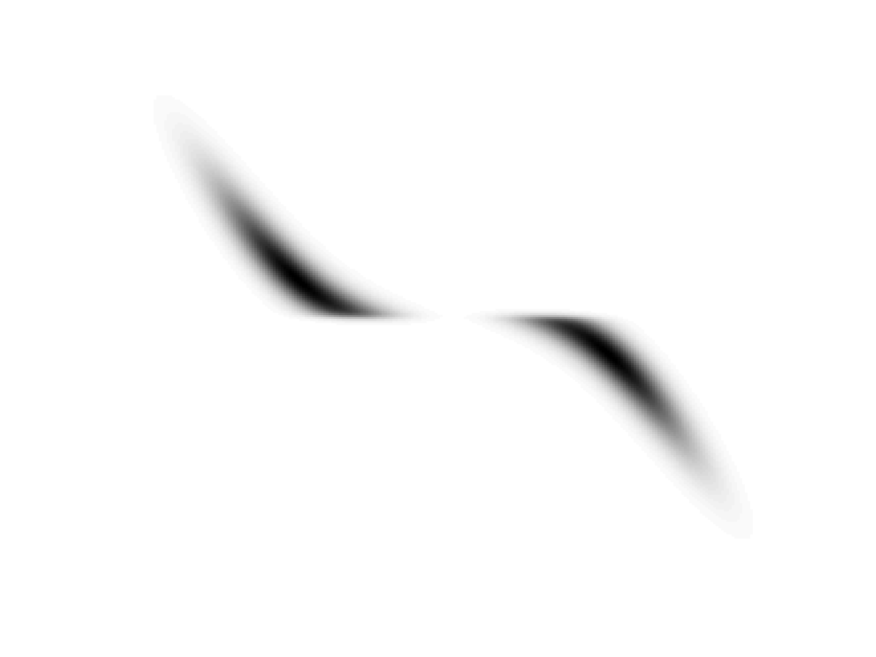}
\includegraphics[width=2.2cm,height=1.1cm]{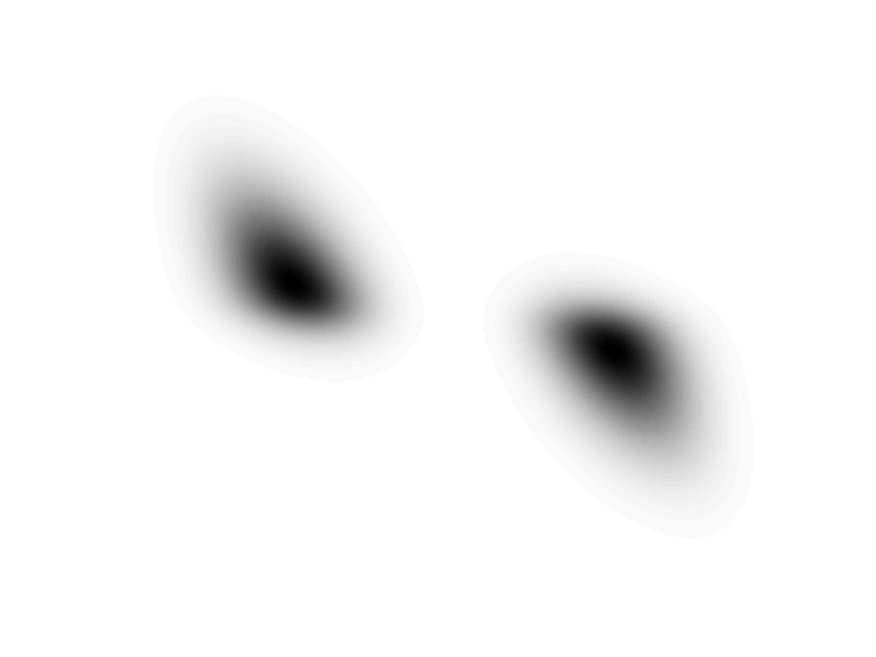}
\includegraphics[width=2.2cm,height=1.1cm]{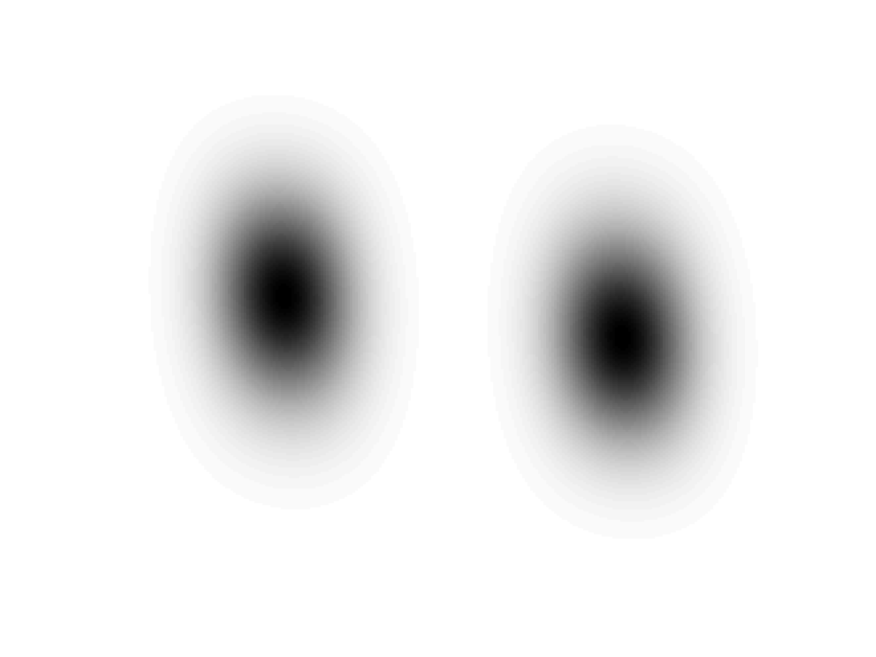}
\includegraphics[width=2.2cm,height=1.1cm]{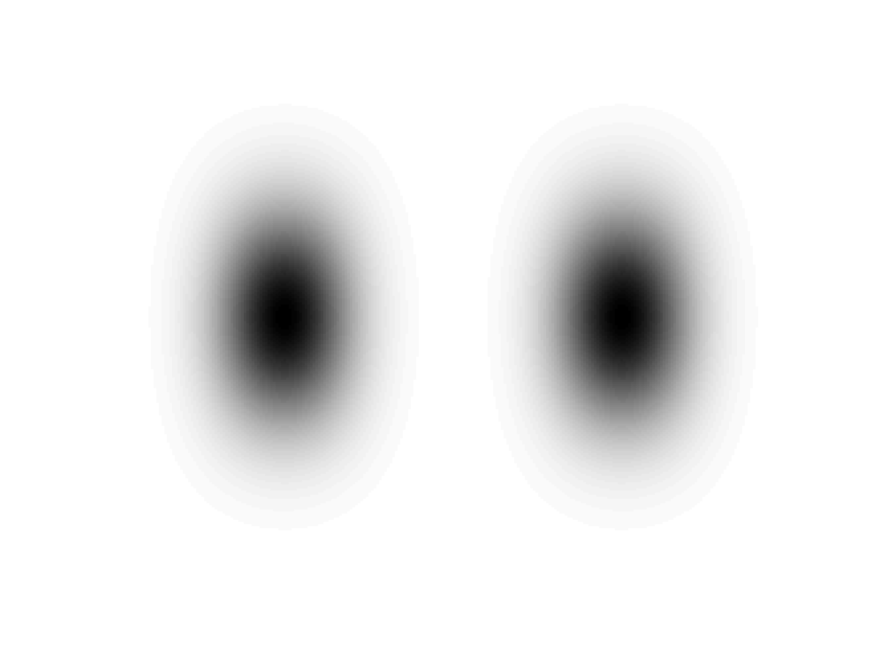}\\
\vspace{0.5cm}
\makebox[0.5cm]{\rotatebox{90}{\makebox[1.1cm]{---}}}
\makebox[0.5cm]{\rotatebox{90}{\makebox[1.1cm]{\texttt{LPN}}}}
\makebox[0.5cm]{\rotatebox{90}{\makebox[1.1cm]{$p=1.1$}}}
\makebox[0.5cm]{\rotatebox{90}{\makebox[1.1cm]{$10^{+0}$}}}
\includegraphics[width=2.2cm,height=1.1cm]{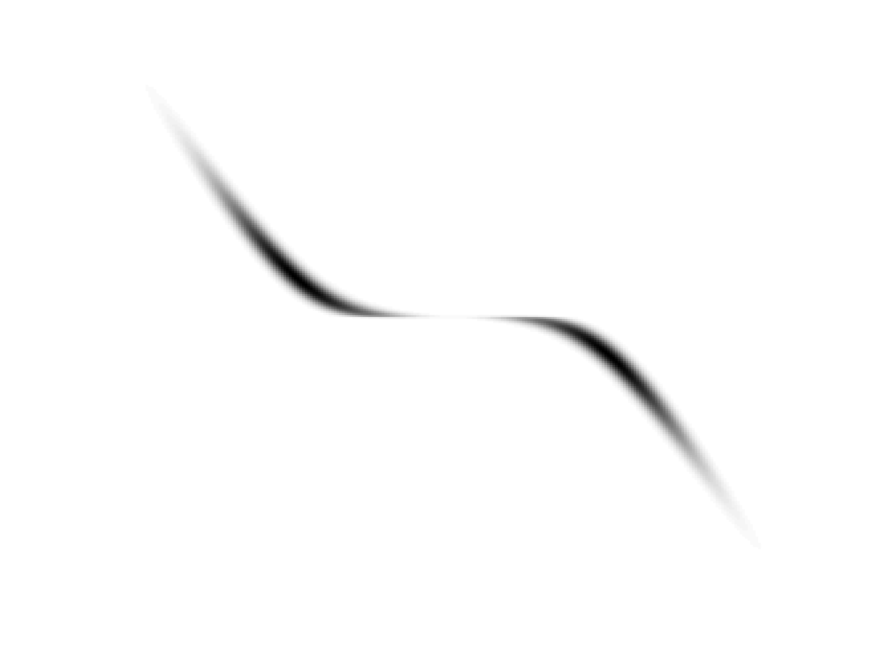}
\includegraphics[width=2.2cm,height=1.1cm]{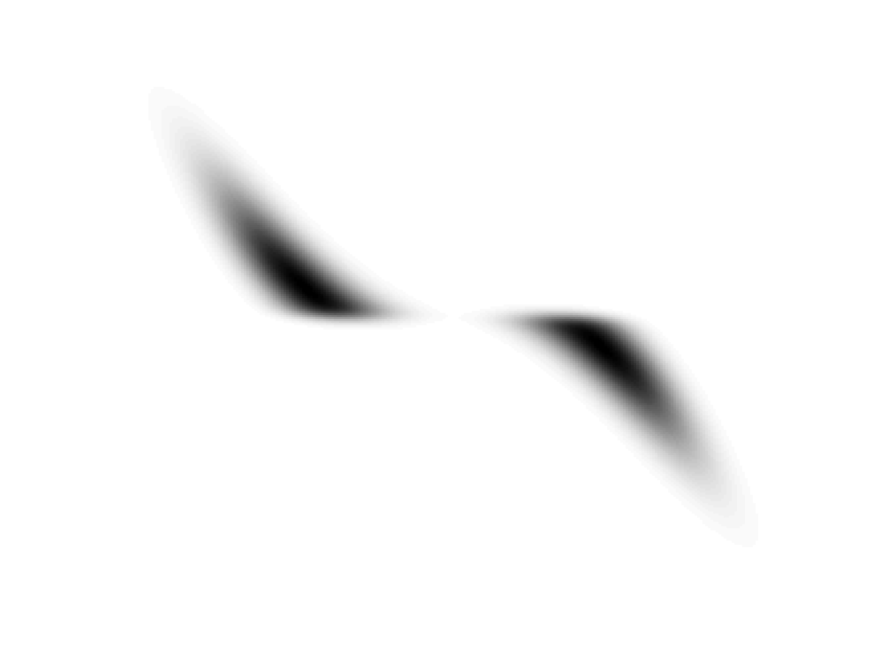}
\includegraphics[width=2.2cm,height=1.1cm]{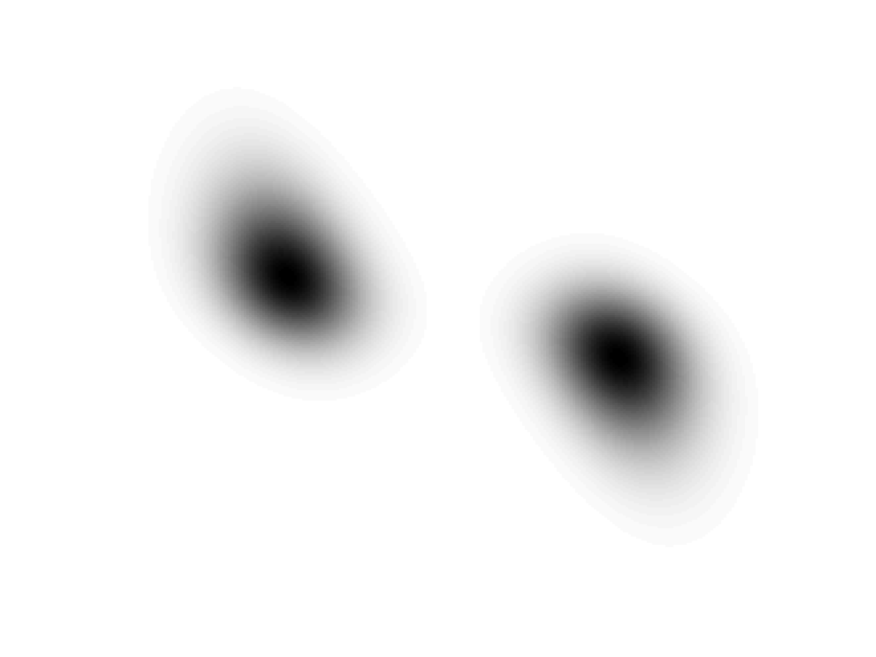}
\includegraphics[width=2.2cm,height=1.1cm]{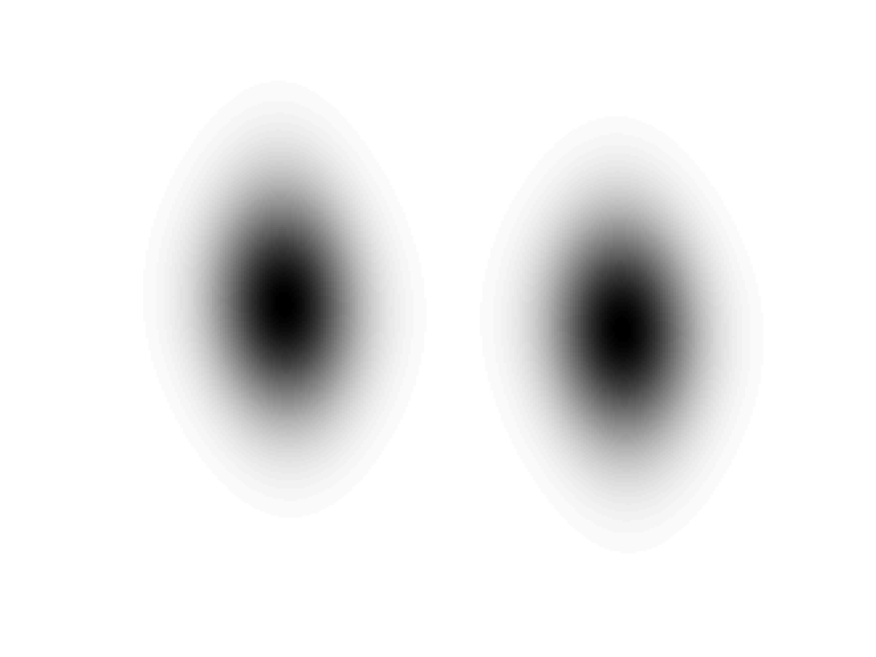}
\includegraphics[width=2.2cm,height=1.1cm]{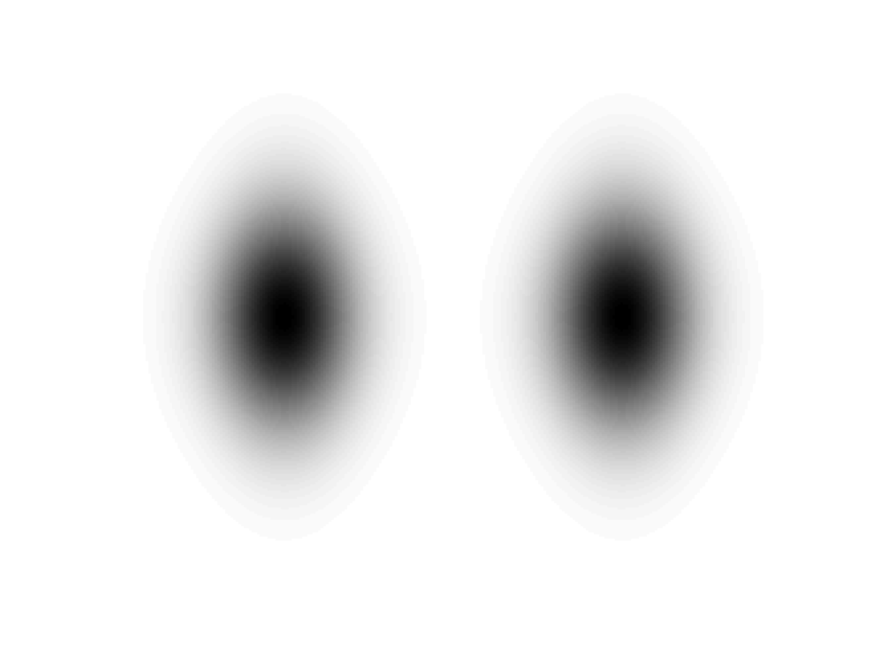}\\
\vspace{0.5cm}
\makebox[0.5cm]{\rotatebox{90}{\makebox[1.1cm]{---}}}
\makebox[0.5cm]{\rotatebox{90}{\makebox[1.1cm]{\texttt{LPN}}}}
\makebox[0.5cm]{\rotatebox{90}{\makebox[1.1cm]{$p=1.5$}}}
\makebox[0.5cm]{\rotatebox{90}{\makebox[1.1cm]{$10^{+1}$}}}
\includegraphics[width=2.2cm,height=1.1cm]{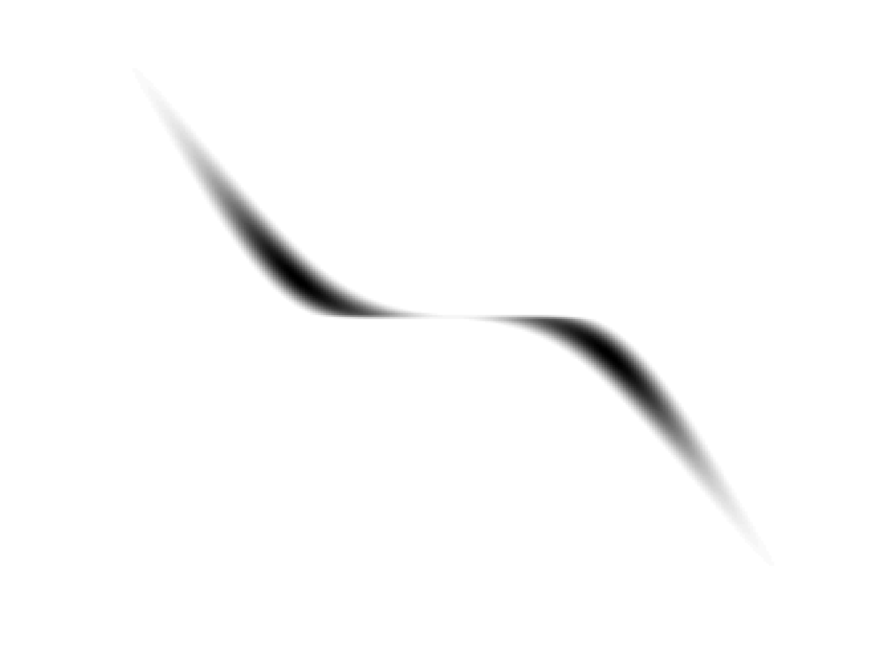}
\includegraphics[width=2.2cm,height=1.1cm]{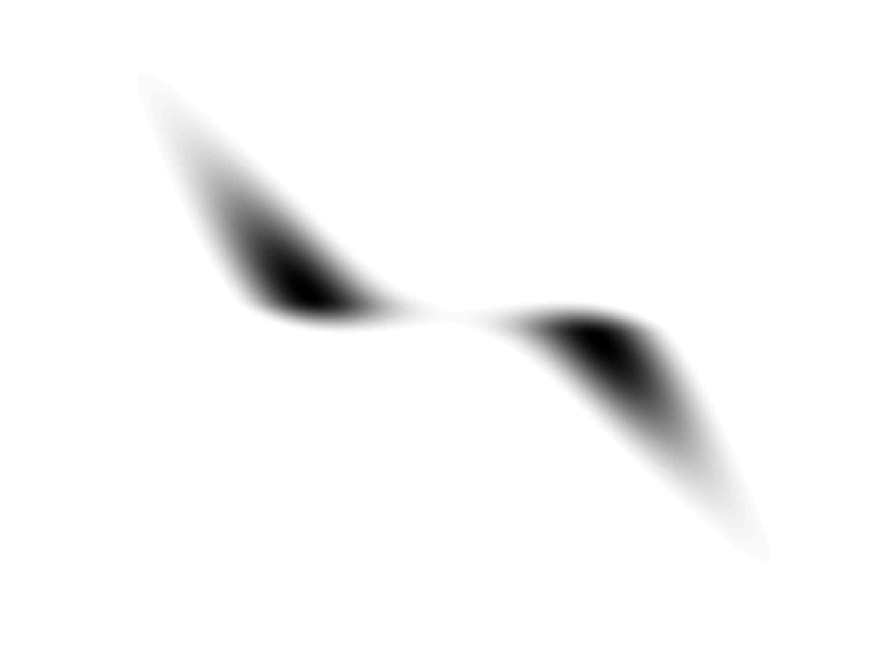}
\includegraphics[width=2.2cm,height=1.1cm]{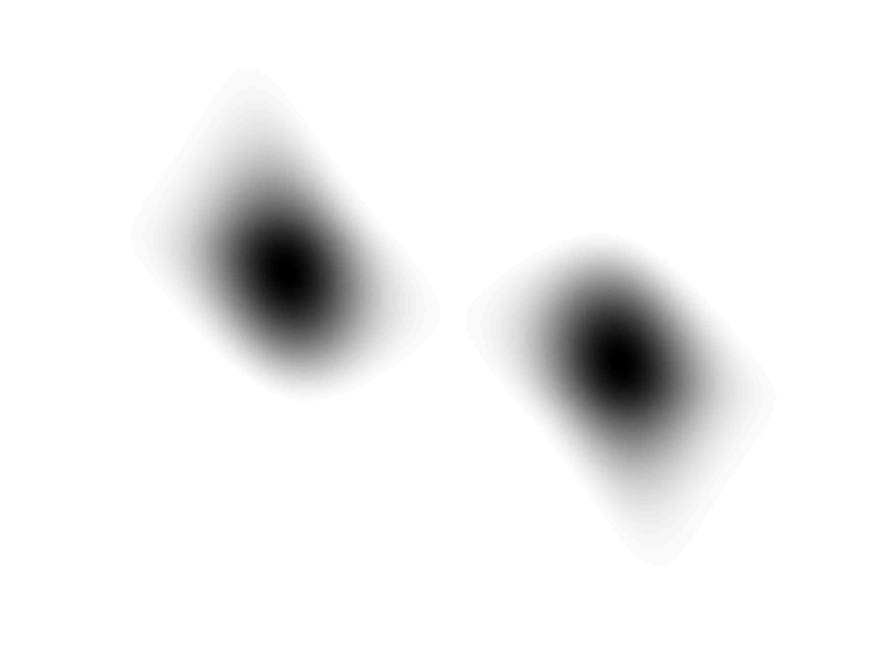}
\includegraphics[width=2.2cm,height=1.1cm]{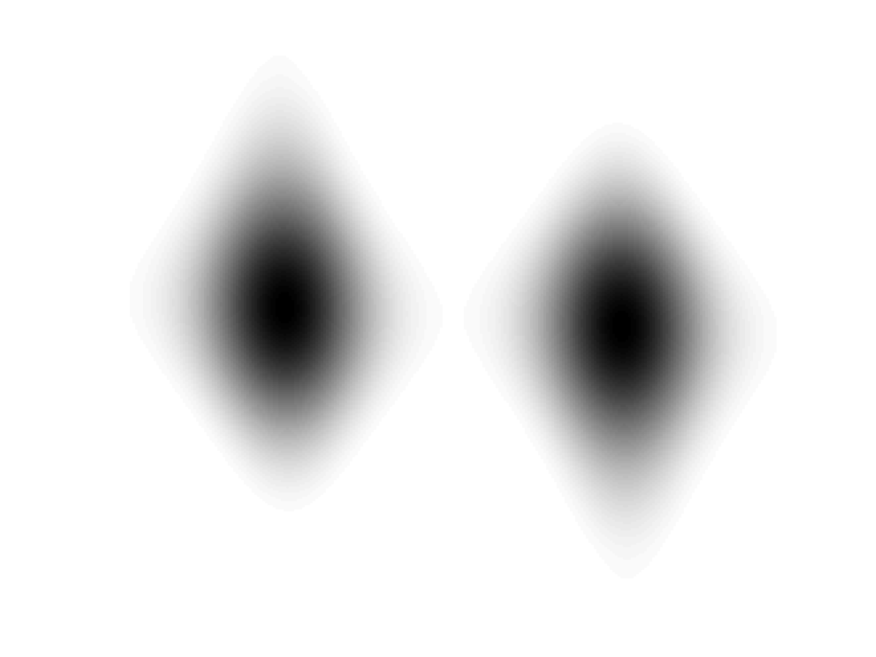}
\includegraphics[width=2.2cm,height=1.1cm]{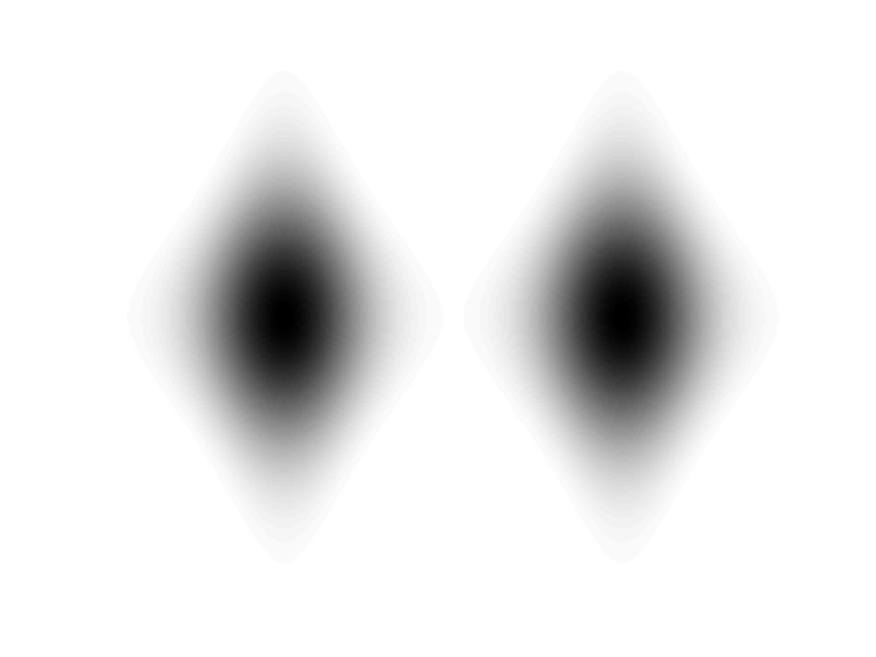}\\
\vspace{0.5cm}
\makebox[0.5cm]{\rotatebox{90}{\makebox[1.1cm]{\texttt{EUC}}}}
\makebox[0.5cm]{\rotatebox{90}{\makebox[1.1cm]{\texttt{LPN}}}}
\makebox[0.5cm]{\rotatebox{90}{\makebox[1.1cm]{$p=2$}}}
\makebox[0.5cm]{\rotatebox{90}{\makebox[1.1cm]{$10^{+2}$}}}
\includegraphics[width=2.2cm,height=1.1cm]{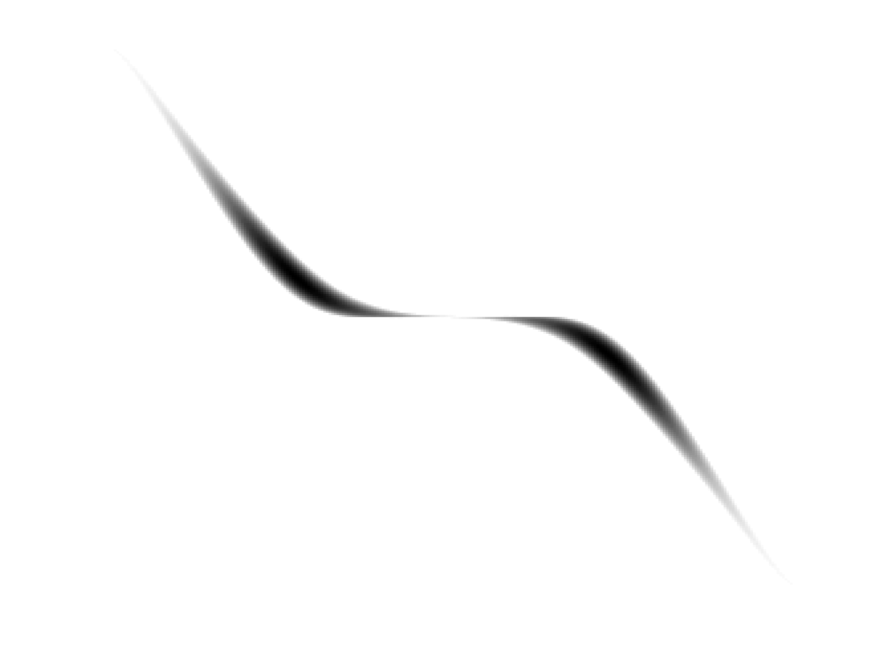}
\includegraphics[width=2.2cm,height=1.1cm]{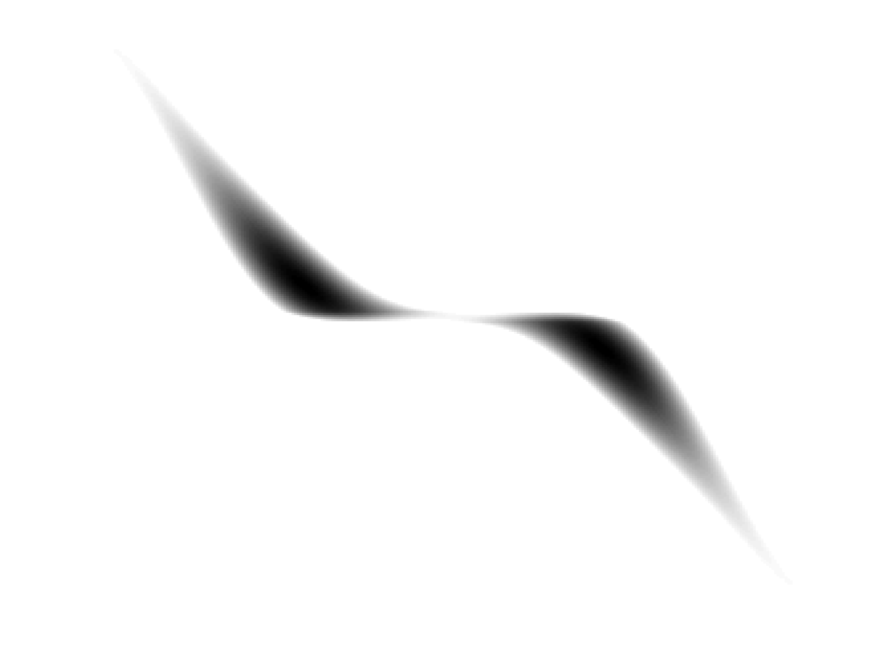}
\includegraphics[width=2.2cm,height=1.1cm]{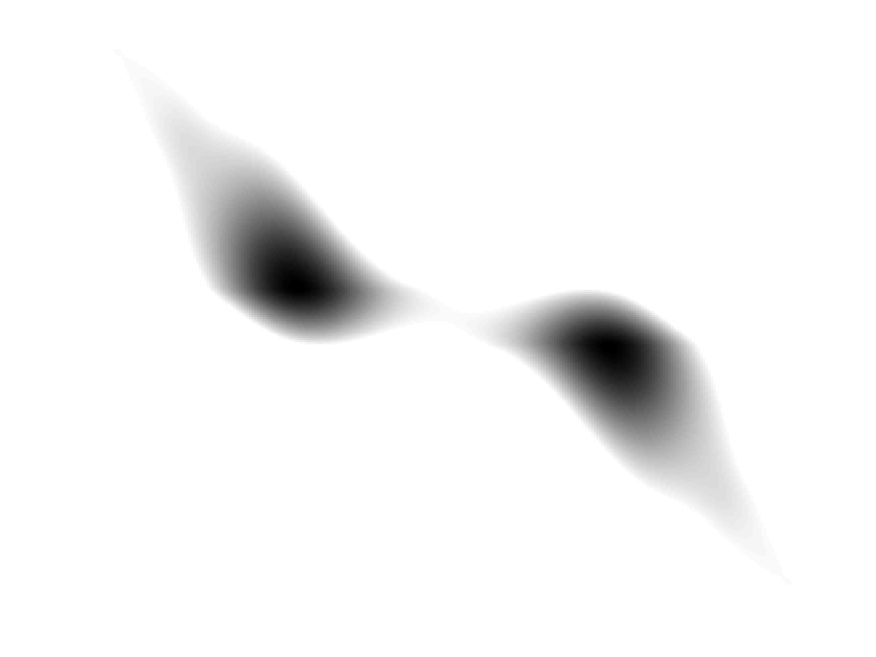}
\includegraphics[width=2.2cm,height=1.1cm]{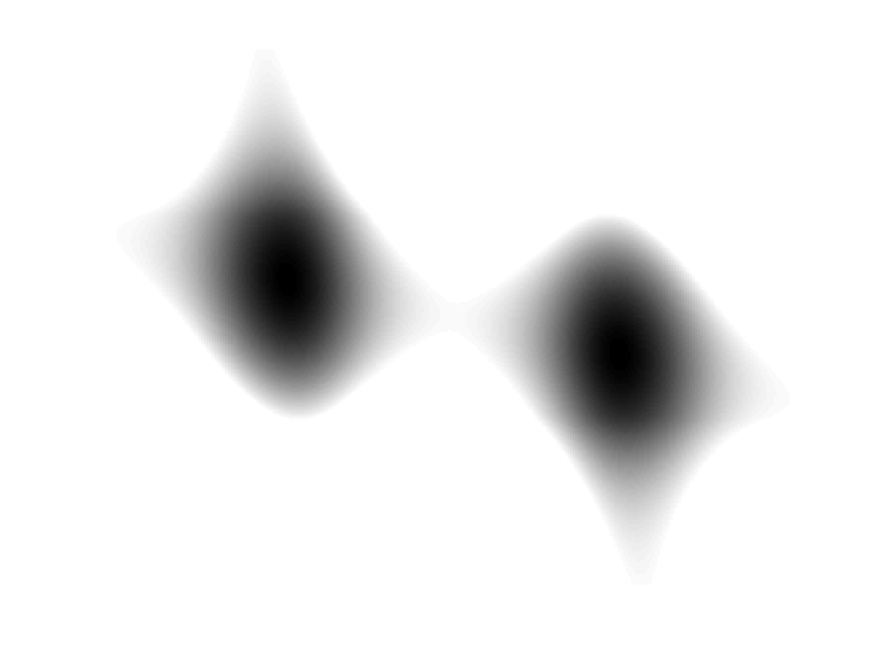}
\includegraphics[width=2.2cm,height=1.1cm]{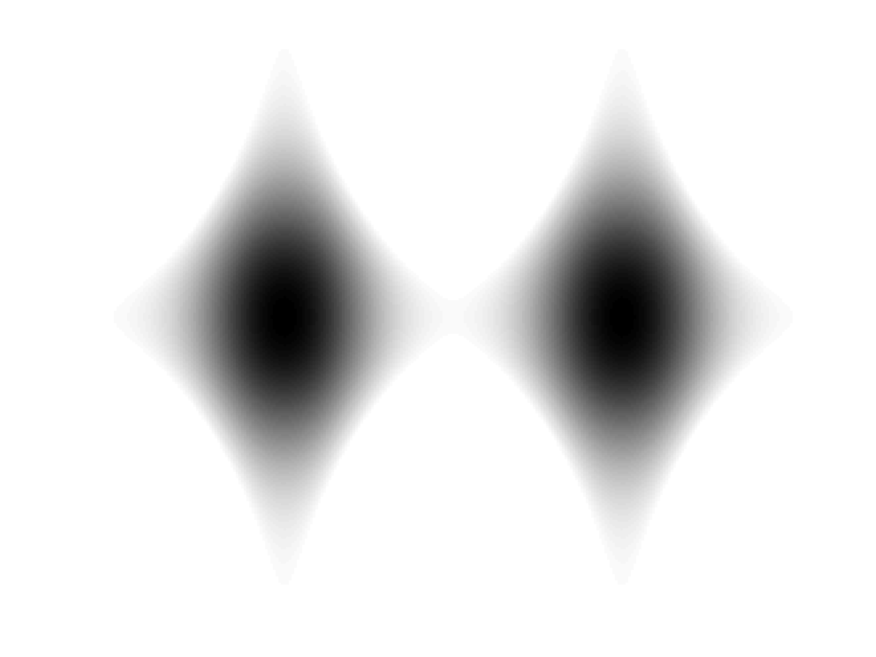}\\
\vspace{0.5cm}
\makebox[0.5cm]{\rotatebox{90}{\makebox[1.1cm]{---}}}
\makebox[0.5cm]{\rotatebox{90}{\makebox[1.1cm]{\texttt{HELL}}}}
\makebox[0.5cm]{\rotatebox{90}{\makebox[1.1cm]{---}}}
\makebox[0.5cm]{\rotatebox{90}{\makebox[1.1cm]{$10^{+2}$}}}
\includegraphics[width=2.2cm,height=1.1cm]{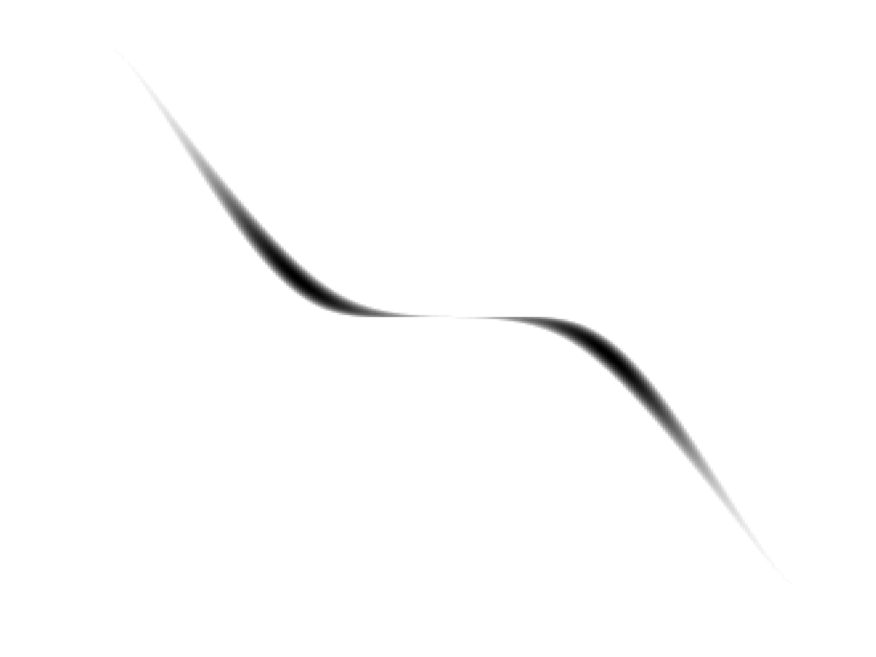}
\includegraphics[width=2.2cm,height=1.1cm]{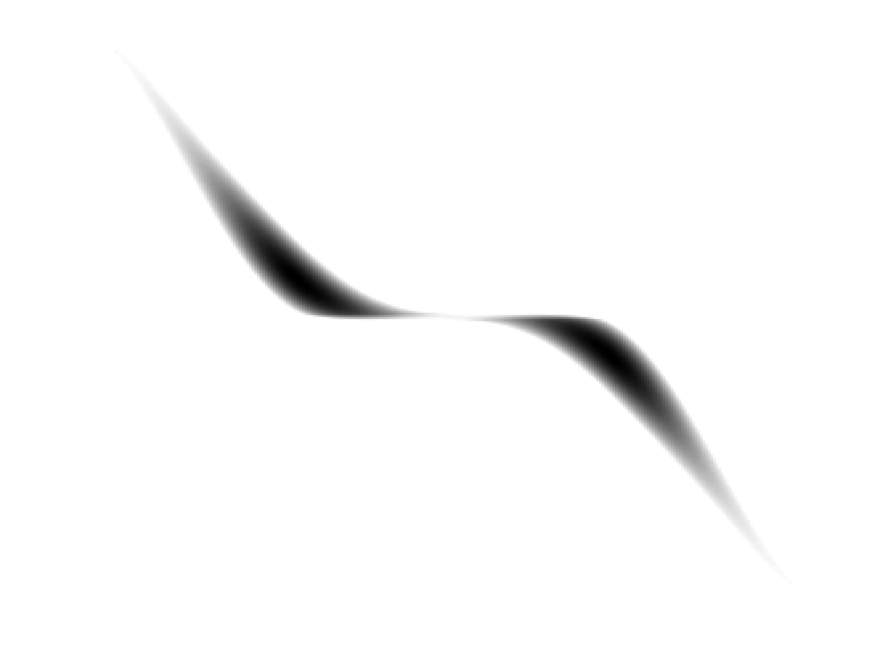}
\includegraphics[width=2.2cm,height=1.1cm]{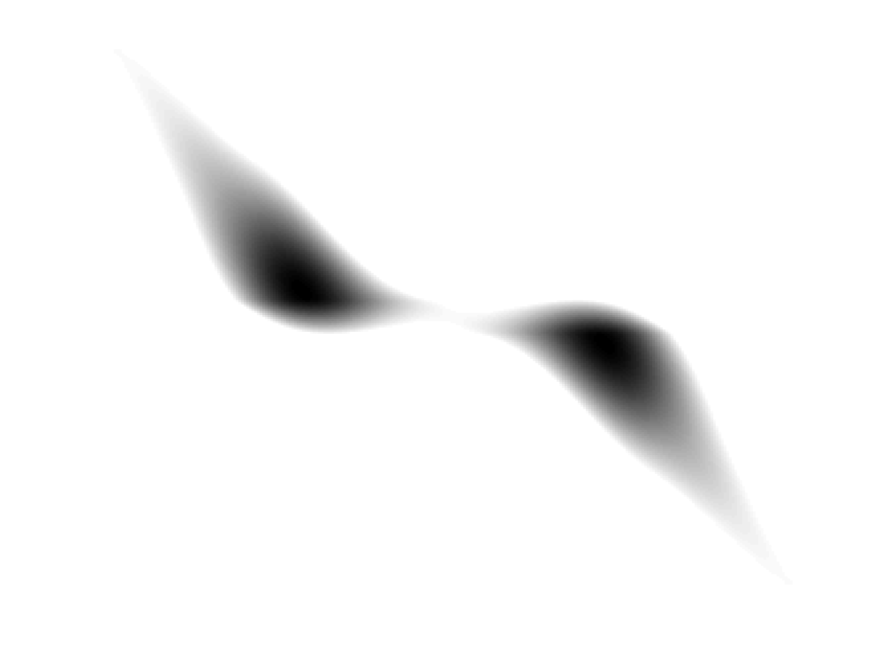}
\includegraphics[width=2.2cm,height=1.1cm]{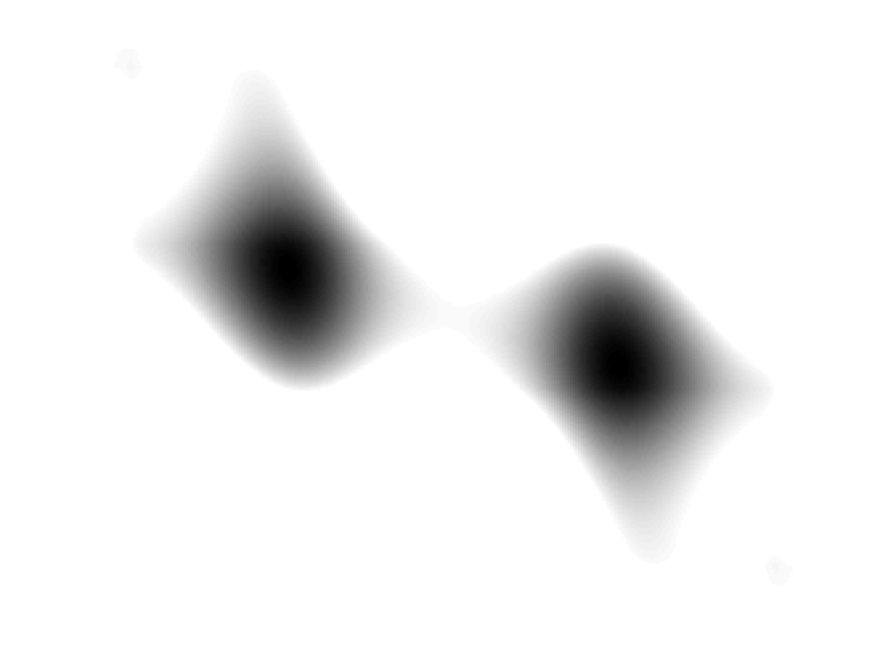}
\includegraphics[width=2.2cm,height=1.1cm]{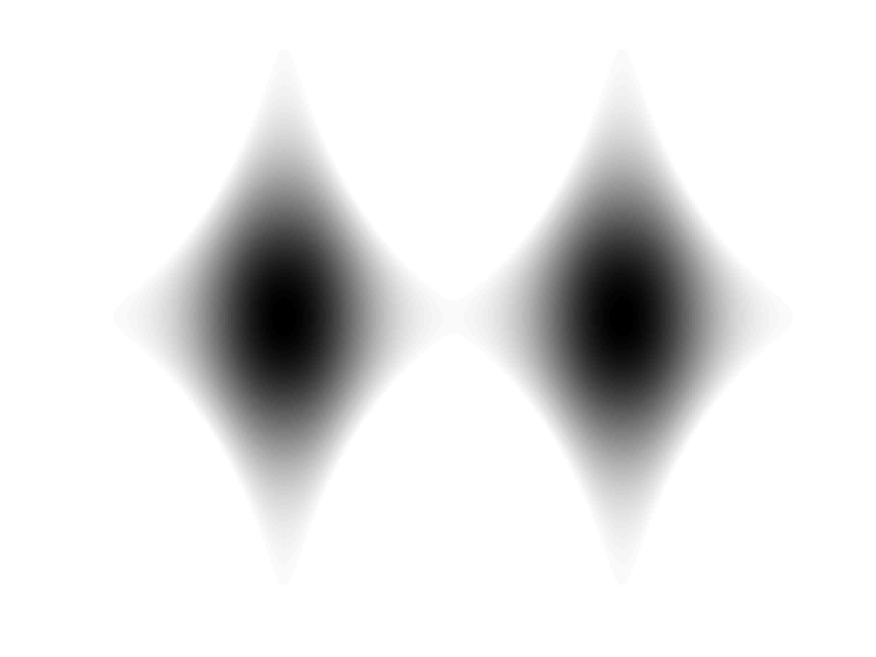}
\caption{Rot mover's plans $\boldsymbol\pi^\star$ for different regularizers $\phi$ and penalties $\lambda = \overline{\lambda} \, \lambda'$.}
\label{fig:gamma}
\end{figure}

We next report in Table~\ref{table:time} the computational times required to reach convergence for the different regularizers and penalties. As a stopping criterion, we use the relative variation with tolerance $10^{-2}$ in $\ell_2$ norm for the main loop of alternate Bregman projections, and the absolute variation with tolerance $10^{-5}$ in $\ell_2$ norm for the auxiliary loops of the Newton-Raphson method. We use the same synthetic data as above but also vary the dimension $d$ to assess its influence on speed. As already observed specifically for Sinkhorn distances~\cite{Cuturi2013}, computing ROT distances is faster for important regularization with larger values of $\lambda$. The regularizers under assumptions (A) do not require the extra projections onto the non-negative orthant, and thus intuitively require less computational effort than the ones that verify assumptions (B). In addition, we notice that when the projections have closed-form expressions, the algorithms are also faster. The results further illustrate the influence of the data dimension $d$ and the difference between ROT and classical OT performances. For a low dimension $d$, the RMD is competitive with EMD in his historical implementation {\oemd}~\cite{Rubner2000}. The super-cubic complexity of the EMD with {\oemd} becomes prohibitive as the data dimension increases in contrast to the RMD which scales better. It should nevertheless be underlined that for reasonable dimensions, fast computation of the EMD can be obtained with a more recent, optimized implementation of the network simplex solver {\nemd}~\cite{Bonneel2011}. For higher dimensions, the super-cubic complexity makes {\nemd} less attractive, though it stays competitive with the RMD under a dimension $d = 4096$.

\begin{table}[t!]
\centering
\begin{tabular}{@{\hspace{2pt}}l@{\hspace{8pt}}l@{\hspace{4pt}}l@{\hspace{4pt}}c@{\hspace{4pt}}c@{\hspace{4pt}}|@{\hspace{2pt}}c@{\hspace{2pt}}|@{\hspace{2pt}}c@{\hspace{2pt}}|@{\hspace{2pt}}c@{\hspace{2pt}}|@{\hspace{2pt}}c@{\hspace{2pt}}|@{\hspace{2pt}}c@{\hspace{2pt}}|@{\hspace{2pt}}c@{\hspace{2pt}}|@{\hspace{2pt}}c@{\hspace{2pt}}|@{\hspace{2pt}}c@{\hspace{2pt}}|@{\hspace{2pt}}c@{\hspace{2pt}}|}
\toprule
		&			&				&			& $d$						& \COLS{$128$}						& \COLS{$256$} 						& \COLS{$512$}\hspace{-2pt}\\
Algorithm	& \PHI 						& $\beta$/$p$	& $\overline{\lambda}$/$\lambda'$ 	& $10^{-2}$	& $10^{-1}$	& $10^{+0}$	& $10^{-2}$	& $10^{-1}$	& $10^{+0}$	& $10^{-2}$	& $10^{-1}$	& $10^{+0}$\\
\midrule
\rmd		& --- 			& \texttt{FDLOG} 	& --- 			& $10^{-2}$					& $0.366$		& $0.079$	 	& $0.044$ 	& $1.091$ 	& $0.311$ 	& $0.116$ 	& $1.865$ 	& $0.571$ 	& $0.273$\\
\hline
\rmd		& \texttt{BSKL} 	& \texttt{BETA}		& $1.00$		& $10^{-2}$					& $0.105$		& $0.013$ 	& $0.008$ 	& $0.259$ 	& $0.055$ 	& $0.017$ 	& $0.680$ 	& $0.100$ 	& $0.038$\\
\hline
\rmd		& --- 			& \texttt{BETA}		& $0.50$		& $10^{-4}$					& $0.971$ 	& $0.102$ 	& $0.044$ 	& $1.922$ 	& $0.251$ 	& $0.147$ 	& $3.526$ 	& $1.339$ 	& $0.281$\\
\hline
\rmd		& \texttt{BIS} 	& \texttt{BETA}		& $0.00$		& $10^{-6}$					& $0.916$		& $0.106$ 	& $0.019$ 	& $1.466$ 	& $0.108$ 	& $0.053$ 	& $2.598$ 	& $0.398$ 	& $0.096$\\
\hline
\rmd		& --- 			& \texttt{LPQN}		& $0.10$		& $10^{-4}$					& $0.968$ 	& $0.068$ 	& $0.055$ 	& $0.732$ 	& $0.173$ 	& $0.152$ 	& $0.416$ 	& $0.309$ 	& $0.305$\\
\hline
\rmd		& --- 			& \texttt{LPQN}		& $0.50$		& $10^{-3}$					& $0.404$ 	& $0.057$ 	& $0.042$ 	& $0.778$ 	& $0.163$ 	& $0.160$ 	& $0.780$ 	& $0.305$ 	& $0.304$\\
\hline
\rmd		& --- 			& \texttt{LPQN}		& $0.90$ 		& $10^{-1}$					& $0.226$		& $0.047$ 	& $0.040$ 	& $0.751$ 	& $0.178$ 	& $0.131$ 	& $1.110$ 	& $2.492$ 	& $0.214$\\
\hline
\rmd		& --- 			& \texttt{LPN} 		& $1.10$		& $10^{+0}$					& $1.570$ 	& $0.349$ 	& $0.148$ 	& $5.941$ 	& $1.557$ 	& $0.492$ 	& $6.357$ 	& $0.293$ 	& $0.926$\\
\hline
\rmd		& --- 			& \texttt{LPN}		& $1.50$		& $10^{+1}$					& $0.399$ 	& $0.099$ 	& $0.053$ 	& $1.170$ 	& $0.474$ 	& $0.166$ 	& $6.688$ 	& $2.163$ 	& $0.532$\\
\hline
\rmd		& \texttt{EUC} 	& \texttt{LPN}		& $2.00$		& $10^{+2}$					& $0.074$ 	& $0.043$ 	& $0.043$ 	& $0.253$ 	& $0.240$ 	& $0.237$ 	& $7.308$ 	& $3.190$ 	& $0.966$\\
\hline
\rmd		& --- 			& \texttt{HELL} 		& ---			& $10^{+2}$					& $0.197$ 	& $0.097$		& $0.087$		& $0.429$ 	& $0.316$ 	& $0.299$ 	& $5.570$ 	& $1.826$ 	& $0.823$\\
\hline
\oemd	& --- 			& --- 				& --- 			& --- 							& \COLS{$0.231$}						& \COLS{$1.912$}						& \COLS{$10.95$}\hspace{-2pt}\\
\hline
\nemd	& --- 			& --- 				& --- 			& --- 							& \COLS{$ 0.003$}						& \COLS{$ 0.011$}						& \COLS{$0.076$}\hspace{-2pt}\\
\bottomrule
\toprule
      		&			&				&			& $d$						& \COLS{$1024$} 						& \COLS{$2048$}						& \COLS{$4096$}\hspace{-2pt}\\
Algorithm	& \PHI 						& $\beta$/$p$	& $\overline{\lambda}$/$\lambda'$ 	& $10^{-2}$	& $10^{-1}$	& $10^{+0}$	& $10^{-2}$	& $10^{-1}$	& $10^{+0}$	& $10^{-2}$	& $10^{-1}$	& $10^{+0}$\\
\midrule
\rmd		& --- 			& \texttt{FDLOG} 	& --- 			& $10^{-2}$					& $4.156$		& $3.517$		& $1.410$		& $15.85$		& $9.663$		& $5.109$		& $54.19$		& $33.87$		& $17.24$\\
\hline
\rmd		&\texttt{BSKL} 	& \texttt{BETA}		& $1.00$		& $10^{-2}$					& $2.992$		& $0.705$		& $0.192$		& $13.42$		& $1.923$		& $0.630$		& $49.95$		& $7.074$		& $2.548$\\
\hline
\rmd		& --- 			& \texttt{BETA}		& $0.50$		& $10^{-4}$		        			& $ 8.015$	& $2.769$		& $0.888$		& $42.95$		& $7.538$		& $3.557$		& $101.3$		& $21.13$		& $10.86$\\
\hline 
\rmd		&\texttt{BIS} 	& \texttt{BETA}		& $0.00$		& $10^{-6}$					& $4.439$		& $0.777$		& $ 0.550$	& $6.590$		& $3.262$		& $2.218$		& $41.80$		& $12.96$		& $6.742$\\
\hline
\rmd		& --- 			& \texttt{LPQN}		& $0.10$		& $10^{-4}$					& $4.068$		& $2.174$		& $1.291$		& $6.962$		& $4.890$		& $4.334$		& $51.15$		& $15.86$		& $14.02$\\
\hline
\rmd		& --- 			& \texttt{LPQN}		& $0.50$		& $10^{-3}$					& $7.819$		& $4.198$		& $1.314$		& $26.34$		& $6.129$		& $4.301$		& $53.74$		& $15.65$		& $11.98$\\
\hline
\rmd		& --- 			& \texttt{LPQN}		& $0.90$ 		& $10^{-1}$					& $3.584$		& $2.264$		& $1.054$		& $13.80$		& $4.571$		& $3.285$		& $43.83$		& $14.22$		& $11.51$\\
\hline
\rmd		& --- 			& \texttt{LPN} 		& $1.10$		& $10^{+0}$					& $9.110$		& $4.924$		& $1.956$		& $38.98$		& $16.47$		& $8.400$		& $145.6$		& $65.95$		& $32.82$\\
\hline
\rmd		& --- 			& \texttt{LPN}		& $1.50$		& $10^{+1}$					& $18.97$		& $9.509$		& $2.539$		& $61.92$		& $20.41$		& $9.314$		& $236.6$		& $77.87$		& $45.94$\\
\hline
\rmd		&\texttt{EUC} 	& \texttt{LPN}		& $2.00$		& $10^{+2}$					& $11.90$		& $5.805$		& $2.161$		& $31.43$		& $14.22$		& $4.906$		& $117.1$		& $50.88$		& $27.67$\\
\hline
\rmd		& --- 			& \texttt{HELL} 		& ---			& $10^{+2}$					& $18.22$		& $6.629$		& $3.456$		& $35.45 $	& $20.03$		& $7.199$		& $205.0$		& $48.42$		& $31.75$\\
\hline
\oemd	& --- 			& --- 				& --- 			& --- 							& \COLS{$85.56$}						& \COLS{$482.2$}						& \COLS{$+\infty$}\hspace{-2pt}\\
\hline
\nemd	& --- 			& --- 				& --- 			& --- 							& \COLS{$0.482$}						& \COLS{$2.760$}						& \COLS{$13.23$}\hspace{-2pt}\\
\bottomrule
\end{tabular}
\caption{Computational times in seconds required to reach convergence for different regularizers $\phi$ and penalties $\lambda = \overline{\lambda} \, \lambda'$, with varying dimensions $d$.}
\label{table:time}
\end{table}

As a consequence, a numerical alternative to our algorithms for solving ROT problems with reasonable dimensions is to rely on conditional gradient methods similar to~\cite{Ferradans2014}. Indeed, such methods imply the iterative resolution of linearized ROT problems, that can be reformulated as EMD problems and therefore be solved with the fast network simplex approach~\cite{Bonneel2011}. Lastly, for a fair interpretation of the above timing results, we must mention that the two EMD schemes tested were run under MATLAB from native C/C++ implementations\footnote{\url{http://robotics.stanford.edu/~rubner/emd/default.htm}}\footnote{\url{http://liris.cnrs.fr/~nbonneel/FastTransport/}}  via compiled MEX files\footnote{\url{https://github.com/francopestilli/life/tree/master/external/emd}}\footnote{\url{https://arolet.github.io/code/}}. Hence, these EMD codes are quite optimized in comparison to our pure MATLAB prototype codes\footnote{\url{https://www.math.u-bordeaux.fr/~npapadak/GOTMI/codes.php}} for the RMD. It is thus plausible that optimized C/C++ implementations of our algorithms would be even more competitive in this context.

\subsection{Audio Classification}
\label{subsec:audio}

We now assess our methods in the context of audio classification, and specifically address the task of acoustic scene classification where the goal is to assign a test recording to one of predefined classes that characterizes the environment in which it was captured. We consider the framework of the DCASE 2016 IEEE AASP challenge with the TUT Acoustic Scenes 2016 database~\cite{Mesaros2016}. The data set consists of audio recordings at 44.1$\,$kHz sampling rate and 24-bit resolution. The metadata contains ground-truth annotations on the type of acoustic scene for all files, with a total of 15 classes: home, office, library, caf{\'{e}}/restaurant, grocery store, city center, residential area, park, forest path, beach, car, train, bus, tram, metro station. The audio material is cut into 30-second segments, and is split into two subsets of 75\%--25\% containing respectively 78--26 segments per class for development and evaluation, resulting in a total of 1170--390 files for training and testing. A 4-fold cross-validation setup is given with the training set. The classification accuracy, that is, the number of correctly classified segments among the total number of segments, is used as a score to evaluate systems.

A baseline system is also provided with the database for comparison. This system is based on Mel-frequency cepstral coefficient (MFCC) timbral features with Gaussian mixture model (GMM) classification. One GMM with diagonal covariance matrix is learned per class by expectation-maximization (EM), after concatenating and normalizing in mean and variance the extracted MFCCs from the training segments in that class. A test file is assigned to the class whose trained GMM leads to maximum likelihood for the extracted MFCCs for that file, where the MFCCs are considered as independent samples and normalized with the learned mean and variance for the respective classes. The baseline system is ran with its default parameters: 40$\,$ms frame size, 20$\,$ms hop size, 60-dimensional MFCCs comprising 20 static (including energy) plus 20 delta and 20 acceleration coefficients extracted with standard settings in RASTAMAT, 16 GMM components learned with standard settings in VOICEBOX.

Since MFCCs potentially take negative values, OT tools cannot be applied directly to this kind of features. Therefore, the common approach is to compute OT appropriately on GMMs estimated from MFCCs instead. Our proposed system follows this principle, and is implemented in the very same pipeline as the baseline for a fair comparison, with the following differences. One GMM is learned by EM for each training segment instead of class. Any normalization on the MFCCs per class is thus removed. Since less components are typically required to model one segment compared to one class, the spurious GMM components are further discarded as post-processing by keeping only those with weight and variances all greater than $10^{-2}$. Instead of applying a GMM classifier, all individual models are exploited to train a support vector machine (SVM) classifier. An exponential kernel for the SVM is designed by introducing a distance between two mixtures $P, Q$ based on the RMD as follows:
\begin{equation}
\kappa(P, Q) = \exp(-d_{\boldsymbol\gamma, \lambda, \phi}(\boldsymbol\omega, \boldsymbol\upsilon) / \tau) \enspace,
\end{equation}
where the exponential decay rate $\tau > 0$ is a kernel parameter, and $\boldsymbol\omega, \boldsymbol\upsilon \in \Sigma_d$ are the respective weights of the $d = 16$ (or less) components for the two GMMs $P, Q$. The cost matrix $\boldsymbol\gamma \in \R_+^{d \times d}$ depends on $P, Q$ and is the square root of a symmetrized Kullback-Leibler divergence, called the Jeffrey divergence, between the pairwise Gaussian components:
\begin{equation}
\gamma_{ij} = \sqrt{\frac{1}{4} \sum_{k = 1}^{l} \frac{{(\sigma_{ik}^2 - \varsigma_{jk}^2)}^2 + (\sigma_{ik}^2 + \varsigma_{jk}^2) {(\mu_{ik} - \nu_{jk})}^2}{\sigma_{ik}^2 \varsigma_{jk}^2}} \enspace,
\end{equation}
where $\boldsymbol\mu_i$ and $\boldsymbol\sigma_i^2$, respectively $\boldsymbol\nu_j$ and $\boldsymbol\varsigma_j^2$, are the means and variances of the $l = 60$ MFCC features for component $i$ in the first mixture $P$, respectively component $j$ in the second mixture $Q$. The SVM classifier is implemented with standard settings in LIBSVM, and requires an additional soft-margin parameter $C > 0$ to be tuned. Notice that, even if the kernel is not positive-definite, LIBSVM is still able to provide a relevant classification by guaranteeing convergence to a stationary point~\cite{Lin2003,Haasdonk2005,Alabdulmohsin2014}. All separable regularizers $\phi$ from Section~\ref{sec:examples} with different penalties $\lambda > 0$ are tested for the RMD in comparison to the EMD. The two distances between $\p, \q$ and $\q, \p$ with cost matrix $\boldsymbol\gamma$ transposed are computed and averaged, so as to remove any asymmetry due to practical issues. The number of iterations is limited to $100$ for the main loop of the algorithm and to $10$ for the auxiliary loops of the Newton-Raphson method, and the tolerance is set to $10^{-6}$ in all loops for convergence with the $\ell_\infty$ norm on the marginal difference checked after each iteration as a termination criterion. The parameters $\tau, \, C \in 10^{\{-1,+0,+1,+2\}}$ and penalty $\lambda \in \Lambda$, where $\Lambda$ is a manually chosen set of four successive powers of ten depending on the range of the regularizer $\phi$, are tuned automatically by cross-validation.

The obtained results on this experiment in terms of accuracy per system are reported in Table~\ref{tab:audio}. The optimal penalties $\lambda \in \Lambda$ selected by cross-validation for each regularizer $\phi$ are also included, while the optimal parameters $\tau, \, C$ are not displayed since they actually all equal $10^{+1}$ independently of the kernel used. We first notice that the proposed system {\texttt{SVM}} (support vector machine classifier) consistently outperforms the baseline system {\texttt{GMM}} (Gaussian mixture model classifier). This proves the benefits of incorporating individual information per sound via an SVM rather than exploiting global information per class with a GMM. This further demonstrates the relevance of OT and more general ROT problems for the design of kernels between GMMs in the SVM pipeline. We also notice that {\texttt{RMD}} (rot mover's distance kernel) is at least competitive with {\texttt{EMD}} (earth mover's distance kernel) for all proposed regularizers, except from {\texttt{EUC}} which does not perform as well. This might be a consequence of the regularization profile for {\texttt{EUC}}, or equivalently {\texttt{LPN}} with $p = 2$, which does not spread enough mass across similar bins, implying a lack of robustness to slight variations in the means and variances of the GMM components. Reducing the power parameter in {\texttt{LPN}} brings back to a competitive system with {\texttt{EMD}} for $p = 1.1$, and even a better trade-off with improved accuracy for $p = 1.5$. We obtain similar results for {\texttt{LPQN}} with $p = 0.9$ and $p = 0.5$, with now the best compromise for the lowest power value $p = 0.1$ which clearly outperforms {\texttt{EMD}}. As a remark, the accuracy for {\texttt{LPN}} and {\texttt{LPQN}} is not unimodal with respect to $p$ which controls the spread of mass in the regularization. We suspect this is because the performance is a function of both the spread of mass and the amount of regularization, whose coupling allows for similar compromises in terms of results within different regimes of use. Concerning {\texttt{BETA}} now, we observe that the existing Sinkhorn-Knopp algorithm {\texttt{BSKL}} for $\beta = 1$ does not improve the accuracy compared to {\texttt{EMD}}. Increasing the spread of mass with $\beta = 0$ in {\texttt{BIS}} is even worse. The best performance is obtained with a range in between for $\beta = 0.5$, which slightly improves results over {\texttt {EMD}}. Using {\texttt{LOG}} here slightly degrades the performance compared to {\texttt{EMD}} and {\texttt{BSKL}}. Interestingly, the overall best accuracy on this application is obtained for {\texttt{HELL}} which beats all other systems, including {\texttt{EUC}}, by a safe margin. In contrast to the experiment on synthetic data with dimension 256 presented in Section~\ref{subsec:synth}, where both {\texttt{BSKL}} and {\texttt{LOG}}, respectively {\texttt{EUC}} and {\texttt{HELL}}, behave similarly due to equivalence up to a constant for low values in the transport plans, the range of the transport plans here is much higher since the dimension of the input distributions is at most 16 (typically less than 10). This raises the importance of choosing a good regularizer depending on the actual task and its inherent design criteria such as the data dimension.

\begin{table}[t!]
\centering
\begin{tabular}{llllcccc}
\toprule
\CLASS \hspace{0.1cm}		& \PHI \hspace{0.1cm}			& $\beta$/$p$ 	& $\Lambda$ 			& $\lambda$	& Accuracy\\
\midrule
\texttt{GMM} 	& --- 			& --- 			& --- 				& ---			& ---					& --- 			& $77.2\%$\\
\SVM		& \texttt{EMD} 	& --- 			& --- 				& --- 			& ---					& ---			& $81.3\%$\\
			& \RMD		& --- 			& \texttt{FDLOG}	& --- 			& $10^{\{-2,-1,+0,+1\}}$	& $10^{-1}$	& $81.0\%$\\
			& 			& \texttt{BSKL} & \texttt{BETA}		& $1.00$		& $10^{\{-2,-1,+0,+1\}}$	& $10^{+0}$	& $81.3\%$\\
			& 			& --- 			& \texttt{BETA}		& $0.50$ 		& $10^{\{-3,-2,-1,+0\}}$	& $10^{-2}$	& $81.5\%$\\
			& 			& \texttt{BIS} 	& \texttt{BETA}		& $0.00$		& $10^{\{-4,-3,-2,-1\}}$	& $10^{-2}$	& $81.3\%$\\
			& 			& --- 			& \texttt{LPQN}		& $0.10$ 		& $10^{\{-2,-1,+0,+1\}}$	& $10^{-1}$	& $82.1\%$\\
			& 			& --- 			& \texttt{LPQN}		& $0.50$ 		& $10^{\{-2,-1,+0,+1\}}$	& $10^{-1}$	& $81.3\%$\\
			& 			& --- 			& \texttt{LPQN}		& $0.90$ 		& $10^{\{-2,-1,+0,+1\}}$	& $10^{+0}$	& $81.0\%$\\
			& 			& --- 			& \texttt{LPN}		& $1.10$ 		& $10^{\{-1,+0,+1,+2\}}$	& $10^{+0}$	& $81.0\%$\\
			& 			& --- 			& \texttt{LPN}		& $1.50$ 		& $10^{\{+0,+1,+2,+3\}}$	& $10^{+1}$	& $81.8\%$\\
			& 			& \texttt{EUC} 	& \texttt{LPN}		& $2.00$		& $10^{\{+1,+2,+3,+4\}}$	& $10^{+3}$	& $77.4\%$\\
			& 			& --- 			& \texttt{HELL} 		& --- 			& $10^{\{+1,+2,+3,+4\}}$	& $10^{+2}$	& $82.8\%$\\
\bottomrule
\end{tabular}
\caption{Results of the experiment on audio classification.}
\label{tab:audio}
\end{table}

\section{Conclusion}
\label{sec:conclusion}

In this paper, we formulated a unified framework for smooth convex regularization of discrete OT problems. We also derived some algorithmic methods to solve such ROT problems, and detailed their specificities for classical regularizers and associated divergences from the literature. We finally designed a synthetic experiment to illustrate our proposed methods, and proved the relevance of ROT problems and the RMD on a real-world application to audio scene classification. The obtained results are encouraging for further development of the present work, and we now discuss some interesting perspectives for future investigation.

Firstly, we want to assess the effect of other regularizers on the solutions, notably when adding an affine term. From a geometrical viewpoint, such a transformation is equivalent to simply translating the cost matrix, with no effect on the Bregman divergence itself. For a given regularizer, we could therefore parametrize a whole family of interpolating regularizers, and tune the translation parameter according to the application. In particular, a recent work developed independently of ours makes use of Tsallis entropies to regularize OT problems with ad hoc solvers~\cite{Muzellec2016}. These regularizers could be integrated readily to our more general framework based on alternate Bregman projections, since Tsallis entropies are equivalent to $\beta$\nobreakdash-potentials and $\ell_p$ (quasi)\nobreakdash-norms up to an affine term.

In another direction, we would like to extend some theoretical results that hold for the Boltzmann-Shannon entropy and associated Kullback-Leibler divergence. Specifically, it is known that the related rot mover's plan converges in norm to the earth mover's plan with an exponential rate as the penalty decreases~\cite{Cominetti1994}. It is not straightforward, however, to generalize this to other regularizers and divergences. In addition, it would be worth elucidating some technical restrictions under which metric properties such as the triangular inequality can be proved similarly to Sinkhorn distances.

We also plan to study other pattern recognition tasks in text, image and audio signal processing. Intuitive possibilities include retrieval and classification for various kinds of data modeled via histograms of features or GMMs. Among potential approaches, this can be addressed by exploiting the RMD either directly in a nearest-neighbor search, or in the design of kernels for an SVM as done here for acoustic scenes. For such tasks, it would be relevant to provide insight into the choice of a good regularizer for the actual problem, or develop methods for automatic tuning of regularization parameters, and for learning the cost matrix from the data as can be done for the EMD~\cite{Cuturi2014}. Even if we mostly focused on separable regularizers, it would be relevant to further use the quadratic forms associated to Mahalanobis distances in certain applications, and maybe propose a parametric learning scheme for the quadratic regularizer from the data.

Lastly, a more prospective idea is to use the RMD instead of Sinkhorn distances in the recent works built on the entropic regularization mentioned in Section~\ref{sec:intro}. We also think that variational ROT problems could be formulated for statistical inference, notably parameter estimation in finite mixture models by minimizing loss functions based on the RMD~\cite{Dessein2017}. This would leverage new applications of our ROT framework for more general machine learning problems. Such developments are yet involved and require some theoretical effort before reaching enough maturity to address practical setups.

\paragraph{Acknowledgments}
This study has been carried out with financial support from the French State, managed by the French National Research Agency (ANR) in the frame of the ``Investments for the future'' Program IdEx Bordeaux (ANR-10-IDEX-03-02), Cluster of excellence CPU and the GOTMI project (ANR-16-CE33-0010-01). The authors would like to thank Annamaria Mesaros for her kind help with the evaluation on the DCASE 2016 IEEE AASP challenge, Charles-Alban Deledalle for his valuable advice on the use of the computing platform PlaFRIM, and Marco Cuturi for the insightful discussions about this work.

\bibliographystyle{abbrv}

\begin{thebibliography}{10}

\bibitem{Ahuja1993}
R.~K. Ahuja, T.~L. Magnanti, and J.~B. Orlin.
\newblock {\em Network Flows: Theory, Algorithms and Applications}.
\newblock Prentice-Hall, Inc., Upper Saddle River, NJ, USA, 1993.

\bibitem{Alabdulmohsin2014}
I.~Alabdulmohsin, X.~Gao, and X.~Zhang.
\newblock Support vector machines with indefinite kernels.
\newblock In {\em Asian Conference on Machine Learning (ACML)}, pages 32--47,
  2014.

\bibitem{Amari2000}
S.-i. Amari and H.~Nagaoka.
\newblock {\em Methods of Information Geometry}, volume 191 of {\em
  Translations of Mathematical Monographs}.
\newblock American Mathematical Society, Providence, RI, USA, 2000.

\bibitem{Arjovsky2017}
M.~Arjovsky, S.~Chintala, and L.~Bottou.
\newblock Wasserstein {GAN}.
\newblock Technical report, arXiv:1701.07875, 2017.

\bibitem{Bauschke2000}
H.~H. Bauschke and A.~S. Lewis.
\newblock Dykstra's algorithm with {Bregman} projections: A convergence proof.
\newblock {\em Optimization}, 48(4):409--427, 2000.

\bibitem{Benamou2015}
J.-D. Benamou, G.~Carlier, M.~Cuturi, L.~Nenna, and G.~Peyr{\'{e}}.
\newblock Iterative {Bregman} projections for regularized transportation
  problems.
\newblock {\em SIAM Journal on Scientific Computing}, 37(2):A1111--A1138, 2015.

\bibitem{Bernton2017}
E.~Bernton, P.~E. Jacob, M.~Gerber, and C.~P. Robert.
\newblock Inference in generative models using the {Wasserstein} distance.
\newblock Technical report, arXiv:1701.05146, 2017.

\bibitem{Bigot2013}
J.~Bigot, R.~{Gouet}, T.~Klein, and A.~L{\'{o}}pez.
\newblock Geodesic {PCA} in the {Wasserstein} space.
\newblock Technical report, arXiv:1307.7721, 2013.

\bibitem{Blondel2017}
M.~Blondel, V.~Seguy, and A.~Rolet.
\newblock Smooth and sparse optimal transport.
\newblock Technical report, arXiv:1710.06276, 2017.

\bibitem{Bonneel2011}
N.~Bonneel, M.~van~de Panne, S.~Paris, and W.~Heidrich.
\newblock Displacement interpolation using lagrangian mass transport.
\newblock {\em ACM Transactions on Graphics}, 30(6):158:1--158:12, 2011.

\bibitem{Bousquet2017}
O.~Bousquet, S.~Gelly, I.~Tolstikhin, C.-J. Simon-Gabriel, and
  B.~Sch{\"{o}}lkopf.
\newblock From optimal transport to generative modeling: the {VEGAN} cookbook.
\newblock Technical report, arXiv:1705.07642, 2017.

\bibitem{Cazelles2017}
E.~Cazelles, V.~Seguy, J.~Bigot, M.~Cuturi, and N.~Papadakis.
\newblock {Log-PCA} versus geodesic {PCA} of histograms in the {Wasserstein}
  space.
\newblock Technical report, arXiv:1708.08143, 2017.

\bibitem{Cominetti1994}
R.~Cominetti and J.~San~Mart{\'{i}}n.
\newblock Asymptotic analysis of the exponential penalty trajectory in linear
  programming.
\newblock {\em Mathematical Programming}, 67(1--3):169--187, 1994.

\bibitem{Courty2015}
N.~Courty, R.~Flamary, D.~Tuia, and A.~Rakotomamonjy.
\newblock Optimal transport for domain adaptation.
\newblock {\em IEEE Transactions on Pattern Analysis and Machine Intelligence},
  39(9):1853--1865, 2015.

\bibitem{Cuturi2013}
M.~Cuturi.
\newblock Sinkhorn distances: Lightspeed computation of optimal transport.
\newblock In {\em International Conference on Neural Information Processing
  Systems (NIPS)}, pages 2292--2300, 2013.

\bibitem{Cuturi2014}
M.~Cuturi and D.~Avis.
\newblock Ground metric learning.
\newblock {\em Journal of Machine Learning Research}, 15(1):533--564, 2014.

\bibitem{Cuturi2016}
M.~Cuturi and G.~Peyr{\'{e}}.
\newblock A smoothed dual approach for variational {Wasserstein} problems.
\newblock {\em SIAM Journal on Imaging Sciences}, 9(1):320--343, 2016.

\bibitem{Dessein2017}
A.~Dessein, N.~Papadakis, and C.-A. Deledalle.
\newblock Parameter estimation in finite mixture models by regularized optimal
  transport: A unified framework for hard and soft clustering.
\newblock Technical report, arXiv:1711.04366, 2017.

\bibitem{Dhillon2007}
I.~S. Dhillon and J.~A. Tropp.
\newblock Matrix nearness problems with {Bregman} divergences.
\newblock {\em SIAM Journal on Matrix Analysis and Applications},
  29(4):1120--1146, 2007.

\bibitem{Ferradans2014}
S.~Ferradans, N.~Papadakis, G.~Peyr{\'{e}}, and J.-F. Aujol.
\newblock Regularized discrete optimal transport.
\newblock {\em SIAM Journal on Imaging Sciences}, 7(3):1853--1882, 2014.

\bibitem{Frogner2015}
C.~Frogner, C.~Zhang, H.~Mobahi, M.~Araya-Polo, and T.~Poggio.
\newblock Learning with a {Wasserstein} loss.
\newblock In {\em International Conference on Neural Information Processing
  Systems (NIPS)}, pages 2053--2061, 2015.

\bibitem{Galichon2015}
A.~Galichon and B.~Salani{\'{e}}.
\newblock Cupid's invisible hand: Social surplus and identification in matching
  models.
\newblock Technical report, SSRN:1804623, 2015.

\bibitem{Genevay2017}
A.~Genevay, G.~Peyr{\'{e}}, and M.~Cuturi.
\newblock {GAN} and {VAE} from an optimal transport point of view.
\newblock Technical report, arXiv:1706.01807, 2017.

\bibitem{Grauman2004}
K.~Grauman and T.~Darrell.
\newblock Fast contour matching using approximate earth mover's distance.
\newblock In {\em IEEE Computer Vision and Pattern Recognition (CVPR)}, pages
  220--227, 2004.

\bibitem{Gudmundsson2007}
J.~Gudmundsson, O.~Klein, C.~Knauer, and M.~Smid.
\newblock Small {Manhattan} networks and algorithmic applications for the earth
  mover's distance.
\newblock In {\em European Workshop on Computational Geometry (EuroCG)}, pages
  174--177, 2007.

\bibitem{Haasdonk2005}
B.~Haasdonk.
\newblock Feature space interpretation of {SVMs} with indefinite kernels.
\newblock {\em IEEE Transactions on Pattern Analysis and Machine Intelligence},
  27(4):482--492, 2005.

\bibitem{Idel2016}
M.~Idel.
\newblock A review of matrix scaling and {Sinkhorn's} normal form for matrices
  and positive maps.
\newblock Technical report, arXiv:1609.06349, 2016.

\bibitem{Indyk2003}
P.~Indyk and N.~Thaper.
\newblock Fast image retrieval via embeddings.
\newblock In {\em International Workshop on Statistical and Computational
  Theories of Vision (SCTV)}, 2003.

\bibitem{Kurras2015}
S.~Kurras.
\newblock Symmetric iterative proportional fitting.
\newblock In {\em International Conference on Artificial Intelligence and
  Statistics (AISTATS)}, pages 526--534, 2015.

\bibitem{Lin2003}
H.-T. Lin and C.-J. Lin.
\newblock A study on sigmoid kernels for {SVM} and the training of
  {non\nobreakdash-PSD} kernels by {SMO\nobreakdash-}type methods.
\newblock Technical report, National Taiwan University, 2003.

\bibitem{Ling2007}
H.~Ling and K.~Okada.
\newblock An efficient earth mover's distance algorithm for robust histogram
  comparison.
\newblock {\em IEEE Transactions on Pattern Analysis and Machine Intelligence},
  29(5):840--853, 2007.

\bibitem{Mesaros2016}
A.~Mesaros, T.~Heittola, and T.~Virtanen.
\newblock {TUT} database for acoustic scene classification and sound event
  detection.
\newblock In {\em European Signal Processing Conference (EUSIPCO)}, pages
  1128--1132, 2016.

\bibitem{Montavon2016}
G.~Montavon, K.-R. M{\"{u}}ller, and M.~Cuturi.
\newblock Wasserstein training of restricted {Boltzmann} machines.
\newblock In {\em International Conference on Neural Information Processing
  Systems (NIPS)}, pages 3718--3726, 2016.

\bibitem{Muzellec2016}
B.~Muzellec, R.~Nock, G.~Patrini, and F.~Nielsen.
\newblock Tsallis regularized optimal transport and ecological inference.
\newblock In {\em AAAI Conference on Artificial Intelligence (AAAI)}, pages
  2387--2393, 2018.

\bibitem{Naor2007}
A.~Naor and G.~Schechtman.
\newblock Planar earthmover is not in $l_1$.
\newblock {\em SIAM Journal on Computing}, 37(3):804--826, 2007.

\bibitem{Obermann2015}
A.~M. Oberman and Y.~Ruan.
\newblock An efficient linear programming method for optimal transportation.
\newblock Technical report, arXiv:1509.03668, 2015.

\bibitem{Pele2008}
O.~Pele and M.~Werman.
\newblock A linear time histogram metric for improved {SIFT} matching.
\newblock In {\em European Conference on Computer Vision (ECCV)}, pages
  495--508, 2008.

\bibitem{Pele2009}
O.~Pele and M.~Werman.
\newblock Fast and robust earth mover's distances.
\newblock In {\em IEEE International Conference on Computer Vision (ICCV)},
  pages 460--467, 2009.

\bibitem{Rabin2009}
J.~Rabin, J.~Delon, and Y.~Gousseau.
\newblock A statistical approach to the matching of local features.
\newblock {\em SIAM Journal on Imaging Sciences}, 2(3):931--958, 2009.

\bibitem{Rolet2016}
A.~Rolet, M.~Cuturi, and G.~Peyr{\'{e}}.
\newblock Fast dictionary learning with a smoothed {Wasserstein} loss.
\newblock In {\em International Conference on Artificial Intelligence and
  Statistics (AISTATS)}, pages 630--638, 2016.

\bibitem{Rubner2000}
Y.~Rubner, C.~Tomasi, and L.~J. Guibas.
\newblock The earth mover's distance as a metric for image retrieval.
\newblock {\em International Journal of Computer Vision}, 40(2):99--121, 2000.

\bibitem{Schmitz2018}
M.~A. Schmitz, M.~Heitz, N.~Bonneel, F.~Ngol{\`{e}}, D.~Coeurjolly, M.~Cuturi,
  P.~Gabriel, and J.-L. Starck.
\newblock Wasserstein dictionary learning: Optimal transport-based unsupervised
  non-linear dictionary learning.
\newblock {\em SIAM Journal on Imaging Sciences}, 11(1):643--678, 2018.

\bibitem{Schmitzer2016a}
B.~Schmitzer.
\newblock A sparse multi-scale algorithm for dense optimal transport.
\newblock {\em Journal of Mathematical Imaging and Vision}, 56(2):238--259,
  2016.

\bibitem{Schmitzer2016b}
B.~Schmitzer.
\newblock Stabilized sparse scaling algorithms for entropy regularized
  transport problems.
\newblock Technical report, arXiv:1610.06519, 2016.

\bibitem{Seguy2015}
V.~Seguy and M.~Cuturi.
\newblock Principal geodesic analysis for probability measures under the
  optimal transport metric.
\newblock In {\em International Conference on Neural Information Processing
  Systems (NIPS)}, pages 3312--3320, 2015.

\bibitem{Sha2007}
F.~Sha, Y.~Lin, L.~K. Saul, and D.~D. Lee.
\newblock Multiplicative updates for nonnegative quadratic programming.
\newblock {\em Neural Computation}, 19(8):2004--2031, 2007.

\bibitem{Shirdhonkar2008}
S.~Shirdhonkar and D.~W. Jacobs.
\newblock Approximate earth mover's distance in linear time.
\newblock In {\em IEEE Conference on Computer Vision and Pattern Recognition
  (CVPR)}, pages 1--8, 2008.

\bibitem{Sinkhorn1967}
R.~Sinkhorn and P.~Knopp.
\newblock Concerning nonnegative matrices and doubly stochastic matrices.
\newblock {\em Pacific Journal of Mathematics}, 21(2):343--348, 1967.

\bibitem{Solomon2015}
J.~Solomon, F.~de~Goes, G.~Peyr{\'{e}}, M.~Cuturi, A.~Butscher, A.~Nguyen,
  T.~Du, and L.~Guibas.
\newblock Convolutional {Wasserstein} distances: Efficient optimal
  transportation on geometric domains.
\newblock {\em ACM Transactions on Graphics}, 34(4):66:1--66:11, 2015.

\bibitem{Solomon2014}
J.~Solomon, R.~M. Rustamov, L.~Guibas, and A.~Butscher.
\newblock Wasserstein propagation for semi-supervised learning.
\newblock In {\em International Conference on Machine Learning (ICML)}, pages
  306--314, 2014.

\bibitem{Thibault2017}
A.~Thibault, L.~Chizat, C.~Dossal, and N.~Papadakis.
\newblock Overrelaxed sinkhorn-knopp algorithm for regularized optimal
  transport.
\newblock Technical report, arXiv:1711.01851, 2017.

\bibitem{Thorlund-Petersen2004}
L.~Thorlund-Petersen.
\newblock Global convergence of {Newton}'s method on an interval.
\newblock {\em Mathematical Methods of Operations Research}, 59(1):91--110,
  2004.

\bibitem{Tseng1993}
P.~Tseng.
\newblock Dual coordinate ascent methods for non-strictly convex minimization.
\newblock {\em Mathematical Programming}, 59(1--3):231--247, 1993.

\bibitem{Villani2009}
C.~Villani.
\newblock {\em Optimal Transport: Old and New}, volume 338 of {\em
  Comprehensive Studies in Mathematics}.
\newblock Springer, Berlin Heidelberg, Germany, 2009.

\bibitem{Zen2014}
G.~Zen, E.~Ricci, and N.~Sebe.
\newblock Simultaneous ground metric learning and matrix factorization with
  earth mover's distance.
\newblock In {\em International Conference on Pattern Recognition (ICPR)},
  pages 3690--3695, 2014.

\end{thebibliography}

\end{document}